\documentclass[numbers,webpdf,imaiai]{ima-authoring-template-arxiv}%

\theoremstyle{thmstyletwo}%
\newtheorem{theorem}{Theorem}
\newtheorem{proposition}[theorem]{Proposition}%

\newtheorem{example}{Example}%
\newtheorem{remark}{Remark}%

\newtheorem{definition}{Definition}
\newtheorem{corollary}{Corollary}
\numberwithin{equation}{section}

\newcommand{\R}{\mathbb{R}}

\newcommand{\E}{\mathbb{E}}
\newcommand{\N}{\mathcal{N}}

\newtheorem{assumption}{Assumption}

\newcommand{\eps}{\varepsilon}
\newcommand{\Phizero}{\Phi_{0}}
\newcommand{\Phieps}{\Phi_{\eps}}
\newcommand{\inner}[2]{\langle #1,\, #2\rangle}
\newcommand{\Hess}{\nabla_x^2}
\newcommand{\grad}{\nabla_x}

\DeclareMathOperator{\ran}{ran}
\DeclareMathOperator{\Sym}{Sym}
\DeclareMathOperator{\corank}{corank}
\DeclareMathOperator{\codim}{codim}

\DeclareMathOperator{\spann}{span}
\newcommand{\jet}[2]{j^{#1}_{#2}}
\newcommand{\Msf}{M}

\newtheorem{lemma}{Lemma}

\begin{document}


\title[Speciation in Generative Diffusion Models on Compact Riemannian Manifolds]{A Theory of Speciation in Generative Diffusion Models \\ on Compact Riemannian Manifolds}

\author{Alessio Marta*\ORCID{0000-0002-3388-7168}
\address{\orgdiv{Department of Mathematics}, \orgname{University of Milan}, \orgaddress{\street{via Saldini 50}, \postcode{20133 Milan}, \country{Italy}}}}
\author{Paola Causin\ORCID{0000-0002-8285-8101}
\address{\orgdiv{Department of Mathematics}, \orgname{University of Milan}, \orgaddress{\street{via Saldini 50}, \postcode{20133 Milan},  \country{Italy}}}}


\corresp[*]{Corresponding author: \href{email:email-id.com}{alessio.marta@unimi.it}}

\abstract{Speciation in generative diffusion models denotes the emergence of
distinct stable branches during denoising, through which initially
undifferentiated trajectories progressively commit to different data
classes. In this work we develop an intrinsic theory of speciation for
diffusion models supported on compact Riemannian manifolds: the aim is to go beyond existing theoretical descriptions, which usually identify speciation with a symmetric pitchfork bifurcation and assume to work in a large-dimensional space.
We characterize speciation by
bifurcations of the critical points of the evolving probability density. A
spectral heat-kernel representation makes explicit the role of the
manifold geometry, while Poincaré--Hopf and Morse theory impose global
constraints on the number and type of score equilibria and reveal
topologically--imposed geometrical modes.
For mixtures of heat kernels, we prove that generic speciation events
have a one-dimensional critical kernel and admit an $A_2$ fold normal
form; pitchforks and simultaneous multidirectional transitions arise
from nongeneric symmetric configurations. We derive
geometry-dependent estimates of speciation times for bimodal mixtures
and Riemannian regular simplices. We further establish structural
stability of nondegenerate folds under score perturbations and show
that the first-order time shift is determined solely by the component
of the score error along the critical direction.
The theory is illustrated on the sphere using mixtures of
von Mises--Fisher distributions, where pitchfork and saddle--node
bifurcations, topological modes, and hierarchical multiple
speciations are observed. Finally, a chart-based intrinsic
score-learning scheme based on neural networks contrasts the theoretically predicted
transitions on prototypal and more complex datasets.}

\maketitle
\vspace*{30pt}
\section{Introduction}
\label{sec:intro}
Diffusion Models (DMs) are an advanced class of Generative AI models that learn to synthesize data by progressively denoising a sample drawn from pure noise into a coherent output, by reversing a stochastic process that gradually corrupts data into noise~\cite{Song2021}. In the currently accepted theoretical framework, 
both the forward noising process and the reverse generative dynamics can be described within the realm of stochastic differential equations (SDEs), whose probability densities evolve according to Fokker--Planck equations~\cite{Song2019,Ho2020}. 

\medskip

\noindent
DMs exhibit a remarkably rich behavior, in which generation unfolds through a sequence of dynamical transitions. Namely, during the reverse diffusion process, the trajectories progressively commit with  increasing specificity to semantic classes at characteristic times, commonly referred to as \textit{speciation times}. A series of recent works~\cite{raya2023spontaneous, biroli2024dynamical, ambrogioni2025statistical}, drawing inspiration from statistical mechanics, has interpreted these transitions as manifestations of spontaneous symmetry breaking, analogous to phase transitions in equilibrium systems, and has related them to pitchfork-like bifurcations of the score field governing the reverse dynamics. 
In~\cite{biroli2024dynamical} dynamical regimes of the backward generative diffusion are identified, corresponding to characteristic windows separated by milestone times: 
\begin{quote}
{[...] the first one is basically pure Brownian motion; in the second one the backward trajectory finds one of the main classes of the data (for instance if the data consists of images of horses and cars, a given trajectory will specialize towards one of these two categories). In the third regime, the diffusion ``collapses'' onto one of the examples of the dataset.}    
\end{quote} Following the terminology of the present paper, the first and second windows are separated by a \textit{``speciation time''}, while the second and the third windows by a 
\textit{``collapse time''}. 

\medskip

\noindent Concurrently, a growing body of work has emphasized the fact that the latent representations learned by modern neural networks possess a highly non-Euclidean geometric structure~\cite{arvanitidis2021pulling,park2023understanding,pegios2024counterfactual,Yu2025}. Both theoretical analyses and empirical observations indicate that data distributions concentrate on low-dimensional manifolds, whose intrinsic geometry influences learning, representation, and sampling. This geometric viewpoint has motivated the extension of generative models from Euclidean spaces to Riemannian manifolds, where diffusion processes are naturally governed by the Laplace--Beltrami operator rather than the standard Laplacian~\cite{debortoli2022,huang2022riemannian}. This viewpoint is also consistent with the \textit{manifold hypothesis},
according to which high-dimensional observations concentrate near a
low-dimensional nonlinear manifold. When the generative dynamics is
restricted to this intrinsic support, both noising and denoising must
respect its metric, topology, and volume structure.

\medskip

\noindent \textbf{{Original contributions}}

Speciation in diffusion models has been related in past literature to symmetry-breaking phenomena. This theoretical picture remains incomplete in two important respects. First, more general bifurcation mechanisms appear to govern the qualitative evolution of the score dynamics. Second, existing analyses are almost exclusively developed in Euclidean spaces, overlooking the influence of the intrinsic geometry of the data manifold on the organization of the generative process. The purpose of this work is to establish a more general framework able to address both the above points. Adopting the perspective of stochastic differential equations and their associated Fokker--Planck equations, we investigate diffusion processes on compact Riemannian manifolds and characterize the evolution of the corresponding score field. We derive explicit representations of the probability density and of the score, showing how the geometry of the manifold enters the dynamics through its metric, spectrum, and heat kernel. This formulation allows us to interpret speciation as the emergence of stable critical points of the score field and to study their appearance through the language of dynamical systems and bifurcation theory. 
Our analysis reveals the role played by the geometry of the underlying manifold in shaping the organization of the generative dynamics. 
The spectrum of the Laplace--Beltrami operator determines the temporal hierarchy of the diffusion process, while the topology constrains the admissible configurations of critical points through classical index theorems. Moreover, the nodal structure of the eigenfunctions identifies regions characterized by different convergence rates toward equilibrium, providing a geometric interpretation of the heterogeneous temporal organization of generation.
This theoretical framework is illustrated analytically and numerically on $\mathbb{S}^d$ manifolds, where explicit spectral representations make it possible to characterize the bifurcation scenarios leading to speciation. We show that both pitchfork and saddle-node bifurcations naturally arise depending on the geometry and on the initial data distribution, yielding multiple speciation times and hierarchical commitment of trajectories. Finally, we discuss numerical methods for learning score fields on manifolds and present experiments confirming the theoretical predictions.

\medskip

\noindent The main contributions of this work are the following.

\begin{itemize}

\item \textbf{An intrinsic geometric formulation of speciation.}
We formulate noising and denoising directly on compact Riemannian
manifolds through Brownian motion, the Laplace--Beltrami heat flow, and
the Riemannian score.
The heat-kernel
spectral expansion makes explicit how the metric and the
Laplace--Beltrami spectrum enter the temporal organization of the
generative process.

\item \textbf{Topological constraints and geometrical modes.}
Using Poincaré--Hopf and Morse theory, we show that the number and type
of critical points of the score landscape are constrained by the
topology of the data manifold. This leads to a distinction between
\emph{data modes}, associated with the intended clusters, and
\emph{geometrical modes}, which carry no data class but are forced to
exist by the topology of the manifold.

\item \textbf{Genericity and local classification of speciation.}
For sufficiently rich mixtures of heat kernels, we prove that a
generic speciation event has a one-dimensional critical kernel. Under
real-analyticity assumptions, this property holds for a full-measure
set of mixture parameters. We further derive the generic $A_2$ normal
form of the evolving density. Pitchfork bifurcations and simultaneous
losses of stability in several directions arise instead from
non-generic symmetric configurations.

\item \textbf{Geometry-dependent estimates of speciation times.}
For bimodal heat-kernel mixtures, we reduce the critical-point problem
to the minimizing geodesic joining the two centers and derive
short-time estimates of the bifurcation position and time. For equal
weights, the resulting expression depends explicitly on the geodesic
distance between the modes. We also obtain a corresponding estimate
for symmetric configurations supported on Riemannian regular
simplices.

\item \textbf{Structural stability under score perturbations.}
We prove that a nondegenerate $A_2$ speciation persists under small
smooth perturbations of the score, typical for example of approximations by obtained from neural networks. The first-order shift of the
speciation time depends only on the projection of the score error onto
the critical direction,
Consequently, score errors transverse to the critical direction do not
shift the speciation time at first order.

\item \textbf{Numerical approach based on coordinate--representation.}
We develop a chart-based approach for simulating Riemannian
diffusion and learning intrinsic score fields by implicit score
matching.
\end{itemize}

This work is companion to a forthcoming paper
of the same authors, where 
mechanisms of speciation are investigated in the Eulerian framework. While the topics of the two works are naturally related, the emphasis here is on investigating how the intrinsic geometry of the data
manifold constrains which score equilibria can exist, which
bifurcations occur generically, when they occur, and how they can be
computed and learned in intrinsic coordinates. Thus the two works
address complementary levels of the same phenomenon: local
singularity theory in the Euclidean setting, and global spectral,
topological, and geometric organization on manifolds.

\medskip

\noindent The paper is organized as follows. In Section~\ref{sec:Rie-noising} we recall some key notions on noising and denoising processes on Euclidean and Riemannian manifolds and we fix the notation employed in the remainder of the work. In Section~\ref{sec:general_theory_speciation} we discuss how the topology of the manifold impacts the generative dynamics with the creation of data and geometrical modes and we present a general theory of speciation. In particular, we prove that for mixtures of heat kernels with a sufficient number of components, a generic speciation event has a one-dimensional critical kernel or, in other words, that the generative dynamics speciates one direction per bifurcation. We also give the normal form of the generic bifurcation event and we derive geometry-dependent estimates of speciation times for bimodal heat-kernel mixtures and for symmetric mixtures whose centers are placed on the vertices of Riemannian regular simplices. In Section~\ref{sec:stability} we discuss how the theoretical results of the previous section behave under perturbations. We prove the stability of the generic $A_2$ speciation, giving a first order estimate for the displacement of the speciation time, which depends only on the projection of the score error onto the critical direction. In Section~\ref{sec:illustrative_s2} we illustrate the above theoretical findings on the sphere $\mathbb{S}^2$ for assigned mixtures. In Section~\ref{sec:numerical_experiments} we develop a numerical chart-based approach to simulate Riemannian diffusion and an implicit score matching algorithm in local coordinates. We use these algorithms to perform some numerical experiments in which we validate our theoretical findings on given mixtures and real-world datasets.
\section{Noising and denoising processes}
\label{sec:Rie-noising}

\subsection{Euclidean noising and denoising processes}
We let $\{a_i\}_{i=1}^K \in \R^n$ represent $K$ independent data points sampled from the true underlying data distribution $p_0(a)$ that we aim to model; 
we let $W(t), {t\ge 0}$ denote a standard Brownian motion on $\mathbb{R}^n$
defined on a filtered probability space
$(\Omega,\mathcal{F},(\mathcal{F}_t)_{t\ge 0},\mathbb{P})$ satisfying the
usual conditions.
We consider a Ornstein-Uhlenbeck
forward diffusion process $(X_t)_{t\in[0,T]}$ defined as the solution
of the following It\^o stochastic differential equation (SDE)
\begin{equation}
\label{eq:app-forward-sde-euclidean}
\mathrm{d}X(t) \;= -X(t) dt + \mathrm{d}W(t),
\qquad X_0\sim p_0.
\end{equation}
Conditionally on 
$X_0=a$, at any given time $t>0$, the state $X(t)$ is distributed according to a Gaussian with mean $a{\rm e}^{-t}$ and variance $\Delta t = 1-{\rm e}^{-2t}$.

\null

\noindent The time-marginal density $p_t=p_t(x), x\in\mathbb{R}^n,$ of $X_t$ with respect to the Lebesgue
measure on $\mathbb{R}^n$ obeys the Fokker--Planck (or forward Kolmogorov)
equation
\begin{equation}
\label{eq:app-fp-euclidean}	
\partial_t p_t \;=\; \nabla \cdot (xp_t(x))+\frac{1}{2}\,\Delta p_t, 
\qquad p_t|_{t=0}=p_0.
\end{equation}
In this Variance Preserving (VP) process, $p_t$ has stationary law $\N(0,I)$ as $t\to+\infty$. 
The (Stein) score of the forward process at time $t$ is defined as
\[
S_t(x) \;:=\; \nabla_x \log p_t(x), \qquad x\in\mathbb{R}^n,
\]
whenever $p_t$ is positive and differentiable. Anderson's time-reversal
identity yields a reverse-time process $(Y_\tau)_{\tau\in[0,T]}$ whose
marginals coincide with those of $X(t),{t\in[0,T]}$ under
$\tau = T - t$, and which solves the reverse-time SDE
\begin{equation}
\label{eq:app-reverse-sde-euclidean}
\mathrm{d}Y(\tau)
\;=\;
 S_{T-\tau}(Y_\tau)\,\mathrm{d}\tau
\,+\, \mathrm{d}\bar W(\tau),
\, Y_0 \sim p_T,
\end{equation}
where $\bar W(\tau), {\tau\ge 0}$ is a Brownian motion adapted to the reverse
filtration. Equation \eqref{eq:app-reverse-sde-euclidean} provides the mathematical 
foundation of score-based generative modeling. 
In practice, the score field is unavailable and it is approximated by a neural network trained through denoising score matching to give $S_\theta$~\cite{Ho2020}.
For each fixed diffusion time $t$, we measure the score approximation error through the Fisher divergence
weighted by the data distribution
\begin{equation}
  \mathcal{D}_t(S_\theta\Vert S_t;p_t) \;=\; \E_{x\sim p}\bigl\|S_\theta(x)-S_t(x)\bigr\|^2 = 
 \int ||S_\theta(x)-S_t(x)||^2p_t(x)dx,
  \label{eq:fisher}
\end{equation}
equivalently $\|S_\theta-S_t\|_{L^2(p_t)}^2$.

\subsection{Riemannian noising and denoising processes}
\paragraph{Preliminaries}
We denote by $(M,g)$ a smooth, connected, compact $n$-di\-men\-sio\-nal Riemannian manifold without boundary, equipped with metric~$g$.
We let $x=(x^1,\dots,x^n)$ be the local coordinates, $g_{ij}$ the metric tensor components, with inverse $g^{ij}$, and $\mathrm{d}\mathrm{vol}_g
=
\sqrt{|g|}\,
\mathrm{d}x^1\cdots \mathrm{d}x^n$, $|g|:=\det(g_{ij})$,
the Riemannian volume element. In these coordinates, given a smooth function $f$, the components of the Riemannian gradient $\nabla_g f$ are given by $g^{ij}\partial_j f$. To characterize the noising process corresponding to
 \eqref{eq:app-forward-sde-euclidean} and~\eqref{eq:app-fp-euclidean} on $(M,g)$, we will make use of the Laplace--Beltrami operator $\Delta_g$, which for a smooth function $f:M\to\mathbb{R}$ is given by
\begin{equation}
\label{eq:app-laplace-beltrami}
\Delta_g f
\;=\; \operatorname{div}_g (\nabla_g f) = 
g^{ij}\,\partial_i\partial_j f
\,-\,
g^{ij}\,\Gamma^k_{ij}\,\partial_k f ,
\end{equation}
where 
$\operatorname{div}_g(X)=\partial_i X^i+\Gamma^i_{ik}X^k$ and where 
$\partial_j=\partial/\partial x^j$,   $\Gamma^k_{ij}$ being the Christoffel symbols of the Levi--Civita connection, and Einstein's summation convention is understood. The contracted Christoffel symbols
$\Gamma^k:=g^{ij}\Gamma^k_{ij}$
appearing in \eqref{eq:app-laplace-beltrami} 
are quantities encoding how the coordinate basis vectors change from point to point on a manifold.

\medskip

\noindent From now on, for readability, we omit the subscript $g$ to the operators, however keeping in mind  
that these are to be intended in the sense specified above. 

\paragraph{Canonical noising process on compact Riemannian manifolds} A Brownian motion on $(M,g)$ admits the Stratonovich representation
\begin{equation}
\label{eq:app-bm-stratonovich}
\mathrm{d}X_t
=
\sum_{\alpha=1}^{n}
U_\alpha(X_t)\circ \mathrm{d}W_t^\alpha ,
\end{equation}
where $(W_t)_{t\ge0}$ is a standard Brownian motion in $\mathbb{R}^n$
and $\{U_1,\ldots,U_n\}$ is a local orthonormal frame field. The
resulting diffusion is independent of the choice of orthonormal frame. The time--marginal density $p_t=p_t(x)$ of $X_t$ with respect
to the Riemannian volume measure satisfies the \textit{forward Fokker--Planck equation}
\begin{equation}
\label{eq:app-fp-forward}
\partial_t p_t
=
\frac12\,\Delta p_t.
\end{equation}
Analogously to the Euclidean case, we introduce the Riemannian score, 
defined as
\begin{equation}
S_t(x)
:=
\nabla \log p_t =
\frac{\nabla p_t(x)}{p_t(x)}
\label{eq:scoredef}
\end{equation}
The score is a fundamental ingredient in the 
reverse process, which can be
written in Stratonovich form as
\begin{equation}
\label{eq:app-reverse-sde-stratonovich}
\mathrm{d}Y_\tau
=
S_{T-\tau}(Y_\tau)\,\mathrm{d}\tau
+
\sum_{\alpha=1}^{n}
U_\alpha(Y_\tau)\circ
\mathrm{d}\bar W_\tau^\alpha, 
\quad Y_0\sim p_T,
\end{equation}
for $\tau=T-t\in[0,T]$, where $(\bar W_\tau)_{\tau\ge0}$ is another standard
Brownian motion.

\null

\noindent In practice, we will not use the Stratonovich form~\eqref{eq:app-bm-stratonovich}, but we will rather make use of the It\^o SDE form 
of~\eqref{eq:app-bm-stratonovich},
which reads, in local coordinates, 
for $k=1,\ldots,n$
\begin{equation}
\label{eq:app-bm-ito}
dX^k_t = \frac12 b^k(X_t)\, dt + \sum_{j=0}^n\sigma_{j}^k(X_t) dW^j_t
\end{equation}
where $b^k(x) = g^{ij}(x)\Gamma^k_{ij}(x)$  and  
the non--constant diffusion matrix $\sigma_j^i =(\sqrt{g^{-1}(X_t)})_j^i$ is obtained from $U_\alpha=U_\alpha^k\partial_k$ and
$g^{kl}=
U_\alpha^k U_\alpha^l$ ~\cite{HsuStochastic}.
Notice that the form of the drift and diffusion depend on the specific choice of the local coordinates. 

\begin{remark}
The term $b^k(x)$ (``It\^o drift'') is not to be considered as a drift on its own, but rather as part of the diffusive term associated with the Laplace-Beltrami operator as consequence of the fact that the It\^o integral does not satisfy the standard chain rule. Notice as well that $b^k(x)$ does not transform as a vector field under change of coordinates since the
Christoffel symbols are not tensors. 
\end{remark}
\noindent  The corresponding It\^o SDE form of the reverse process reads,
for $k=1,\ldots,n$
\begin{equation}
\label{eq:app-reverse-sde-ito}
\mathrm{d}Y_\tau^k
=
\left[
\frac12\,b^k(Y_\tau)
+
\sum_{l=0}^n g^{kl}(Y_\tau)\,
\partial_l \log p_{T-\tau}(Y_\tau)
\right]
\mathrm{d}\tau
+
\sum_{j=0}^n\sigma_{j}^k(Y_\tau)\,
\mathrm{d}\bar W_\tau^j
\end{equation}

Furthermore, exactly as in the Euclidean scenario \cite{Song2021,deveney2025closing}, we can derive a probability flow ODE such that the forward time evolution of a trajectory initialized from $p_0$ have marginal distribution $p_t$ at time $t$. To this end, let us rewrite the diffusion term in equation~\eqref{eq:app-fp-forward} as the divergence of a probability current. Writing $\nabla p_t = p_t\,\nabla\log p_t$ and making use of the definition of the Laplace-Beltrami operator yields
\begin{equation*}
\partial_t p_t
= \tfrac12 \mathrm{div}\!\big(p_t\nabla\log p_t\big).
\end{equation*}
This is a continuity equation of the form $\partial_t p_t + \mathrm{div}_g(v_t\,p_t) = 0$, with velocity field 
\begin{equation}
v_t \;=\;-\; \tfrac{1}{2}\,\nabla\log p_t =  -\tfrac{1}{2} S_t
\end{equation}
where $S_t$ is the score. If $S_t$ is Lipschitz, the method of characteristics yields that a solution of this equation with initial datum $p_0$ is the pushforward $(\Phi_t)_* p_0$, with $\Phi_t$ the flow defined by the probability flow ODE
\begin{equation}
\label{eq:pflow_equation}
\frac{d}{dt}\Phi_t(x) =  -\tfrac{1}{2} S_t \big(\Phi_t(x)\big),\qquad \Phi_0=\mathrm{id}.
\end{equation}
The main practical advantage of the probability flow ODE compared to the forward SDE is that it can be integrated both forward and backward without the need of a reverse equation. In particular, to generate a sample, we can initialize $
x(T)\sim p_T$ and then integrate \eqref{eq:pflow_equation} backward in time from $T$ to $0$. From a theoretical point of view, the probability flow description allows to study the generative process as a dynamical system.

\medskip

\noindent In the forthcoming numerical experiments - if not differently specified - we will always consider the reverse process, however using for easiness the forward time variable $t$ (notice that $t \to 0$ points to initial data and increasing $t$ points to pure noise). 
 
\paragraph{Heat-kernel expansion and spectral representation}
We compute, for selected cases, the explicit solution of the forward Fokker-Planck equation~\eqref{eq:app-fp-forward},  
and derive from it the expression of
the corresponding score.

\null

\noindent 
We let $\{\phi_k\}_{k\ge0}$ be an
$L^2(M,\mathrm{vol}_g)$-orthonormal basis of eigenfunctions satisfying
\[
-\Delta_g\phi_k=\lambda_k\phi_k,
\quad
\phi_0=\frac{1}{\sqrt{\mathrm{vol}_g(M)}}.
\]
Under the present assumptions on $(M,g)$, the spectrum of $-\Delta_g$ is purely
discrete and such that:
\[
0=\lambda_0<\lambda_1\le \lambda_2\le\cdots\to\infty,
\]
with $\lambda_0$ simple. 
Then the solution of problem~\eqref{eq:app-fp-forward} with $p_t|_{t=0}=p_0$,
admits the spectral expansion
(see, e.g., \cite{zbMATH05681750})
\begin{equation}
\label{eq:app-heat-kernel-expansion}
p_t(x)
=
\sum_{k\ge0}
e^{-\lambda_k t/2}\,
c_k\,\phi_k(x),
\qquad
c_k=
\bigl\langle p_0,\phi_k\bigr\rangle_{L^2(M,\mathrm{vol}_g)} 
\end{equation}
As $t$ increases, all nonconstant modes 
in~\eqref{eq:app-heat-kernel-expansion} decay exponentially, the slowest
decay rate being governed by the first spectral gap $\lambda_1/2$ and
the higher frequencies decaying first
(a review on lower bounds on spectral gap can be found in~\cite{he2013lower}). 
Eventually, due to compactness,
the density converges to the steady value $c_0\phi_0=\frac{1}{\mathrm{vol}_g(M)}$ and
the norm of the score tends to zero.


\noindent 
The score vector associated to the probability distribution $p_t$ as per \eqref{eq:app-heat-kernel-expansion} is
\begin{equation}
\label{eq:score-general}
S_t(x)
=
\frac{\sum_{k\ge 1}
e^{-\lambda_k t/2}\,
c_k\,\nabla \phi_k(x)}{\sum_{k\ge 0}
e^{-\lambda_k t/2}\,
c_k\,\phi_k(x)}
\end{equation}
In the following, 
to carry out the computations, we will focus on initial distributions which are \textit{multimodal mixtures} of $K$ distributions:
$$p_0(x)= \sum_{i=1}^K w_i  p_0^{(i)}(x),\qquad \sum_{i=1}^K w_i=1.$$
By linearity,
the solution of the forward SDE problem evolves as 
$p_t(x)= \sum_{i=1}^K w_i p_t^{(i)}(x)$, with exact
score 
\begin{align}
S_t(x)
=
\frac{\sum_{i=1}^K w_i \nabla p_t^{(i)}(x)}
{\sum_{i=0}^K w_i p_t^{(i)}(x)}
= 
\frac{
  \sum_{k=1}^\infty e^{-\lambda_k t/2} \left(\sum_{i=1}^K w_i c_k^{(i)} \right) \nabla \phi_k(x)}{\sum_{k=0}^\infty e^{-\lambda_k t/2} \left(\sum_{i=1}^K w_i c_k^{(i)} \right)  \phi_k(x)}.
\end{align}
\section{General theory of speciation}
\label{sec:general_theory_speciation}
\subsection{An illustrative example: from pitchfork to fold--type speciation}

Before entering in the core of the discussion, we devote this section to a simple example which advances, within a 1D model, the idea that speciation is a local bifurcation phenomenon governed by general normal-form theory. In this framework, the pitchfork is only one possible realization, corresponding to the ideal case of exact symmetry, while generic perturbations naturally lead to its universal unfolding. We consider here the problem, for $a\in \mathbb{R}, t^* \in \mathbb{R}^+$
\begin{equation}\label{eq:model}
  \dot x \;=\; -x\big((x-a)^2-(t-t^*)\big)
         \;=\; \big((t-t^*)-a^2\big)\,x \;+\; 2a\,x^2 \;-\; x^3 ,
\end{equation}
whose equilibria are the steady point $x=0$ and the two mobile
nullclines $x=a\pm\sqrt{t-t^*}$, existing for $t\ge t^*$. 
The parameter $a$ measures the departure from perfect symmetry. When $a=0$ the system undergoes indeed a (supercritical) pitchfork, for $a \ne 0$
the pitchfork unfolds into an imperfect one, where the secondary branch is generated through a saddle-node bifurcation. In catastrophe theory these bifurcations are also known as $A_3$ and $A_2$ bifurcations, respectively \cite{ArnoldGuseinZadeVarchenko}. The branch $x=0$ persists even after the saddle--node appearance and exchanges
stability at the critical time $t_c=t^*+a^2$. 
Figure~\ref{fig:bifurcations} shows the pitchfork ($a=0$, left)
and the saddle-node ($a=1$, right) scenarios, both with $t^*=1$, together with
some representative trajectories obtained for different initial conditions. 
 To expose the universal unfolding parameters, we proceed as customarily by 
 placing the bifurcation at the natural local origin by the translation $x=u+\tfrac{2a}{3}$, which  
removes the quadratic term of \eqref{eq:model} and recasts the model in the
\textit{cusp normal form}
\begin{equation}\label{eq:cusp}
  \dot u \;=\; h + \mu\,u - u^3, \quad
  \mu = (t-t^*)+\frac{a^2}{3}, \quad
  h   = \frac{2a}{3}\Big((t-t^*)-\frac{a^2}{9}\Big).
\end{equation}
The offset $a$ induces, to leading order, a non-zero symmetry-breaking field $h$ which is $O(a)$ in a fixed neighborhood of $t^*$. The term $h$ is the responsible
for the imperfect pitchfork in form~\eqref{eq:cusp}.
Notice that this illustrative example depicts a specific case of
unfolding: the explicit factor $x$ keeps $x=0$ an exact equilibrium for
every $a$ and produces
the transcritical crossing at $t_c$. A generic field $h$ removes this
residual symmetry, 
leaving the canonical picture of a single smooth primary branch plus one
fold.
\begin{figure}[H]
\centering
\includegraphics[width=\textwidth]{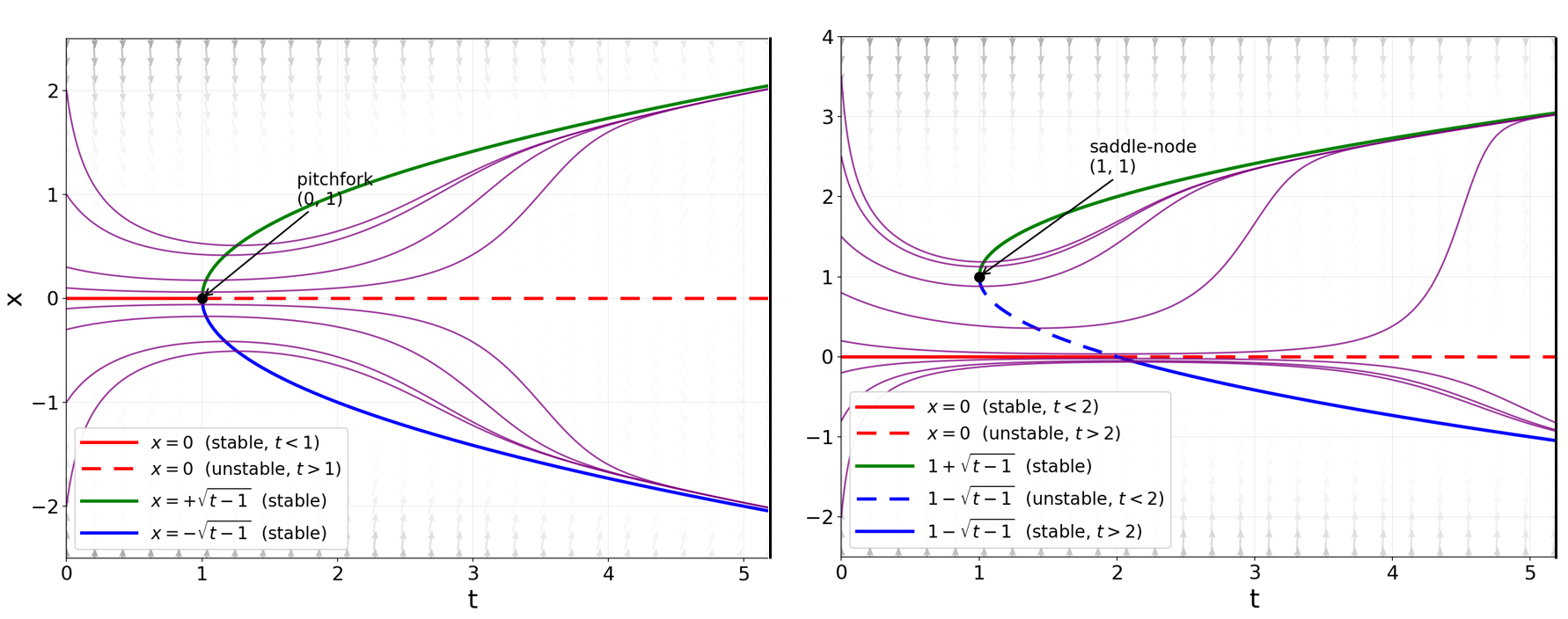}
\caption{Bifurcation scenarios for the 1D problem~\eqref{eq:model} with $t^*=1$.
Left: pitchfork bifurcation ($a=0$). Two stable nullclines branch from the original stable equilibrium point $x=0$, which becomes unstable after the bifurcation time ($t_c=t^*=1$).
Right: saddle-node ``imperfect'' bifurcation ($a=1$). A pair of stable and unstable nullclines emerge at $x=1$, which is not the original equilibrium point of the system; here $t_c=2> t^*$.  
Solid lines denote
stable equilibria, dashed lines unstable ones, and violet curves
    representative trajectories stemming from different initial conditions. Notice that
    increasing time $t$ here corresponds to noise-to-data reverse flow in the following examples.
}\label{fig:bifurcations}
\end{figure}


\subsection{Critical points degeneracy and branch formation}
In~\cite{e28020195,ambrogioni2026PhaseTransitions} and references therein,
the mathematical descriptors of speciation in the generative process are  
pitchfork bifurcations typical of symmetry-breaking scenarios. 
Here, we depart from this concept and
we consider more general, non--symmetric scenarios, typical of real datasets and approximated scores. 

\medskip

\noindent We denote the set of real zeros of the score at time $t$ as $\Xi=\Xi(t)$, which is the union 
of connected loci and whose points can be attractors, repellers or saddle points for the trajectories of the generative process. 
The set $\Xi$ is identified by the critical points of the distribution and thus by $\nabla p_t = 0$, or equivalently  $S_t(x)=0$. In the generative process, a \emph{speciation time} corresponds to time $t^*$ at which new stable branches of zeroes appear in $\Xi$, with a bifurcation. If a branch remains as such till the end of the process, with a final commitment to a certain center, then that time is also a \emph{collapse time}, in accordance with \cite{ventura2025manifolds}. Concretely, these milestones  correspond to points $x^* \in \Xi_{t^*}$ at which the kernel of the Riemannian Hessian of $\log(p_{t^*})$ is non-trivial and such that the non-vanishing eigenvalues are negative. In coordinates, this is tantamount to the vanishing of the determinant of
\begin{equation}
(\operatorname{Hess}\log p_t)_{ij}
=
\frac{\partial_i\partial_j p_t-\Gamma^k_{ij}\partial_k p_t}{p_t}
-
\frac{\partial_i p_t\,\partial_j p_t}{p_t^2},
\end{equation}
which, since this condition must hold together with $S_{t^*}(x^*) = 0$, simplifies to
\begin{equation}
\label{eq:coordinate_hessian}
(\operatorname{Hess}\log p_{t^*})_{ij}|_{x=x^*}
=
\frac{\partial_i\partial_j p_{t^*}}{p_{t^*}}|_{x=x^*}
\end{equation}
from which we conclude that the nature of the critical point $x^*$ is the same for $p_t$ and $\log p_t$.
Notice that in the case of a multimodal distribution, the 
points in the locus $\Xi(t)$ satisfy the vector relation
\begin{equation}
\nabla p_t = \sum_{k=1}^\infty e^{-\lambda_k t/2} \left(\sum_{i=1}^K \eta_i c_k^{(i)} \right) \nabla \phi_k(x) = 0
\label{eq:speciation}
\end{equation}
and the Hessian degeneracy condition~\eqref{eq:coordinate_hessian} in local coordinates reads
\begin{equation}
\label{eq:hessian_degeneracy}
\det\left( \left[ \sum_{k=1}^\infty e^{-\lambda_k t/2} \left(\sum_{i=1}^K \eta_i c_k^{(i)} \right) \partial_{x_\alpha}\partial_{x_\beta}\phi_k(x) \right]_{\alpha\beta} \right) = 0.
\end{equation}

\medskip

\paragraph{Independence from the coordinate chart} In principle, different charts could be used or learned on the manifold: the following proposition
guarantees that condition~\eqref{eq:coordinate_hessian} is coordinate-independent. 
\begin{proposition}
The points for which $\ \det(\operatorname{Hess} \log p_{t^*}(x^*))=0$, or equivalently  $\det(\operatorname{Hess} p_{t^*}(x^*))=0$, do not depend on the choice of the chart. 
\end{proposition}
\begin{proof} \,
Let $y = f(x)$ be a smooth change of coordinates near $x^*$. Then, in the $y$ coordinates, the Hessian reads:
\begin{equation*}
\dfrac{\partial y^\alpha}{\partial x^i} \dfrac{\partial y^\beta}{\partial x^j} \frac{\partial_{y^\alpha}\partial_{y^\beta} (p_{t^*} \circ f^{-1} ) }{p_{t^*} \circ f^{-1}} \Big|_{y=f(x^*)} 
\end{equation*}
and therefore
\begin{equation*}
\begin{split}
0 = \det\left( \frac{\partial_{x^\alpha}\partial_{x^\beta} p_{t^*} }{p_{t^*}}\Big|_{x=x*}  \right) = \det \left( \dfrac{\partial y^\alpha}{\partial x^i} \dfrac{\partial y^\beta}{\partial x^j} \frac{\partial_{y^\alpha}\partial_{y^\beta} (p_{t^*} \circ f^{-1} ) }{p_{t^*} \circ f^{-1}} \Big|_{y=f(x^*)} \right) \\
= \det (J_f)^2 \det \left( \frac{\partial_{y^\alpha}\partial_{y^\beta} (p_{t^*} \circ f^{-1} ) }{p_{t^*} \circ f^{-1}} \right)\Big|_{y=f(x^*)}.
\end{split}
\end{equation*}
where $J_f$ is the Jacobian of the function $f$. Since $f$ is a diffeomorphism, $\det (J_f)^2 \neq 0$ and the claim follows.
\end{proof}
In addition, the fact that at a critical point $\operatorname{Hess}_{y}(p_{t^*} \circ f^{-1})|_{y=y^*} = (J_f)^T \operatorname{Hess}_{x}(p_{t^*}) J_f |_{x=(x^*)}$ implies the following result.
\begin{corollary}
The signature of the eigenvalues of $\operatorname{Hess} p_{t^*}(x^*)$ is preserved under change of charts.
\end{corollary}

\subsection{Topological constraints}\label{sec:ghosts}

The number and type of non--degenerate critical points of $p_t$ are not uniquely determined by the above criticality conditions but are also  constrained by the
topology of the space itself. To adress this point, we recall the following classical results, adapted to the present case~\cite{zbMATH03555096,zbMATH06256815}:
\begin{proposition}[Poincar\'e--Hopf and weak Morse inequality]\label{prop:morse}
For a Morse density $p_t$ with isolated spatial critical points on a closed $d$--manifold $M$,
\begin{equation}\label{eq:morse}
  \sum_\zeta (-1)^{\zeta} C_\zeta = \chi(M),
  \qquad
  \#\{\text{critical points}\} \; = \sum_\zeta C_\zeta\ge\; \sum_{k=0}^{d} b_k(M),
\end{equation}
where $\{\zeta\}$ are the Morse indices of $\log p_t$ (the number of negative eigenvalues of its Hessian), $C_\zeta$ is the number of critical points of index $\zeta$,
$\chi$ is the Euler characteristic and $b_k$ the Betti numbers. 
\end{proposition}
Given a target distribution with $K$ intended clusters, we call \textit{data modes}
the critical points belonging to $\Xi$ corresponding to the commitment to the data
centers and 
\textit{geometrical modes} 
the critical points in $\Xi$ forced to
exist through~\eqref{eq:morse}. These geometrical modes can be both attractors and repellers, depending on the value of their Morse index, and contribute in the  hierarchical sifting of the classes in time. 
In particular, they can be present both from the very beginning of the reverse process (see Example~\ref{ex:topologica_vmf}) or emerge along the generative dynamics (see Section~\ref{sec:trimodal_theoretical}). 

\begin{example}[Persistent geometrical mode in an unimodal von Mises--Fisher distribution]
We consider here as an illustrative example the case of a single ($K=1$) von Mises-Fisher density on
the unit sphere $\mathbb{S}^n$ (see ~\ref{app:appA} for its definition) with center $\mu$ and
variance parameter $\kappa>0$, the initial score is the tangential field
\begin{equation}\label{eq:vmf-score}
  S_0(x) = \nabla_{\mathbb{S}^n}(\kappa\,\mu\cdot x)
       = \kappa\big(\mu-(\mu\cdot x)\,x\big),
\end{equation}
whose zeros are $x=\pm\mu$: 
a data mode at $+\mu$ (maximum, $\zeta=+2$) and a geometrical mode at the antipode $-\mu$ (minimum, $\zeta=0$), both nondegenerate with
$\operatorname{Hess}\log p=\mp\kappa\,I$ on the tangent space. 
Considering for simplicity the case $n=2$, it is immediate to see that this behaviour responds to the constraint of Proposition~\ref{prop:morse}. Indeed, with these  Morse indices we get $(-1)^2+(-1)^0=2=\chi(\mathbb{S}^2)$ 
that explains the fact that the maximum has a correspondent companion minimum. 
This latter acts as a source from which trajectories diverge, the streamlines being all the grand circle geodesics that connect
$-\mu$  to the 
target $\mu$.
Notice that in this unimodal configuration 
there are only a maximum and a minimum with no available saddle points to contribute in the sum. This implies that the minimum is persistent. 
Notice, in addition, that the antipodal minimum  
has no counterpart in the corresponding Euclidean problem -a 
single Gaussian on the strip $(0,\pi)\times (0,2\pi)$- which has $\chi=1$ and thus exactly one critical point of maximum. 
\label{ex:topologica_vmf}
\end{example}

\begin{remark} If one works in a local chart (and not in a global one), some of the critical points implied by Proposition~\ref{prop:morse} may be locally missed, if they lay out of the chart.  
\end{remark}

\noindent At a  bifurcation point, from a topological point of view, $p_t$ is not Morse and Proposition~\ref{prop:morse} does not apply.  
The behavior of a bifurcation in $\Xi$ depends rather on the local structure of the degenerate critical points of $p_t$, which can be described by normal forms in singularity theory (see, e.g., \cite{zbMATH01552061,zbMATH03981627,zbMATH00193465}).

\subsection{Speciation theory for mixtures of heat kernels}
\label{section:generic_mixtures}
We consider hereafter a mixture of heath kernels
\begin{equation}  p_0(x)\;=\;\sum_{i=1}^{K} w_i\,p^{(i)}(x;y_i,\sigma^2_i)
\qquad x \in \mathbb{R}^n
\label{eq:heat-mixture}
\end{equation}
and we denote by $\Theta:=(w,\sigma^2,y)\in\Delta^{\circ}_{K-1}\times(0,\infty)^K\times \Msf^K$ the parameters of the mixture, namely the weights $w_i \in (0,1)$ on the probability simplex $\Delta^{\circ}_{K-1}$, the initial component scales $\sigma_i^2>0$ -- the Riemannian analogous of the variance of a component in the Euclidean scenario -- and the position of the centers $y_i$.
\begin{remark}
We can approximate the empirical distribution of a dataset setting $\sigma_i \rightarrow 0$ for every $i$.
\end{remark}
Let $p_t$ be the forward evolution of~\eqref{eq:heat-mixture} under the forward Fokker-Planck equation~\eqref{eq:app-fp-forward}. Given a sufficient number of centers, we can characterize the bifurcations of $p_t$ by the following propositions. The proofs, which are rather technical, are deferred to Appendix~\ref{app:appB}. The first result affirms that, in absence of a specific structure, the generative dynamics speciates one direction per bifurcation.
\begin{proposition}
[Appendix~\ref{app:appB}, Proposition \ref{prop:corank-one}]
Let $(M,g)$ be a compact Riemannian manifold with smooth metric $g$. Then there is a nonempty open set of parameters $G\subseteq \Delta^{\circ}_{K-1}\times(0,\infty)^K\times \Msf^K$ such that for every $\theta \in G$ every bifurcation $(x^*,t^*)$
Since a bifurcation is a pair $(x^*,t^*)$ with $\nabla p_{t^*}(x^*)=0$
and $\operatorname{Hess} u_{t^*}(x^*)$
satisfies $\dim\ker \operatorname{Hess}_{x^*}p_{t^*}=1$.  Moreover, when $g$ is real-analytic the set $G$ has full measure in the whole parameter space $\Delta^{\circ}_{K-1}\times(0,\infty)^K\times \Msf^K$, provided $K\ge\binom{n+2}{2}+n+1$.
\label{prop:onedirection}
\end{proposition}

\noindent Notice that, in the analytic scenario, this result do not exclude the presence of bifurcation events in which the generative dynamics speciates simultaneously in different directions -- or more precisely with $\dim\ker \operatorname{Hess}_{x^*}p_{t^*}>1$. It just affirms that a random choice of the parameters yields $\dim\ker \operatorname{Hess}_{x^*}p_{t^*}=1$ with full probability. In the exceptional set -- of null measure -- for which this property does not holds true, we find scenarios with symmetric configurations of the parameters, see Section~\ref{sec:riemannian_simplices} for an example. 

\medskip
In our framework, near a bifurcation point, the time plays the role of a perturbation parameter of a canonical spatial singularity \cite{ArnoldGuseinZadeVarchenko}; the associated normal form is called an unfolding of the spatial singularity. For a mixture of heat kernels the generic bifurcation assumes the following $A_2$ normal form:
\begin{proposition}[Appendix~\ref{app:appB}, Corollary~\ref{cor:cerf}]
Let $\theta$ be in $G$ of the Proposition~\ref{prop:onedirection}  and let $(x^*,t^*)$ be a bifurcation. Then, for a sufficiently large number of components $K$ of the mixture ($K\ge\binom{n+3}{3}+n+1$
suffices when $g$ is analytic), there are a $t$-dependent $C^\infty$ change of spatial coordinates near $x^*$, a $C^\infty$ reparametrization of $t$ near $t^*$, and signs $\varepsilon_1,\dots,\varepsilon_n\in\{\pm1\}$ such that near the bifurcation the distribution
has the normal form  
\begin{equation}
  \label{eq:cerf}  p_t(x)\;=\;c(t)\;+\;x_1^3\;+\;\varepsilon_1\,(t^*-t)\,x_1   \;+\;\varepsilon_2x_2^2+\dots+\varepsilon_nx_n^2,
\end{equation}
where $c(t)$ is a smooth function of $t$ alone, $x_1$ is the critical coordinate and the signs
$\varepsilon_2,\dots,\varepsilon_n$ are those of the nonzero Hessian eigenvalues at $(x^*,t^*)$.  \label{prop:A2form}
\end{proposition}
\noindent The corresponding score is
\begin{equation}
  \label{eq:cerf_score}  S_t(x)\;=\; \left(\frac{3x_1^2\;+\;\varepsilon_1\,(t-t^*)}{p_t}, \frac{2\varepsilon_2 x_2}{p_t}, \dots, \frac{2\varepsilon_n x_n}{p_t} 
  \right)
\end{equation}
Near a critical point, immediately after the speciation time, the new branches depart from the critical points as $x_1 \propto \pm \sqrt{t^*-t}$. 

\begin{remark}
The condition $K\ge\binom{n+3}{3}+n+1$ in Proposition~\ref{prop:A2form} is sufficient, but not necessary. For example, a single asymmetric pair of heat kernels already produces an $A_2$ fold for $n>2$. 
\end{remark}

\begin{remark}
It is interesting to check whether the above assumption on the minimal number of centers hold  
for real datasets. Consider for example the CelebA dataset \cite{liu2015faceattributes}, whose intrinsic dimension is estimated in the range $26$ to $63$, see e.g. \cite{pope2021the,zhan2026learning}. Setting this as the dimension $n$ of the working space (for example using a Stable Diffusion approach in a latent space learned via a well trained autoencoder) implies that $n^3$ ranges from $17,576$ to $250,047$. The number of images in this dataset,  considered as a mixture of heat kernels, is $K=202,599$, therefore the condition on $K$ is acceptably satisfied.
\end{remark}

\subsubsection{Estimates of speciation time}

\paragraph{The bimodal distribution case.}
\label{sec:bimodal_speciation_time}
Consider the bimodal mixture of heat kernels 
\begin{equation}
p_0(x) = \alpha \, p_{\sigma^2}(x; y_1) + (1-\alpha)\, p_{\sigma^2}(x; y_2),
\end{equation}
where $\alpha$ is the mixing weight. By construction, we can use the semigroup property to study the evolution of $p$ in time under the forward Fokker-Planck equation, yielding:
\begin{equation}
p_t(x) = \alpha \, p_{\sigma^2 + t/2}(x; y_1) + (1-\alpha)\, p_{\sigma^2+t/2}(x; y_2),
\end{equation}
where the shape of the mixture is preserved under time evolution upon a shift of the spread parameter $\sigma^2 \mapsto \sigma^2 + t/2$.  Therefore, the problem reduces to a one-parameter family of static mixtures, parameterized by $s = \sigma^2 + t/2 \geq \sigma^2$. Using this latter parametrization and considering for simplicity the zeroth-order approximation (valid for low curvature manifolds) gives
\begin{equation}\label{eq:zeroth_order_cimodal}
p_s(x) \approx \dfrac{1}{(4\pi s)^{n/2}} \left( \alpha\, e^{-d(x,y_1)^2/4s} + (1-\alpha)\, e^{-d(x,y_2)^2/4s} \right),
\end{equation}
with, up to an irrelevant constant, the score 
\begin{equation}
S_s(x) = \frac{\alpha \,\nabla p_s(x,y_1) + (1-\alpha)\,\nabla p_s(x,y_2)}{\alpha\, p_s(x,y_1) + (1-\alpha)\, p_s(x,y_2)} = \pi_1(x) p_s(x) + \pi_1(x) p_s(x),
\end{equation}
where $\pi_1 = \alpha (\nabla p_s(x))/p_s(x)$ and  $\pi_2 = (1 - \alpha) (\nabla p_s(x))/p_s(x)$. 

\medskip

\noindent First we prove that, in accordance to the results of Sect.~\ref{section:generic_mixtures}, the speciation mechanism can be reduced to a 1D problem and, namely, along the geodesic connecting the two centers. 
 
 \begin{proposition}[Critical points belong to the geodesic through the centers]
 Suppose that $y_1$ and $y_2$ lie in a geodesically convex open set $U \subset M$ (so that in $U$ there is a unique length-minimizing geodesic connecting the two points)
 and that the distance function is $\mathcal{C}^1$.
 Under the zeroth--order approximation~\ref{eq:zeroth_order_cimodal}, any critical point $x^*$ lies on a geodesic curve connecting $y_1$ and $y_2$. 
  \end{proposition}
\begin{proof} \,
 Setting the score to zero, we obtain:
\begin{equation}
S_s(x) = -\frac{1}{2s} \left[ \pi_1(x)\, \nabla d(x,y_1)^2 + \pi_2(x)\, \nabla d(x,y_2)^2 \right] = 0
\label{eq:scores}
\end{equation}
Using the relation (see~\cite{DoCarmo,lee2019introduction})
\begin{equation}
\nabla d(x, y_i)^2 = 2d(x,y_i)\, \nabla d(x,y_i) = -2\, \exp_x^{-1}(y_i),
\end{equation}
where $\exp_x^{-1}$ is the inverse of the exponential map centered at $x$, we can write~\eqref{eq:scores} as the following equation on the tangent space
\begin{equation}
\frac{1}{s}\left[\pi_1(x)\,\exp_x^{-1}(y_1) + \pi_2(x)\,\exp_x^{-1}(y_2)\right]=0,
\end{equation}
where $\exp_x^{-1}(y_i)$ is the initial velocity of the minimizing geodesic from $x$ to $y_i$. 
Since $\pi_1$ and $\pi_2$ are positive, this equation yields that the vectors $\exp_x^{-1}(y_1)$ and $\exp_x^{-1}(y_2)$ must be antiparallel, with with magnitudes in ratio $\pi_2/\pi_1$. Observing that by definition $\exp_x^{-1}(y_i)$ points from $x$ toward $y_i$ along the minimizing geodesic, these vectors can be antiparallel if and only if the geodesics $x \to y_1$ and $x \to y_2$ leave $x$ in opposite directions, namely $x$ must belong to the geodesic connecting $y_1$ with $y_2$.
\end{proof}
\noindent We now estimate the position and time of the bifurcation by looking for degenerate critical points along the geodesic curve connecting the two centers.

\begin{proposition}[Position of the bifurcation]
Let $\gamma : [0, D] \to M$ be the minimizing geodesic connecting $y_1$ with $y_2$, parametrized by the arc length $u$, such that $\gamma(0) = y_1, \gamma(D) = y_2$ and $D = d(y_1, y_2)$. 
Then, the position of the bifurcation is the solution $u^*=u^*(s)$ of
\begin{equation}
\label{eq:bimodal_score_zero}
\log\frac{\alpha}{\,1-\alpha\,} + \log\frac{u}{D-u} = \frac{D(2u - D)}{4s}
\end{equation}
\end{proposition}
\begin{proof} \,
\, Using the arc length parametrization $u$, we have
\begin{equation}
p_s(u) = \alpha \, e^{-u^2/4s} + (1-\alpha)\, e^{-(D-u)^2/4s}
\end{equation}
and therefore
\begin{equation}
S_s|_{\gamma} = \frac{-\frac{\alpha u}{2s} e^{-u^2/4s} + \frac{(1-\alpha)(D-u)}{2s} e^{-(D-u)^2/4s}}{\alpha \, e^{-u^2/4s} + (1-\alpha)\, e^{-(D-u)^2/4s}}.
\end{equation}
Setting the numerator to zero gives the critical point condition
\begin{equation}
\alpha \, u \, e^{-u^2/4s} = (1-\alpha)(D-u)\, e^{-(D-u)^2/4s}
\end{equation}
Dividing both sides of this equation by $(1-\alpha)(D-u)  e^{-u^2/4s}$ and taking the logarithm, yields the thesis.
\end{proof}
\noindent Eventually, substituting  \eqref{eq:bimodal_score_zero} in the degeneracy condition \eqref{eq:hessian_degeneracy} gives the estimate of the bifurcation time.
\begin{corollary}[Bifurcation time for equal weights]
Under the hypothesis of symmetric weights, we find the 
explicit expression of the bifurcation time
\begin{equation}
\label{eq:bimodal_speciation_time}
t^* = \dfrac{D^2}{4} - 2 \sigma^2
\end{equation}
\end{corollary}
\begin{proof} \,
\, For $\alpha=1/2$ the estimate~\eqref{eq:bimodal_score_zero} becomes
\begin{equation}
\log\frac{u}{D - u} = \frac{D(2u-D)}{4s}.
\end{equation}
The midpoint $u^*=D/2$ of the geodesic trivially solves this equation. Since both the sides of this expression are monotone functions of $u$, the midpoint is the unique solution. Setting $v = u - u^*$ and expanding $p_s$ to second order around $v=0$ 
we find the degeneracy condition
\begin{equation}
\left(\frac{D^2}{8s} - 1\right) = 0,
\end{equation}
which, since $s = \sigma^2 + t/2$, gives the sought estimate.
\end{proof}

\subsubsection{A symmetric configuration: Riemannian regular simplices}
\label{sec:riemannian_simplices}
Consider a Riemannian $p$-simplex, whose vertices $\{x_k\}_{k=1}^p$ are mutually equidistant (see \cite{zbMATH03551609,zbMATH00411588,zbMATH06521450} for the existence and the construction of Riemannian simplices). For simplicity we consider a mixture with equal weights. Suppose that all the vertices belong to a normal coordinate chart centered at the barycenter $x_c$ of the simplex. Introducing polar normal coordinates at the center of mass $x_b$, we have $x_k = \operatorname{exp}_{x_b}(v_k)$ for some $v_k \in T_{x_b} M$ of the form $r_c \, \omega_k$, with $r_c$ distance between the vertices and the barycenter and $\omega_k \in \mathbb{S}^{n-1}$. In these coordinates, given a vector $v$, the following zeroth-order approximation of the Riemannian distance holds true:
\begin{equation}
d(\operatorname{exp}_{x_b}(v),x_k) \approx |v -v_k|.
\end{equation}
The subsequent term in the expansion, which we neglect, is quadratic in $v,v_k$ and involves the Ricci tensor $R(v,u_k)$. Using this approximation, together with the short-time expansion of the empirical $p_t$, the mixture takes the form
\begin{equation}
p_t(x) \approx \dfrac{1}{p} \sum_{k=1}^p \exp^{-|v-v_k|^2/2t}
\end{equation}
where $v = \operatorname{exp}_{x_b}(x)$. For equal weights, Karcher's center of mass construction entails that $\sum_k v_k = 0$. Computing the gradient and the Hessian of $p_t$ at $v=0$ yields
\begin{equation}
\nabla_\alpha p_t(0) \approx  - \dfrac{1}{p} \sum_{k,i=1}^p \delta_{i}^\alpha \dfrac{(v_k^i)}{t} \exp^{-r_c^2/2t}  = 0
\end{equation}
and
\begin{equation}
\operatorname{Hess}^{\alpha\beta} p_t(0) \approx  \dfrac{1}{tp} \sum_{k,i,j=1}^p \delta_{i}^\alpha \delta_{j}^\beta \left( - 1  + \dfrac{(v_k^i v_k^j)}{t} \right) \exp^{-r_c^2/2t} =  \dfrac{1}{tp} \left( p\mathbb{I}^{\alpha\beta} - \dfrac{(v_k v_k^T)^{\alpha\beta}}{t} \right) \exp^{-r_c^2/2t}.
\end{equation}
The center of mass $x_b$ is a critical point. By symmetry, in this zeroth-order approximation all the inner products $\langle v_j, v_k \rangle$ are equal, as in the Euclidean case. Consequently, we have
\begin{equation}
|v_j-v_k|^2 = 2r_c^2 - 2 \langle v_j, v_k \rangle = 2r_c^2 - \dfrac{r_c^2}{n-1} = \dfrac{2n}{n-1}r_c^2
\end{equation}
and thus
\begin{equation}
\operatorname{Hess} p_t(0) \approx  \dfrac{1}{tp} \left( 1 - \dfrac{2}{t(p-1)}r_c^2 \right) \mathbb{I}  \exp^{-r_c^2/2t} 
\end{equation}
The Hessian is degenerate when $1 - \dfrac{2}{t(p-1)} r_c^2 = 0$,  from which we read the approximated speciation time $t^\star \approx  \dfrac{2r_c^2}{p-1}$. Since at order zero in normal coordinates we have $g_{ij} = \delta_{ij}$, we can use the Euclidean circumradius-to-edge relation for regular simplices to give an approximation of the distance $D$ between the vertices $\{x_k\}$, namely $D^2 = \dfrac{2p}{p-1}r_c^2$, yielding the estimate 
\begin{equation}
t^\star \approx  \dfrac{D^2}{2p}.
\end{equation}
Notice that for $p=2$ we obtain again the estimate for the bimodal case, namely the only 2-simplex.

\section{Stability of the generic bifurcation}
\label{sec:stability}
In practical applications, the score $S_t$ is generally unknown. In diffusion models a neural network is trained to approximate it, e.g. via denoising score matching. In Section~\ref{sec:numerical_techniques} we discuss a Riemannian implementation of score matching in local chart, while we refer to \cite{debortoli2022} for an ambient space implementation. 
In any case, the consequence of using a proxy of the score is that
the integration of the reverse 
process does not yield, in general, an exact recovery of the unknown probability
distribution along the trajectory. 
It is therefore of pivotal importance to establish wether the results of the previous section holds true also for an approximation of $p_t$. 
To this end, we
let $\Phi_0(x,t)=\log p_t(x)$ and we 
study the stability of its bifurcations under small perturbations. Let $U \times I \subset M \times (0,\infty)$. Fix a point $(x^*,t^*)\in U\times I$ and a smooth one-parameter family of
perturbations
\begin{equation}
  \label{eq:family}
  \Phieps(x,t) = \Phizero(x,t) + \eps\,\Psi(x,t) + O(\eps^2),
  \qquad \Psi \in C^\infty(U\times I),
\end{equation}
with the $O(\eps^2)$ remainder smooth and uniformly bounded with its
derivatives on a neighbourhood of $(x^*,t^*)$.
We assume the following hypotheses on the bifurcation point:
\begin{assumption}[$A_2$ fold]
\label{ass:fold}
At $(x^*,t^*)$ the reference potential $\Phizero$ satisfies:
\begin{enumerate}
  \item[\textnormal{(C0)}] 
  $\grad\Phizero(x^*,t^*)=0$ (critical point);
  \item[\textnormal{(C1)}] the Hessian $\operatorname{Hess} \Phi_0(x^*,t^*) := \Hess\Phizero(x^*,t^*)$ has a
        one-dimensional kernel $\ker H = \R v$, $v \in T_{x^*}M$;
  \item[\textnormal{(C2)}] $\displaystyle A := \nabla^3\Phizero(x^*,t^*)(v,v,v) \neq 0$
        (the $A_2$, or fold, non-degeneracy);
  \item[\textnormal{(C3)}] $\displaystyle B := g_{x^*}(v,\partial_t\grad\Phizero(x^*,t^*)) \neq 0$
        (time-transversality).
\end{enumerate}
\end{assumption}
\begin{remark}
By Proposition~\ref{prop:A2form}, Assumption \eqref{ass:fold} is satisfied for almost every mixture $p_t$ with a number of centers of $O(n^3)$, with $n$ intrinsic dimension of the data manifold $M$. 
\end{remark}

\noindent Under Assumption~\ref{ass:fold}, we prove that fold bifurcations are preserved and in particular that a speciation event does not disappear under small perturbation.
\begin{proposition}
[Preservation of the fold  bifurcation under  perturbations]
\label{thm:a2_perturbation}
Let Assumption~\ref{ass:fold} holds true. Then there exist $\eps_0>0$ and smooth functions
$x^*(\eps)$, $t^*(\eps)$ on $(-\eps_0,\eps_0)$ with
$x^*(0)=x^*$, $t^*(0)=t^*$ such that for every
$\lvert\eps\rvert<\eps_0$ the perturbed field $\Phieps(\cdot,t)$ has an $A_2$
fold at $\big(x^*(\eps),\,t^*(\eps)\big)$.
Then:
\begin{itemize}
    \item[i)] 
    The speciation time of the perturbed system is shifted by
\begin{equation}
  \label{eq:displacement}
  t^*(\eps) - t^*=
  - \frac{g(v,\grad\Psi(x^*,t^*))}
         {g(v,\partial_t\grad\Phizero(x^*,t^*))}\,\eps
  + O(\eps^2)
\end{equation}
\item[ii)] In normal coordinates centered at $x^*$, the speciation location $x^*(\varepsilon)$ is shifted in the direction of $ker(H|_{x^*})$ of an amount of order $\varepsilon$ (see Eq. \eqref{eq:spatial_displacement}).
\end{itemize}
\end{proposition}
\begin{proof} \,
We prove statements i) and ii) working in normal coordinates centered at the critical point $x^*$. Pull the family $\Phieps$ back to $T_{x^*}M$ through the exponential map by setting
\[
  f(\xi,t) := \Phieps\big(\exp_{x^*}\xi,\ t\big),
  \qquad \xi\in T_{x^*}M,\ \lvert\xi\rvert<r,
\]
for $r$ smaller than the injectivity radius at $x^*$. Since $(d\exp_{x^*})_0=\mathrm{Id}_{T_{x^*}M}$
and the Christoffel symbols vanish at the origin, we have 
$\nabla_\xi f(0,t^*) = d\Phizero(x^*,t^*) = 0$ and $\nabla^2_\xi f(0,t^*) = H$. The Hessian matrix $H$ induces the following orthogonal decomposition of $T_{x^*}M$:
\begin{equation}
\label{eq:TxM_H_decomposition}
T_{x^*}M = \R v \oplus \ran H, \qquad \ker H = \R v .
\end{equation}
Let $P = \inner{v}{\cdot}\,v$ and $Q = I-P$ be the associated orthogonal projections -- with $P$ and $Q$ projections on $\ker H$ and $\ran H = (\ker H)^\perp$, respectively. By hypothesis, $H|_{\ran H}$ is invertible. Making use of \eqref{eq:TxM_H_decomposition}, we can write a tangent vector as
\begin{equation}
  \xi = s\,v + y, \qquad s\in\R,\ \ y\in\ran H \subset T_{x^*}M,
  \label{eq:deco}
\end{equation}
Now let us consider the critical-point
equation for the pulled-back family,
\begin{equation}
\label{eq:full_bifurcation_equation_H}
  F(s,y,t,\eps) := \nabla_\xi f(s v + y,\, t) = 0 .
\end{equation}
As first step, we obtain the bifurcation equation in the kernel subspace by means of a Lyapunov--Schmidt reduction. To this end, consider the equation $Q F = 0$. By (C0), we have $Q F(0,0,t^*,0) = 0$ and the partial Jacobian in $y$ is $Q\,\nabla^2_\xi f(0,t^*)\,Q = Q H Q = H|_{\ran H}$, which is invertible by (C1). Therefore, by the implicit function theorem there is a smooth solution $y = y(s,t,\eps)$ near $(0,t^*,0)$ with $y(0,t^*,0)=0$, solving
$Q F(s,y(s,t,\eps),t,\eps)=0$. Substituting $y$ in \eqref{eq:full_bifurcation_equation_H}, we obtain a reduced bifurcation equation on $\operatorname{Ker}(H)$. Now define the scalar function
\begin{equation*}
  \phi(s,t,\eps) := \inner{v}{F(s,\,y(s,t,\eps),\,t,\,\eps)} .
\end{equation*}
A fold of $\Phieps$ (or equivalently of $f$) corresponds to the condition $\phi = \partial_s\phi = 0$. Let us consider a third order Taylor-expand of $\phi$ about the point $(0,t^*,0)$. Since $(d\exp_{x^*})_0 = \mathrm{Id}$ and $\exp_{x^*}(0)=x^*$ for every
$t$, we find $\nabla_\xi f(0,t) = d\Phieps(x^*,t)$ and
$\nabla_\xi\big(\partial_\eps f\big)(0,t^*) = \grad\Psi(x^*,t^*)$; Furthermore, the third derivative along $v$ satisfies
$\partial_s^2 f\text{-term} = \nabla^3\Phizero(v,v,v)$. Using these data together with
$\nabla_\xi f(0,t^*)=0$ and $\inner{v}{H\,\cdot}=\inner{Hv}{\cdot}=0$, we find:
\begin{align*}
  \phi(0,t^*,0) &= 0, \\
  \partial_s\phi(0,t^*,0)
    &= \inner{v}{H v} = 0, \\
  \partial_s^2\phi(0,t^*,0)
    &= \nabla^3\Phizero(x^*,t^*)(v,v,v) = A, \\
  \partial_t\phi(0,t^*,0)
    &= \inner{v}{\partial_t\grad\Phizero(x^*,t^*)} = B, \\
  \partial_\eps\phi(0,t^*,0)
    &= \inner{v}{\grad\Psi(x^*,t^*)} =: a_0 .
\end{align*}
The last line uses that $\inner{v}{H\,\partial_\eps y} = \inner{Hv}{\partial_\eps y}=0$. Writing $a_1 := \partial_s\partial_\eps\phi(0,t^*,0)$ for the mixed
coefficient, the expansion reads
\begin{equation}
  \label{eq:reduced}
  \phi(s,t,\eps)
  = \frac{A}{2}\,s^2 + B\,(t-t^*) + a_0\,\eps + a_1\, s\,\eps
    + O\!\big(s^3, (t-t^*)^2, \eps^2, s(t-t^*)\big)
\end{equation}
where we wrote explicitly only the terms we shall need later. Imposing $\partial_s\phi = 0$ yields the equation
\begin{equation*}
  A\,s + a_1\,\eps + O(s^2, s\eps, \eps^2) = 0
\end{equation*}
which, using $A\neq 0$ from (C2), gives the following first order expression for the spatial displacement $ s^*(\eps)$ of the critical point:
\begin{equation}
\label{eq:spatial_displacement}
 s^*(\eps) = -\frac{a_1}{A}\,\eps + O(\eps^2).
\end{equation}
Substituting $s=s^* (\eps)$ into $\phi = 0$ and
using \eqref{eq:reduced}, the term $\tfrac{A}{2}s^*t^2$ is $O(\eps^2)$ and
the cross term $a_1 s^*\eps$ is $O(\eps^2)$ so, at first order, the speciation time is the solution of the equation
\begin{equation*}
B\,(t-t^*) + a_0\,\eps + O(\eps^2) = 0 .
\end{equation*}
Solving for $t$ with $B\neq 0$ (C3) gives
\begin{equation*}
t^*(\eps) = t^* - \frac{a_0}{B}\,\eps + O(\eps^2),
\end{equation*}
which is \eqref{eq:displacement}. Smoothness of $x^*(\eps),t^*(\eps)$
follows from the implicit function theorem applied to the system
$\{\phi=0,\ \partial_s\phi=0\}$ in $(s,t)$, whose Jacobian at the base point is
\[
  \begin{pmatrix} \partial_s\phi & \partial_t\phi \\
                  \partial_s^2\phi & \partial_s\partial_t\phi \end{pmatrix}
  \Bigg|_{0}
  = \begin{pmatrix} 0 & B \\ A & \ast \end{pmatrix},
  \qquad \det = -AB \neq 0 .
\]
The non-vanishing determinant guarantees that the displaced critical point remains an $A_2$ fold for $\varepsilon$ in some interval $(-\varepsilon_0,\varepsilon_0)$. By definition of speciation, the other eigenvalues of the unperturbed Hessian are negative, therefore by the $\varepsilon$-continuity of the Hessian $H \Phi_\varepsilon$ there is an interval $(-\widetilde \varepsilon_0,\widetilde \varepsilon_0)$ in which they do not change sign. Restricting $|\varepsilon|<\varepsilon_0$ so that $s^*(\varepsilon)$ and $y(s,t^*,\varepsilon)$ remains inside this neighborhood yields the thesis.
\end{proof}

\begin{remark}
\label{rem:A3_to_A2}
Since perturbing a pitchfork bifurcation produces a saddle-node, we expect that $A_3$ folds are generically not preserved and instead unfold into $A_2$ folds.
\end{remark}

As mentioned above, in DMs the perturbation of interest is a score-approximation error. Suppose thus that the
learned score is $\widehat{S}_t = S_t + E$, with $E$ a smooth score error due to the network approximation. This scenario -- under the hypothesis that the error $E$ is the gradient of some scalar function -- corresponds to $S_t = \nabla_x \Phi_0 $ and $\varepsilon \nabla_x \Psi = E$,
and Proposition~\ref{thm:a2_perturbation} specializes as follows.

\begin{corollary}
\label{cor:score}
Under Assumption~\ref{ass:fold}, with score error $E$ as above small enough, the speciation time shifts at first order by
\begin{equation}
  \label{eq:score-shift}
  \delta t^*
  = t^*(\eps) - t^*
  \approx -\,\frac{g(v,E(x^*,t^*))}{B},
  \qquad B = g(v,\partial_t S_{t}(x^*)|_{t=t^*}).
\end{equation}
with $v$ as per Assumption~\ref{ass:fold}.
\end{corollary}
\noindent While this result does not provide a quantitative estimate since in practice we do not know the exact score, nonetheless it guarantees that small errors in its approximation do preserve the existence of the speciation events with bounded shifts in time and location. In Sect.~\ref{subsec:learned_trimodal} we verify the estimate \eqref{eq:score-shift} with a known distribution.
\begin{remark}
If the error is orthogonal to the critical direction at the speciation point, we have $\inner{v}{E(x^*,t^*)} = 0$ and therefore $\delta t^* = o(\eps)$. In other words, to first order the speciation time is insensitive to score error transverse to $v=\ker(H)$.
\end{remark}

\section{An illustrative case: data from vMF distributions on the $\mathbb{S}^2$ manifold}
\label{sec:illustrative_s2}

We consider here different mixtures obtained from vMF distributions.  
Details about the vMF distribution and on the
procedure for obtaining an analytic solution of the Fokker--Planck equation with $p_0=p_{\text{vMF}}$ are provided in Appendix~\ref{app:appA}.

\subsection{Bimodal distribution}
\label{sec:bimodal_s2}
We consider a mixture of vMF distributions $p_0^{(1)}(x)$ and $p_0^{(2)}(x)$ with the same concentration parameter $\kappa$ and two different mean directions $\mu _1$ and $\mu _2$, respectively.
We set $p_0(\textbf{x}) = w_1 p_0^{(1)}(\textbf{x}) + w_2 p_0^{(2)}(\textbf{x})$, with positive weights $w_1+w_2=1$ and, for simplicity, we set $w_1=w_2=1/2$ and $\beta=\pi$ such that the centers $\mu _1 = (-\sin(\alpha) , 0, \cos(\alpha))$ and $\mu _2=( - \sin(\alpha) , 0, - \cos(\alpha)) $ are  symmetric with respect to the equator of the sphere. 
 By linearity of the Fokker-Planck equation, $p_t = p_t^{(1)} + p_t^{(2)}$, where $p_t^{(1,2)}$ are given by the corresponding form of \eqref{eq:p_sphere_expansion}.  The complete solution then reads
\begin{equation}\label{eq:bimodal_expansion_sphere}
p_t(\theta,\phi)=\sum_{\ell =1}^ \infty
\frac{1}{4\pi} (2\ell +1) c_\ell(\kappa){\rm e}^ {-\ell(\ell+1)t/2} \dfrac{P_\ell(u_+(\theta,\phi)) + P_\ell(u_-(\theta,\phi))}{2}
\end{equation}
where $u_\pm = \pm \cos\theta\cos\alpha - \sin\theta\sin\alpha\cos\phi.$ Figure~\eqref{fig:bimodal_pt} shows the plots of $p_t(\theta,\phi=\pi)$ at different times, for a bimodal distribution with well separated centers (left) and more closely spaced centers (right).
Notice that the section $\phi=\pi$ corresponds to the grand circle passing through the centers of the initial distributions, on which the geodesic connecting them is located.
\begin{figure}[h]
\centering
\includegraphics[width=0.95\textwidth]
{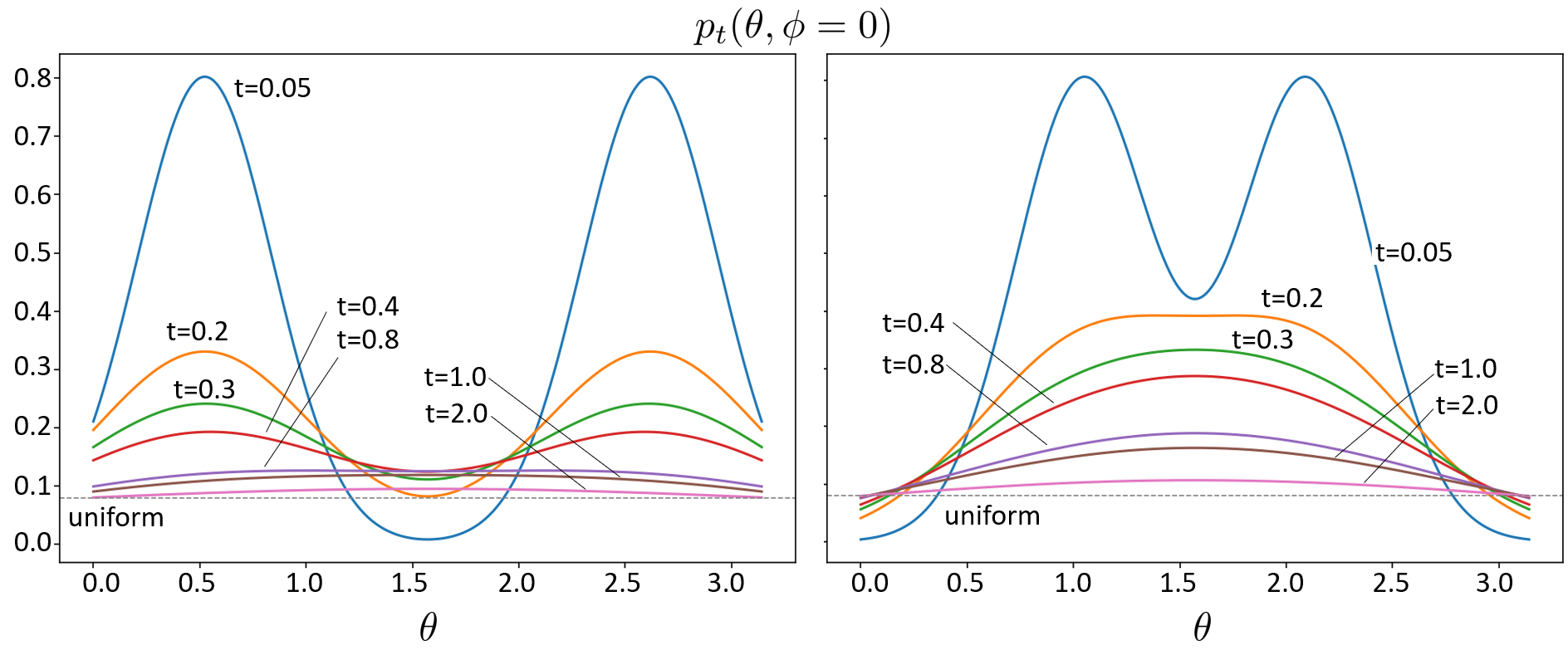}
\caption{Distribution $p_t$ as a function of the angle $\theta$ drawn at the section $\phi=\pi$, for an initial bimodal vMF density. Read along the reverse process,
the curves show the evolution from large $t$ — where $p_t$ is essentially the Riemannian uniform distribution $1/4\pi \approx 0.08$ — down to $t=0$. As $t$ decreases, the two data modes emerge as maxima, while the critical point at $\theta=\pi/2$ changes nature, from a maximum to a minimum separating the two data modes. Left: modes at $\alpha=\pi/6$; the change of nature at $\theta=\pi/2$ occurs around $t \approx 0.8-0.9$. Right: modes at $\alpha=\pi/3$; since the two centers are closer, the transition at the equator is more proximal to initial time, around $t \approx 0.2$. In both panels $\theta=0$ and $\theta=\pi$ remain local minima (geometrical modes).}
\label{fig:bimodal_pt}
\end{figure}

\noindent Following again the computations of~\ref{app:appA}, we compute the score of the bimodal vMF distribution and  then we seek numerically 
the locus $\Xi$ of its zeros, reported in Figure~\ref{fig:sphere_bifurcation}
(left: $\alpha=\pi/6$, right:  $\alpha=\pi/3$) 
again in the section $\phi=\pi$. Pitchfork bifurcations emerge as expected in symmetric scenario and,  
in accordance with~\eqref{eq:bimodal_speciation_time}, the speciation time decreases as the angle $\alpha$ increases, since $D=\pi-2\alpha$. 
Notice that at the speciation points, proceeding backward in time from $t^*$, we have the nullcline equations $\theta \propto \sqrt{t^\star-t}$, in agreement with the normal form~\eqref{eq:cerf_score}.
\begin{figure}[h]
\centering
\includegraphics[width=0.55\textwidth]{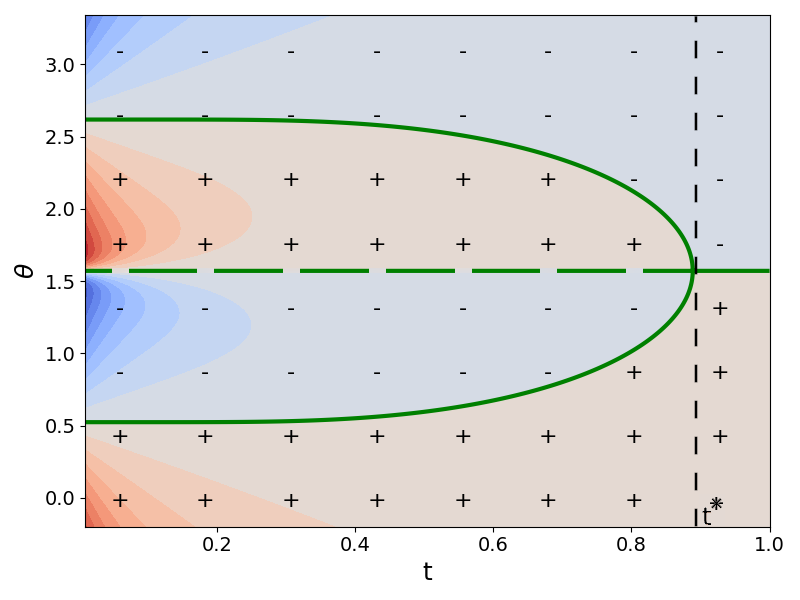}
\includegraphics[width=0.55\textwidth]{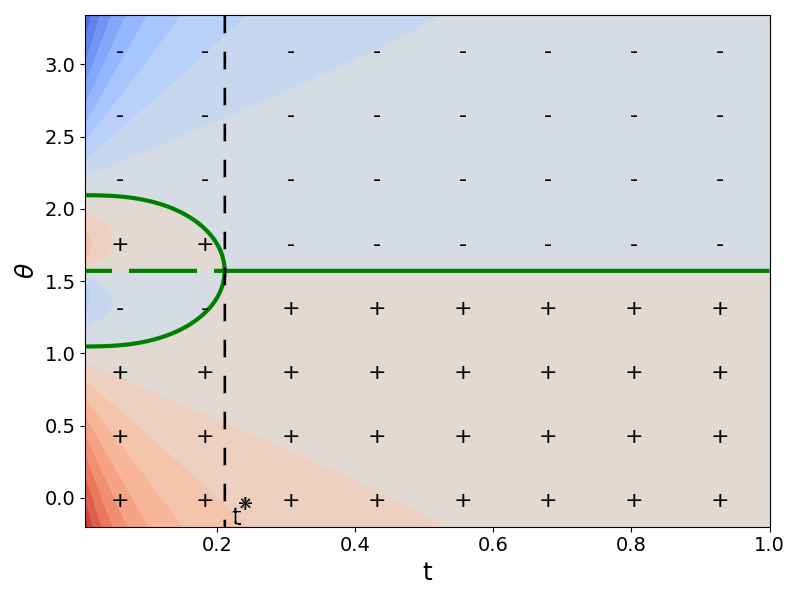}
\caption{Bimodal vMF distribution,  analytic score. Sign (symbols), values (colors) and zeroes (in green) of the $\theta$ component of the score, restricted on the great circle ($\phi=\pi$) connecting the centers of the bimodal distribution  for $\alpha = \pi/6$ (top) and $\alpha = \pi/3$ (bottom). The speciation times $t^* \approx 0.89$ (top) and $t^* \approx 0.21$ (bottom) are in correspondence with the critical points of Figure~\ref{fig:bimodal_pt} and are in good agreement with their theoretical estimates $t^* \approx 0.99$ and $t^* \approx 0.18$, obtained via \eqref{eq:bimodal_speciation_time} -- see Appendix~\ref{sec:vmf_asymptotics} for an estimate of the parameter $\sigma^2$ in the case of a vMF distribution. For small times $t$, the attractors correspond to the centers of the initial bimodal distribution. In both cases, after speciation the previous attractor $\theta=\pi/2$ -- the equator -- becomes a repeller (dashed line). Notice that, for readability, the two panels have different time ranges and colors are mapped on different scales.}
\label{fig:sphere_bifurcation}
\end{figure}

\null 
\noindent 
In Figure~\ref{fig:sphere_score_norm} we report the results of the numerical study of the stability of the points in $\Xi$ for different times
for the case $\alpha=\pi/6$. Notice the presence of the geometrical modes at $\theta=\pi/2$, 
$\phi=0$ and $\phi=2\pi$, which are always  minima of the distribution. 
The eigenvalues of the Hessian matrix of $p_t$ are shown in Figure~\ref{fig:bimodal_hessian}.

\begin{figure}[h!]
    \centering
    \includegraphics[width=.8\linewidth]{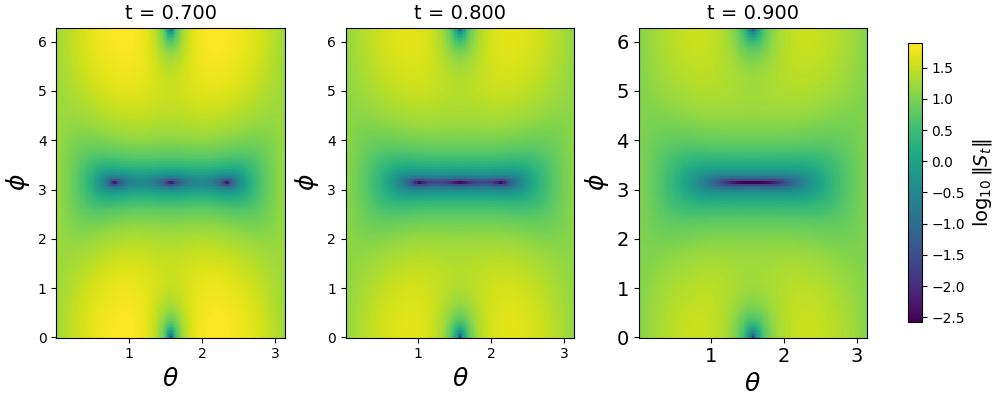}
    \includegraphics[width=.8\linewidth]{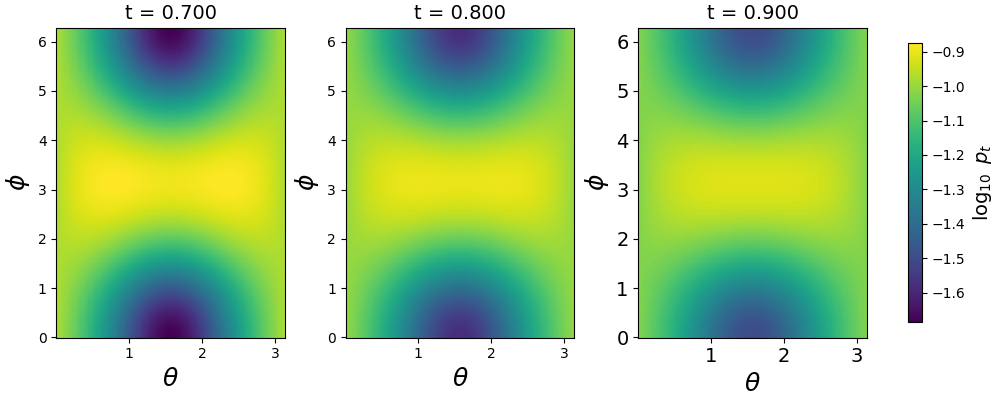}
    \caption{Bimodal vMF distribution, analytic score. Study of the stability of the equilibrium points of the score for the case $\alpha = \pi/6$, in relation to the critical points of $p_t$. Top: Logarithm of the Riemannian norm of the score $S_t$ for different times. The centroids of the dark blue regions correspond to the real zeros of $S_t$, and therefore to the critical points of $p_t$. At the speciation time $t^* \approx 0.89$ the new zeroes of the score emerge via a pitchfork bifurcation from the original critical point, see also Figure~\ref{fig:sphere_bifurcation}. Bottom: Logarithm of $p_t$ at different times. Before  speciation  ($t>t^* \approx 0.89$) the critical point at $\theta = \pi/2$ is a maximum and therefore an attractor. After speciation ($t<t^* \approx 0.89$) this point becomes a saddle, thus an hyperbolic point, while the new critical points at $\theta = \pi/2 \pm \pi/3$ are maxima -- stable equilibrium points of $S_t$. The minima at $\theta = \pi/2$, $\phi=0$ and $\phi=2\pi$ are geometric modes.}
    \label{fig:sphere_score_norm}
\end{figure}


\begin{figure}[h!]
    \centering
        \includegraphics[width=0.75\textwidth]{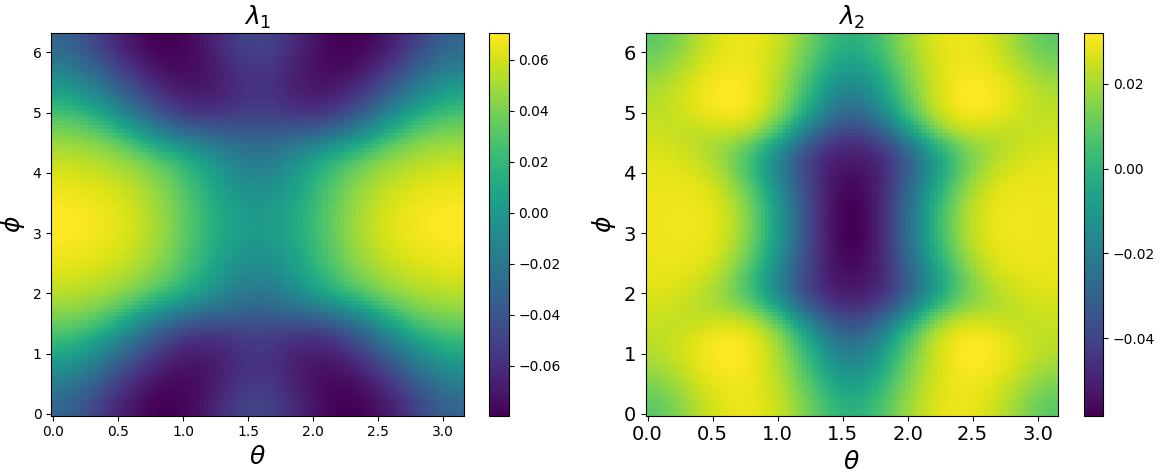}
    \caption{Bimodal vMF distribution, analytic score. Plot of the two eigenvalues $\lambda_{1,2}$ of $\operatorname{Hess} p_t$ at the critical time $t^*$. Notice that at the corresponding critical point $(\theta^*,\pi^*)=(\pi/2,\pi)$  the eigenvalue $\lambda_1$ vanishes,
    while $\lambda_2$ is negative, coherently with the fact that at speciation $\ker(\operatorname{Hess}p_t)$ is one dimensional.}
    \label{fig:bimodal_hessian}
\end{figure}

\newpage

\subsection{Multimodal  vMF distribution}
\label{sec:trimodal_theoretical}
We consider a mixture of three vMF distributions $p_0^{(1)},p_0^{(2)},p_0^{(3)}$ so that
$p_0 = \tfrac13\big(p_0^{(1)}+p_0^{(2)}+p_0^{(3)})$.
The centers are located at the vertices of a spherical isosceles triangle $C_1=(\tfrac{\pi}{6},\tfrac{3\pi}{4})$, $C_2=(\tfrac{\pi}{6},\tfrac{5\pi}{4})$ and $C_3=(\tfrac{5\pi}{6},\pi)$ respectively. In this case, the study of $\Xi$ (reported in Figure~\ref{fig:trimodal_bif}) is more complex: there is a first speciation time at $t^*_1 \approx 0.6$, at which the center $C_3$ becomes an attractor, and a second one at $t^*_2 \approx 0.12$, at which also $C_1$ and $C_2$ become attractors. The first bifurcation is a saddle-node, while the second one is a pitchfork. Figures \ref{fig:trimodal_s2_first_speciation} and \ref{fig:trimodal_s2_second_speciation} report the
numerical study of the stability.

\begin{figure}[h!]
\centering
        \includegraphics[scale=0.39]{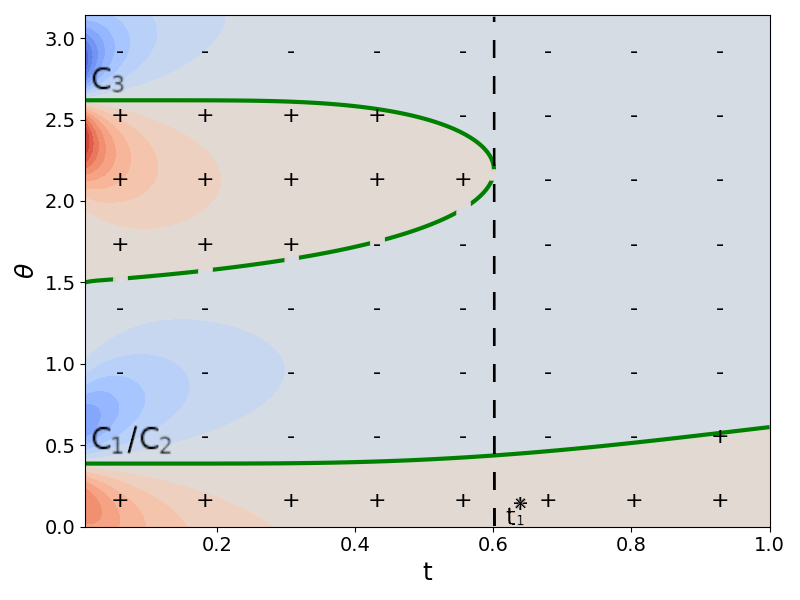}
        \includegraphics[scale=0.4]{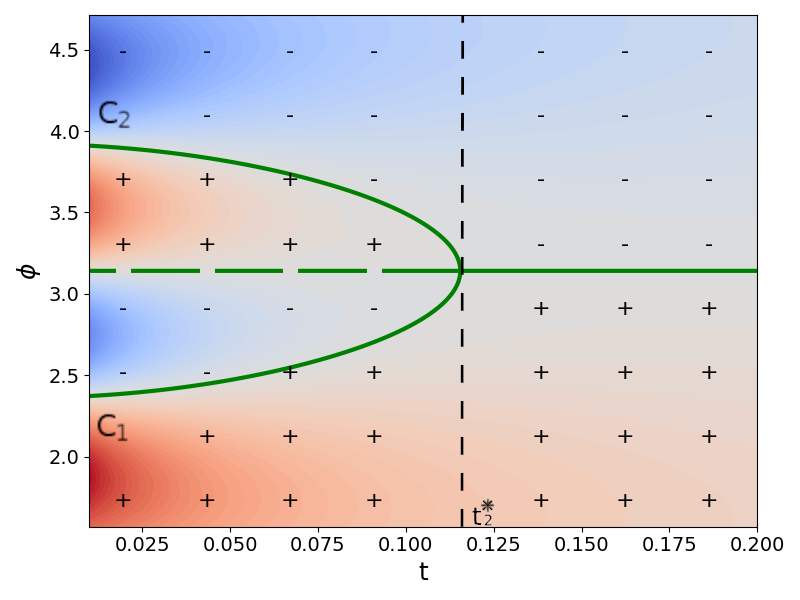}
\caption{Trimodal vMF distribution, analytic score. Plot of the sign (symbols), value (colors) and zeroes (in green) of the components of the score. Top: $\theta$-component at the section $\phi=\pi$. Near the first speciation time
$t^*_1$, the late time branch is shifted toward  $\theta = \pi/6$, the common polar angle of 
$C_1$ and $C_2$ and a saddle-node bifurcation generates the attractor of $C_3$ at $\theta = 5\pi/6$. Bottom: $\phi$-component at the section $\theta=\pi/6$. At the second speciation time $t_2^* \approx 0.12$ a pitchfork bifurcation separates $C_1$ and $C_2$. Notice that, for readability, the two panels have different time ranges and colors are mapped on different scales.
}\label{fig:trimodal_bif}
\end{figure}

\begin{figure}[h!]
\centering

        \centering
        \includegraphics[width=0.9\textwidth]{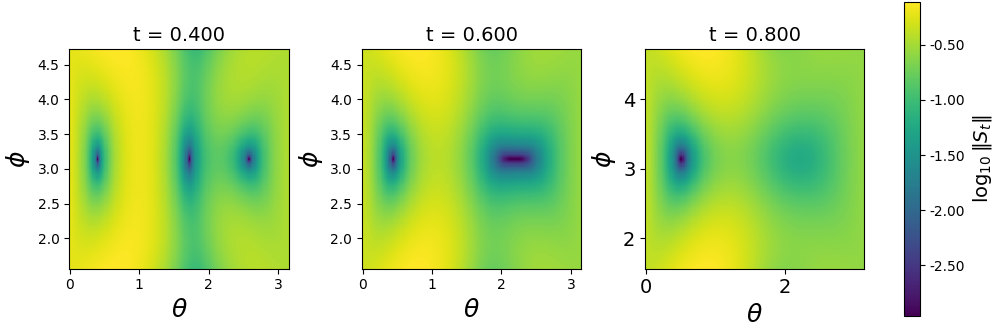}

        \centering
        \includegraphics[width=0.9\textwidth]{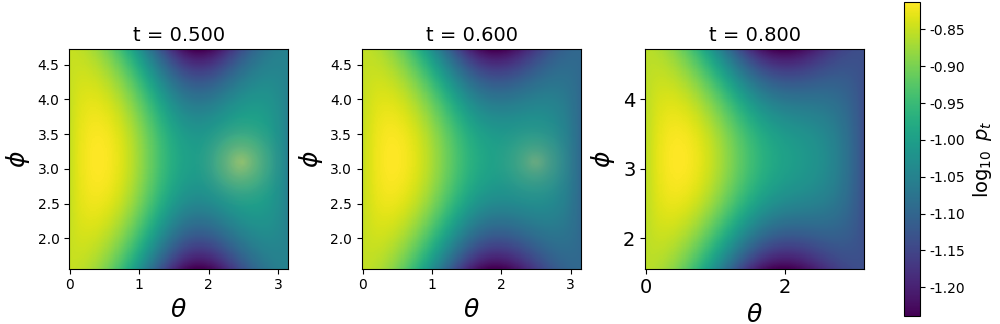}

    \captionsetup{skip=3pt}
\caption{Trimodal vMF distribution, analytic score. 
Behavior near the first speciation (for clarity only angles $\phi \in (\pi/2, 3\pi/2)$ are shown). Top: logarithm of the Riemannian norm of $S_t$. Bottom: logarithm of $p_t$. At $t_1^* \approx 0.6$, the score vanishes at the center $C_3$, which is an attractor. At this time, a second attractor appears at the midpoint of the geodesics through $C_1$ and $C_2$, which are not yet distinguished. The two attractors are separated by a saddle near $(\tfrac{7\pi}{12},\pi)$. This corresponds to a saddle-node bifurcation: the original attractor at $\theta=\pi/6$ persists, while a saddle-node pair emerges near $(\tfrac{5\pi}{6},\pi)$, see also Figure~\eqref{fig:trimodal_bif}. The attractor of the $C_1$-$C_2$ cluster is a geometric mode forced to exist by Poincar\`e-Hopf equality~\eqref{eq:morse} after the $C_3$ speciation.
}\label{fig:trimodal_s2_first_speciation}
\end{figure}

\begin{figure}[h!]
\centering

        \centering
        \includegraphics[width=0.95\textwidth]{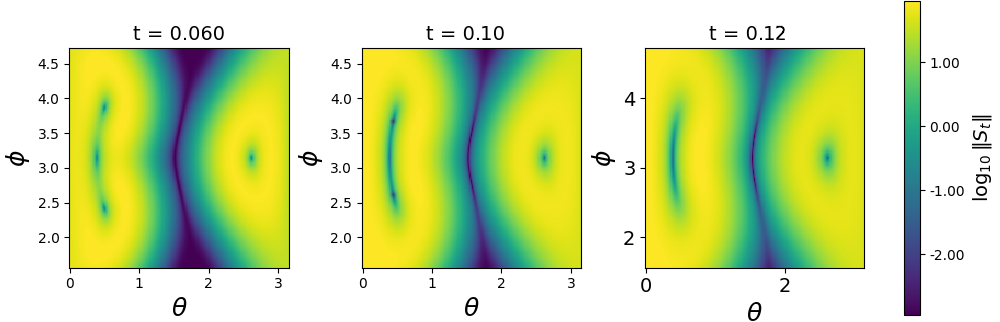}

        \centering
        \includegraphics[width=0.95\textwidth]{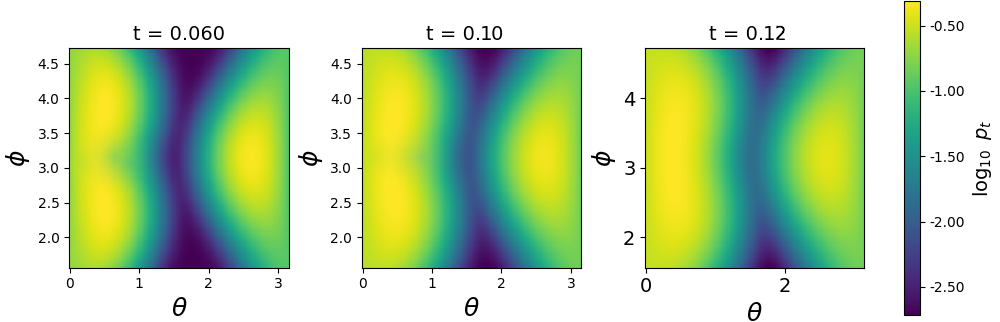}

\caption{Trimodal vMF distribution, analytic score. Behavior after (left), near (middle) and in correspondence (right) of the second speciation at $t^*_2 \approx 0.12$. For clarity only latitudes $\phi \in (\pi/2, 3\pi/2)$ are shown. Top row: logarithm of the Riemannian norm of $S_t$. Bottom row: logarithm of $p_t$. The coarse-grained geometric mode appeared after the first speciation time at $(\pi/6,\pi)$, representing the $C_1$-$C_2$ cluster, is now sifted into two separate attractive basins.
}\label{fig:trimodal_s2_second_speciation}
\end{figure}

\newpage

\section{Experiments with score approximated via neural networks}
\label{sec:numerical_experiments}

In this section we 
present numerical experiments where we
use a neural network to approximate the score. 
Before presenting them, we deem interesting to provide some details
regarding the numerical techniques that we used to obtain the computational results. 

\subsection{Numerical techniques}
\label{sec:numerical_techniques}

The simulation of diffusion on manifolds has been studied in literature by methods based on the exponential map, such as the Geodesic Random Walks (GRWs) of \cite{debortoli2022}. These methods evolve points by sampling tangent vectors and mapping them back to the manifold through the exponential map (or approximating it trough a retraction). Here we exploit the knowledge of the pullback and of the intrinsic dimension of the manifold, to operate directly in a chart: this avoids the computationally expensive projection on the manifold via the exponential map. 

\subsubsection{Numerical integration of the forward SDE}
\label{subsec:forward_sde_integration}
 We integrate the forward SDE using the Euler--Maruyama scheme \cite{NumericalSolutionSDE}. The full integration procedure is depicted in Algorithm~\ref{algo:forwardsde}. If one disposes of an 
 analytic expression of the metric $g$ in coordinates, this latter can be directly used. More in general, if the manifold is embedded in an ambient space,
  the knowledge of the embedding map~$F$ (for example obtained via a network or an embedding algorithm) allows to leverage automatic differentiation to numerically compute~$g=J_F^T J_F$, with $J_F$ Jacobian of the embedding map. 

\begin{algorithm}[htb]
\caption{Forward Riemannian Diffusion (FRD) in Local Coordinates}
\label{algo:forwardsde}
\begin{algorithmic}[1]

\Require Metric $g$, initial point $x_0 \in \mathbb{R}^n$ (local coordinates), time horizon $T$, number of steps $N$

\State $\Delta t = T / N$
\State $x \leftarrow x_0$

\For{$k = 0, \dots, N-1$}

    \State $t_k = k \Delta t$
    
    \State Compute $ g^{ij}(x), \Gamma^i_{jk}(x)$
    
    \State Compute geometric drift: 
    $b^i(x) = \Gamma^i_{jk}(x)\, g^{jk}(x)$
    
    \State Sample Gaussian noise:
    $Z_{k+1} \sim \mathcal{N}(0, I_d)$
    
    \State Euler--Maruyama update:
    $x^i \leftarrow x^i
    + b^i(x)\, \Delta t/2
    + \sqrt{\Delta t}\, Z^i_{k+1}$
    
    \State \textbf{(Optional)} Apply periodic conditions/Change chart (see Remark~\ref{remark:charts_boundary_conditions})

\EndFor

\hspace*{-1.2cm}\Return $x$

\end{algorithmic}
\end{algorithm}

\begin{remark}
\label{remark:charts_boundary_conditions}
Working in a local chart, the Euler--Maruyama update of Algorithm~\ref{algo:forwardsde} may produce a point outside the chosen chart. In the case of an almost global chart, as the spherical polar coordinates for the unit sphere, this problem can be solved applying suitable periodic boundary conditions. In general, a change of chart is needed. If the intrinsic manifold of the data is unknown, an atlas of local charts can be learned e.g. using a mixtures of VAEs whose encoders serve as local charts of the manifold \cite{alberti2024manifold}.  
In our previous work~\cite{causin2025estimatingdatasetdimensionsingular}
we provided an estimation algorithm for the intrinsic dimension of the manifold to guide the construction of the atlas.
\end{remark}

\subsection{Training of the Score Network}

Score-based generative models rely on learning the Stein score $S(x,t) = \nabla_x \log p_t(x)$. Since obtaining $p_t$ directly is often unfeasible, a neural network $S_\Theta(x,t)$ is trained by learning the set of parameters $\Theta$ to approximate $\nabla_x \log p_t(x)$ via score matching \cite{debortoli2022}. 
The key identity underlying the training of the network is the denoising score matching (DSM) identity:
\begin{equation*}
\begin{split}
\nabla_x \log p_t(x)
= \int_{M} 
\nabla_{x_t} \log p_{t|s}(x_t \mid x_s)
\, p_{s|t}(x_s \mid x_t)
\, d\mathrm{Vol}_{M}(x_s)
\end{split}.
\end{equation*}
From this identity, the score can be characterized as the minimizer of:
\begin{equation*}
\mathcal{L}_t(S_\Theta) 
= \mathbb{E}_{X_0, X_t} \left[ 
\left\| S_\Theta(X_t,t) - \nabla_{x_t} \log p_{t|0}(X_t \mid X_0) \right\|^2
\right].
\end{equation*}
To learn the score on a interval $[0,T]$, the full training objective integrates over time:
\[
\mathcal{L}(\theta)
= \int_0^T \zeta(t)\, \mathcal{L}_t(S_\Theta)\, dt,
\]
where $\zeta(t)$ is a weighting function.
Under suitable regularity assumptions on $p_{t|s}(x_t|x_s)s(x_t)$, an equivalent intrinsic formulation which avoids the explicit computations of the gradients of the transition densities can be given \cite{debortoli2022}, with implicit score matching  loss:
\begin{equation*}
\mathcal{L}_{\mathrm{ISM}}(S_\Theta)
=
\mathbb{E}_{t \sim [0,T], X_t}
\left[
\frac{1}{2} \|S_\Theta(X_t,t)\|^2 + \mathrm{div}S_\Theta(X_t,t)
\right].
\end{equation*}

Given these components, the training procedure follows the standard score-matching pipeline: first, we generate noisy samples by simulating the forward diffusion, as discussed in~\ref{subsec:forward_sde_integration}, then we optimize the score network $S_\Theta$ by minimizing $\mathcal{L}_{ISM}$ (see Algorithm~\ref{algo:training}).

\begin{algorithm}[ht]
\caption{Riemannian Implicit Score Matching Training}
\label{algo:training}
\begin{algorithmic}[1]

\Require Metric $g$, initial data distribution $p_0$, score network $S_\Theta$, simulation time $T$, learning rate~$\eta$

\For{iteration $k = 1, \dots, N_{\text{iter}}$}

    \State Sample $X_0 \sim p_0$
    
    \State Sample $t \sim \mathcal{U}([0, T])$
    
    \State Simulate diffusion on $M$: $X_t = \mathrm{FRD}(X_0, t)$ (see Algorithm~\eqref{algo:forwardsde})
    
    \State Evaluate score network : $v = S_\Theta(t, X_t) \in T_{X_t}M$
    
    \State Compute Riemannian norm: $\|v\|^2_{g} = v^\top g(X_t) v$
    
    \State Compute Riemannian divergence: 
    $
    \mathrm{div}_g(S_\Theta)(X_t)
    =
    \frac{1}{\sqrt{|g|}}
    \partial_i\left(\sqrt{|g|}\, v^i \right)$

    \State Compute loss:
    $\mathcal{L}_{ISM}
    =
    \frac{1}{2} \|v\|^2_{g}
    +
    \mathrm{div}_g(S_\Theta)(X_t)
    $
    
    \State Update: $\theta \leftarrow \theta - \eta \nabla_\theta \mathcal{L}$

\EndFor

\end{algorithmic}
\end{algorithm}

\subsection{Enhancing numerical stability}
Numerical computations in finite precision arithmetic are known to introduce rounding errors, loss of symmetry, and ill-conditioning \cite{HighamAccuracyStability,GolubMatrix}. We discuss here the mitigation strategies we adopted to face instabilities 
arising in the numerical integration of the forward SDE while learning the score. 

\paragraph{Metric Regularization}

Given the embedding map $F$, the pullback metric is $g = J_F^\top J_F$, where $J_F$
denotes the Jacobian of $F$. This matrix is symmetric positive semidefinite by construction, and positive definite whenever $J_F$
has full column rank. In practice, however, finite-precision arithmetic introduces two main numerical pathologies. First, floating-point rounding breaks the exact symmetry of $g$, since the $(i,j)$
and $(j,i)$ entries of $J_F^\top J_F$
are computed via different sequences of operations. Second, when $J_F$ has small singular values $\sigma_i$, the corresponding eigenvalues $\lambda_i = \sigma_i^2$
of $g$
are extremely small, since the pullback construction squares the singular values and therefore squares the condition number
$\mathrm{cond}(g) = \mathrm{cond}(J_F)^2$, 
which can make computations involving $g$
or its derivatives - such as the Christoffel symbols - highly unstable.
We address these issues with two complementary strategies:
\begin{itemize}
    \item Symmetrization: replacing $g$
with $\frac{1}{2}(g + g^\top)$
eliminates asymmetry due to floating-point errors exactly, at negligible computational cost
    \item Tikhonov regularization: adding a small multiple of the identity, $g \leftarrow g + \varepsilon \operatorname{Id}$
 for $\varepsilon > 0$ (e.g.
 $\varepsilon = 10^{-6}$), shifts all eigenvalues away from zero and ensures strict positive definiteness.
\end{itemize}
Combining both, we replace $g$ with
$\tilde{g} = \tfrac{1}{2}(g + g^\top) + \varepsilon \operatorname{Id}$.
The matrix $\tilde{g}$
is symmetric by construction and strictly positive definite for any $\varepsilon > 0$, since  $g$ is positive semidefinite, $x^\top g\, x \geq 0$ for all $x$, so
$$x^\top \tilde{g}\, x = \frac{1}{2} \left( x^\top g x + x^\top g^T x  \right) + \varepsilon \|x\|^2 \geq \varepsilon \|x\|^2 > 0 \qquad \forall\, x \neq 0.$$

\paragraph{Metric Inversion}
For small $\varepsilon$ in the Tikhonov regularization, the condition number of $g$ is $\operatorname{cond}(g) = \lambda_{\max} / \lambda_{\min} \approx \left(\sigma_{\max} / \sigma_{\min}\right)^2$, namely the the conditioning of the regularized metric tensor is approximately the square of the conditioning of the Jacobian. For large condition numbers, inverting $\tilde g$ directly is not numerically stable. To improve the numerical stability, we resort to a Cholesky factorization $\tilde g = LL^T$, which for semi-positive definite matrices -- such as $\tilde g$ -- is backward stable \cite{HighamAccuracyStability}.
Therefore, instead of computing $\tilde g^{-1}$ directly, we compute the inverse matrix as $\tilde{g}^{-1} = (L^{-1})^T L^{-1}$.

\paragraph{Volume element via Log-Det stabilization}
While the stability of the determinant is partially ensured by regularizing $g$, in case of small eigenvalues finite precision arithmetic may lead in any case to instabilities. Instead of computing the volume element directly as $\sqrt{\det(\tilde g)}$, we obtain this quantity using a Log-Det stabilization
\begin{equation*}
\log \det(\tilde g) = 2 \sum_i \log L_{ii}
\end{equation*}
where $L$ is the Cholesky factor of $\tilde{g}$. Then we find
$\sqrt{\det(\tilde g)} = \exp(\sum_i \log L_{ii})$. 
Log-determinant formulations are widely used in machine learning due to improved stability and scalability \cite{han2015large}. 

\paragraph{Eigenvalue Decomposition for Matrix Square Roots} The inverse metric square root $\tilde g^{-1}$ is computed via diagonalization, first computing $
\tilde g^{-1} = Q \Lambda Q^\top$ and then $
\sqrt{\tilde g^{-1}} = Q \sqrt{\Lambda} Q^\top$, with $\tilde g^{-1}$ obtained from the Choleski factorization of $\tilde g$. Stability relies on prior Tikhonov regularization of $g$ ensuring $\lambda_i \ge \varepsilon$.

\subsection{Numerical experiments}
We train a neural network to learn an approximation of the score from 
data points from the different  distributions addressed above. 
If not differently specified, we used a fully connected network with 7 layers and smooth activation functions. The code is available at \url{https://github.com/alessiomarta/speciation_theory_compact_riemannian_diffusion_models}.

\medskip 

\noindent We compare the numerical findings with the theoretical results, which ensure that limited perturbations in the score yield
bounded perturbations in the time and location of speciations.
We also address realistic datasets
with unknown analytic distribution.
The computed loci $\Xi$ 
are obtained by monitoring the change of the sign of the score on a fine regular grid.
In all the plots, the wiggles in the curves are dependent on the score actually learned by the net and on the discrete lattice used to monitor sign changes. 
For clarity, we do not show the geometric modes at the antipodes. 

\subsubsection{ Bimodal vMF distribution on $\mathbb{S}^2$}
In this experiment we consider points from the bimodal vMF distribution already discussed with analytical score in   Section~\ref{sec:bimodal_s2}. We begin our study for the $\alpha=\pi/6$ scenario. A generated trajectory is depicted in Figure~\ref{fig:globi} showing the formation of an elbow point in correspondence of the speciation time and the final commitment to target.
\begin{figure}[H]
    \centering
\includegraphics[width=.7\textwidth]
    {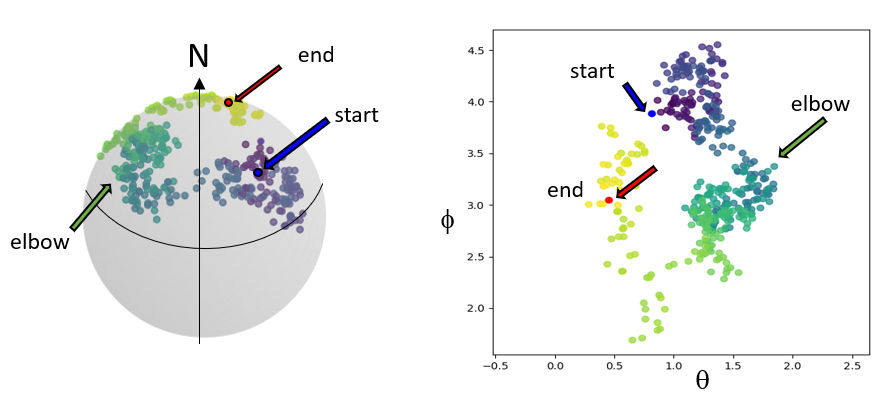}
    \caption{Bimodal vMF distribution, learned score. Trajectory generated from a point sampled at about $\pi/4$, $\phi=6/5 \pi$. First the trajectory heads to the equator, after speciation the equator becomes unstable and repels the process so that the trajectory commits to a point near the north pole. Time is parametrized by the colors of the standard viridis map.}
    \label{fig:globi}    
\end{figure}

Figure~\ref{fig:driftvectors} shows the learned score field before ($t = 1.6$) and after ($t=0.8$) the speciation time ($t^*  \approx 0.89$) and at the end of the reverse process ($t=0$). At the beginning of the generative process the trajectories are pushed towards the intersection of the geodesic connecting the centers of $p_0$ and the equator; after  speciation, the trajectories are attracted by the two original centers. The equilibrium point on the equator becomes unstable, consistent with a pitchfork bifurcation.

\begin{figure}[H]
    \centering
    \includegraphics[width=.475\textwidth]{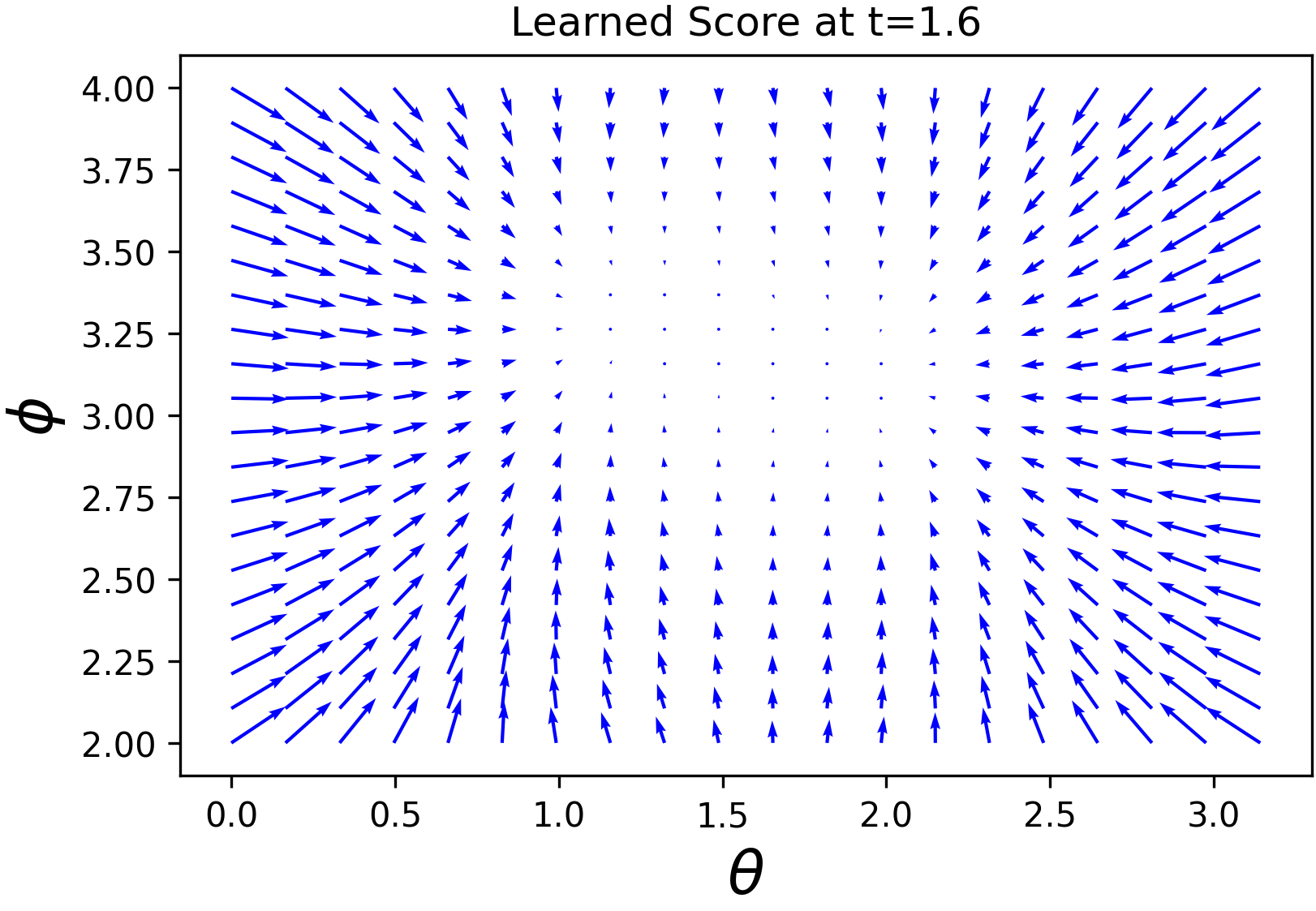} \quad 
    \includegraphics[width=.475\textwidth]{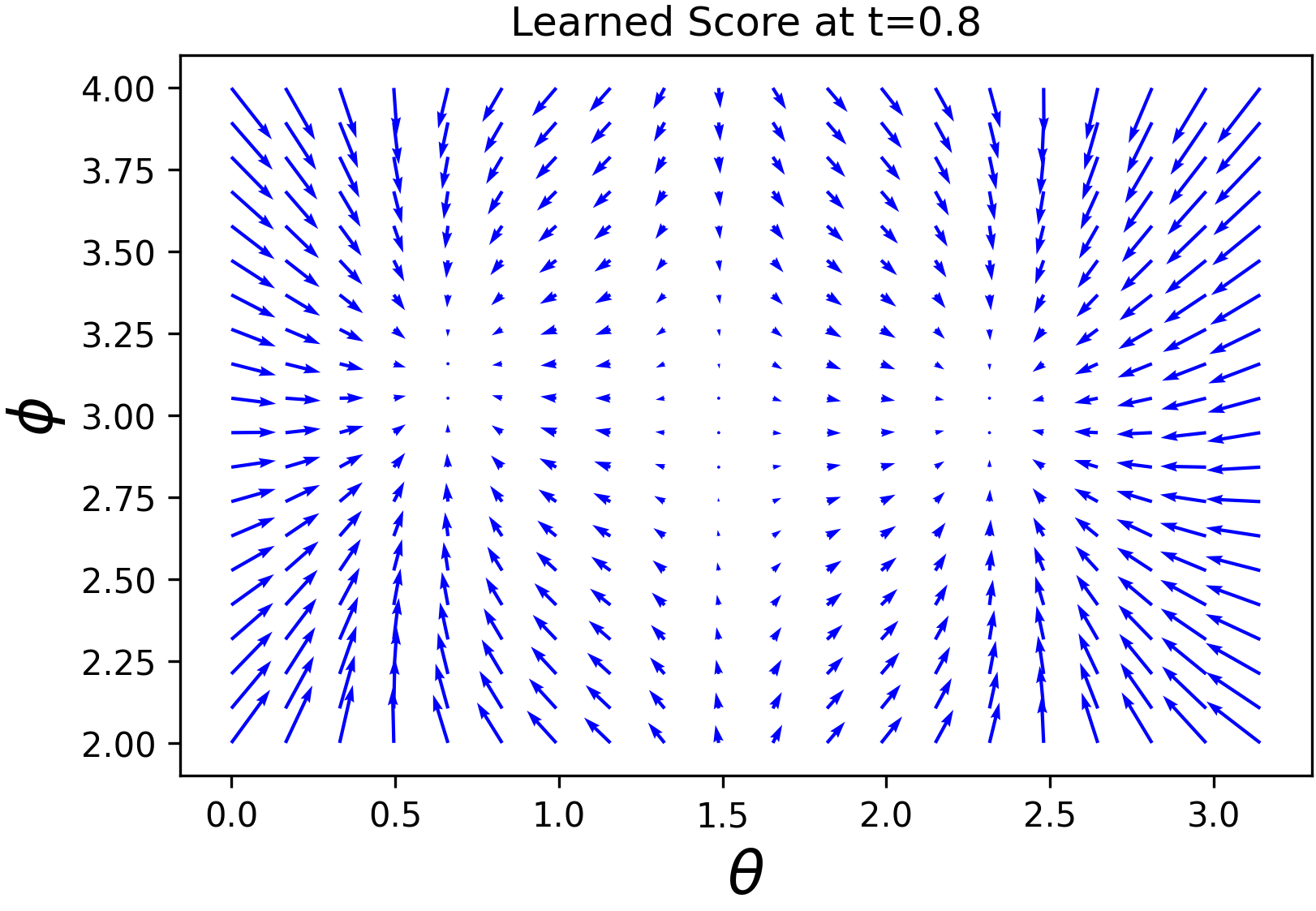}
    \includegraphics[width=.475\textwidth]{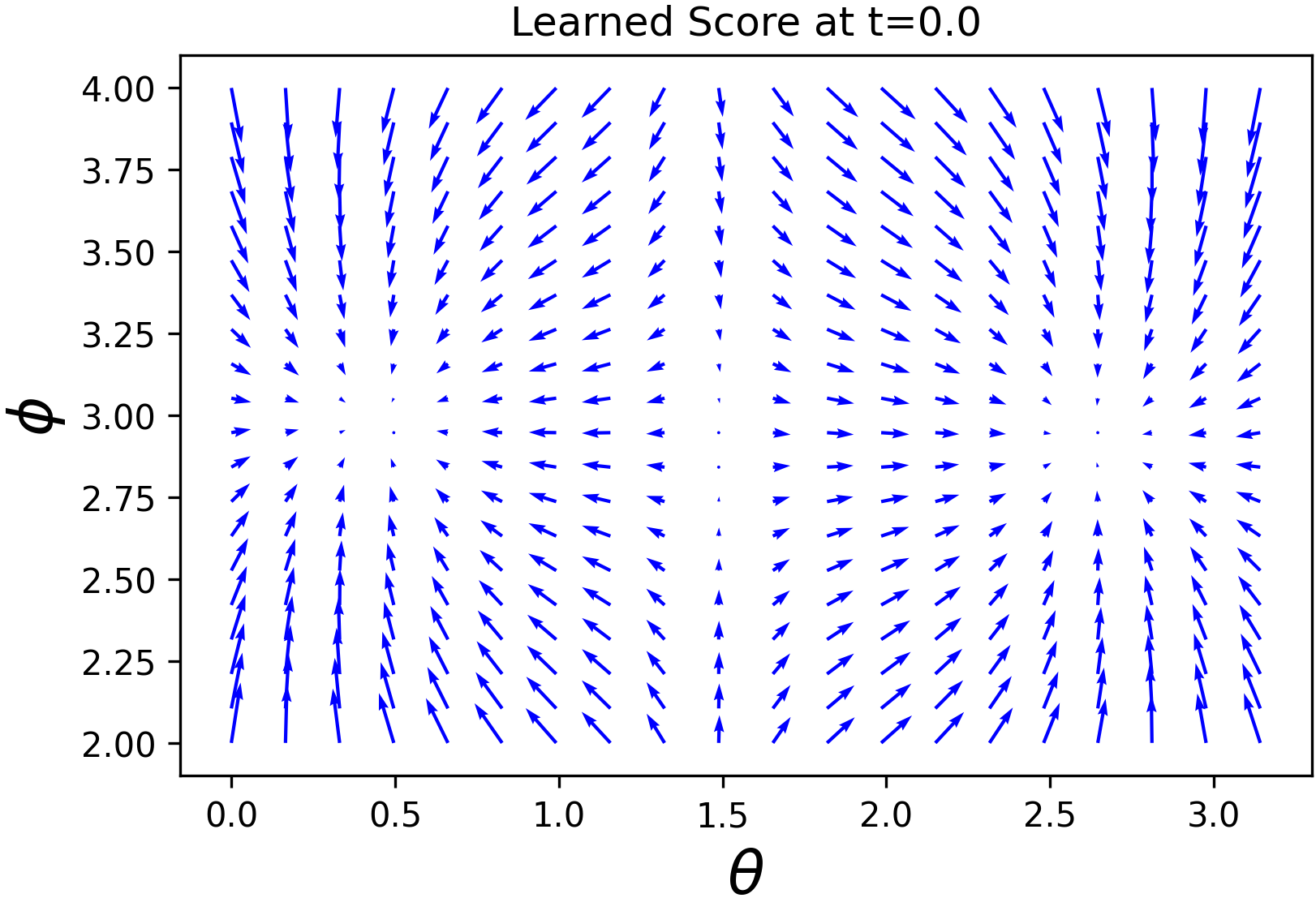}    
    \caption{Bimodal vMF distribution, learned score field before (top left), and after speciation time (top right) and at the end of the reverse process (bottom). For clarity only azimuthal angles $\phi \in (2,4)$ are shown.}\label{fig:driftvectors}
\end{figure}

Repeating the numerical experiment for the other configuration of the centers considered in Section \ref{sec:bimodal_s2} ($ \alpha=\pi/3$) we obtain again results in good agreement with the theoretical findings.
In Figure~\ref{fig:learned_bimodals_bifurcations} we show
the bifurcation diagram obtained
numerically via the learned score. 

\begin{figure}[h]
\centering
        \includegraphics[width=.49\textwidth]{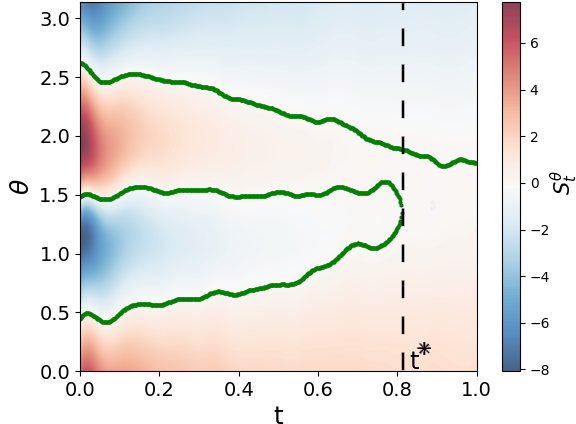}
        \includegraphics[width=.49\textwidth]{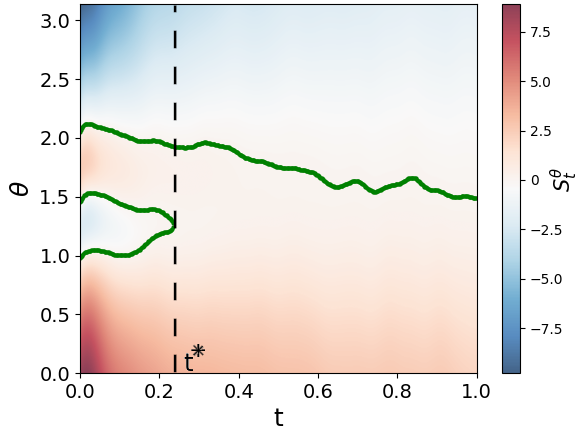}
\caption{Bimodal distribution, learned score. Zeros (green) and values (color) of the $\theta$ component of the  learned score for. Left ($\alpha=\pi/6$): the speciation time $t_{learned}^* \approx 0.86$ is in good agreement with the theoretical one ($t_{theoretical}^* \approx 0.89$) -- see Figure~\ref{fig:sphere_bifurcation}. Right ($\alpha=\pi/3$):  
the pitchfork unfolds into a $A_2$ bifurcation and the numerical speciation time is in good agreement with the theoretical speciation time ($t^\star_{learned} \approx 0.23$ vs $t^\star_{theoretical} \approx 0.21$)
In both cases, the approximation of the score given by the network breaks the perfect pitchfork, as discussed in Remark~\ref{rem:A3_to_A2}. The original stable branch persists, while a secondary stable (bottom branch of the pair)–unstable (top branch of the pair) pair is created through a saddle-node. 
}\label{fig:learned_bimodals_bifurcations}
\end{figure}

\subsection{Trimodal vMF distribution}
\label{subsec:learned_trimodal}
\paragraph{Equal weights}
In this experiment we consider points from the trimodal vMF distribution of Section~\ref{sec:trimodal_theoretical}. All the centers are given the same weight. In Figure~\ref{fig:approx_score_trimodal_triangle} we show the bifurcation diagram of the learned score, for its $\theta$ (left) and $\phi$ (right) components, respectively. The speciation happening at $t_1^*$ satisfies Assumption \ref{ass:fold}. 
At the analytic critical point $t^*=t_1^* \approx 0.6$ and $(\theta^*,\phi^*)=(7\pi/12,\pi)$ found in Section~\ref{sec:trimodal_theoretical}, we have $\partial_t S_t^{analytic}(x^*,t^*) \approx (10.9131,4.0304)$ and $E(x^*,t^*)  = S_t^{learned}(x^*,t^*)  - S_t^{analityc}(x^*,t^*) \approx (-0.2049,0.0200)$. Putting this together with $\operatorname{ker}\operatorname{Hess} (x^*,t^*) = \operatorname{span}((1,0))$, we can quantify  the discrepancy $\delta t^*$ between the theoretical and the learned speciation time via Corollary~\ref{cor:score}.  This yields $\delta t^* \approx 0.019$ and consequently the estimate $t^*_{learned} \approx 0.58$, which is in good agreement with the bifurcation time learned by the network (see Figure \ref{fig:approx_score_trimodal_triangle}, left panel).

\begin{figure}[h]
\centering
        \includegraphics[width=.49\textwidth]{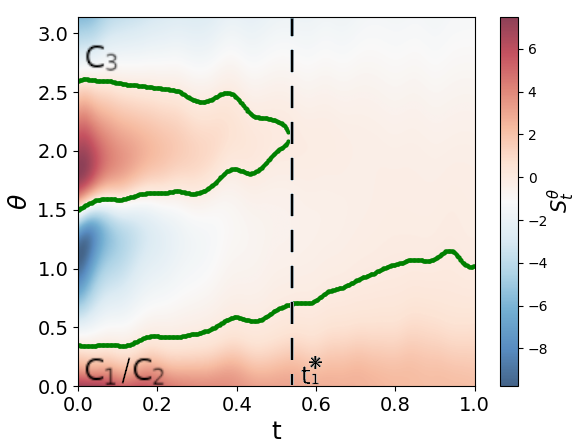}
        \includegraphics[width=.49\textwidth]{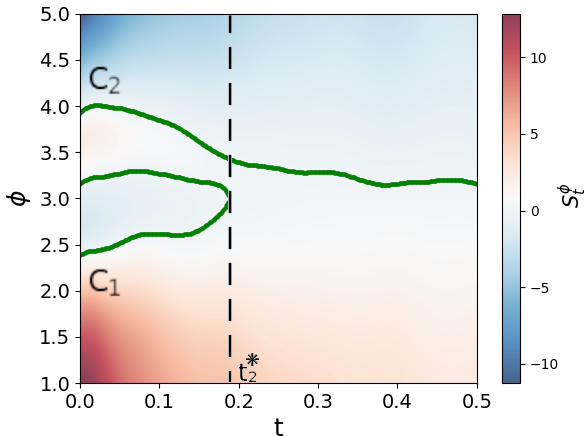}
\caption{Trimodal vMF distribution with equal weights, learned score. Zeros (green) and values (colors) of the components of the learned score.
Left: $\theta$ component at first speciation. As in the analytic model (cf.  Figure~\ref{fig:trimodal_bif}), the late-time branch is steered towards the cluster at $\theta = \pi/6$, while the center at $\theta = 5/6 \pi$ speciates via a saddle-node. Right: $\phi$ component at second speciation. The perfect pitchfork of the analytic solution unfolds into a saddle node due to the approximation of the score. The first speciation time is well approximated, the second (very close to the end of the reverse process) is affected by a larger error. }\label{fig:approx_score_trimodal_triangle}
\end{figure}

\paragraph{Asymmetric weights}
The previous experiment studied a mixture with non-symmetric centers. Here we consider a trimodal distribution built as a mixture of equidistant vMF distributions centered on the equator, with one center $C_1$ at $(\theta_1,\phi_1) = (\pi/2, \pi/3)$ and the other two centers $C_2,C_3$ at a $2\pi/3$ distance form $C_1$, namely with $(\theta_2,\phi_2) = (\pi/2, \pi)$ and $(\theta_3,\phi_3) = (\pi/2, 4\pi/3)$. We run two numerical experiments: first we consider the equal weights scenario, with $w_1=w_2=w_3=1/3$; then the asymmetric case $w_1 = 1/2$, $w_2=w_3=1/4$. 
Figure~\ref{fig:approx_score_trimodal_weights2} shows the result of this experiment
comparing the bifurcation diagrams 
for equal (left) and unequal (right)
weights. 
\begin{figure}[h]
\centering
        \includegraphics[width=.49\textwidth]{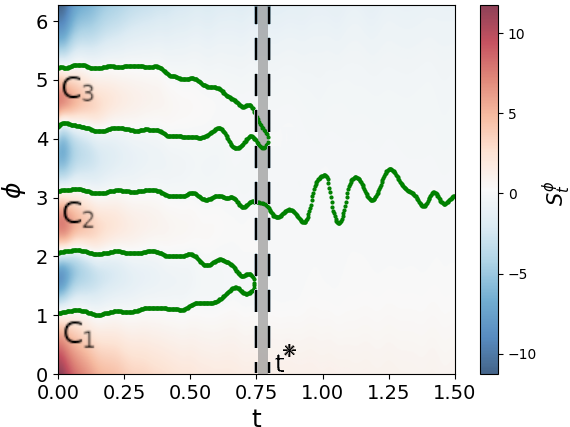}
        \includegraphics[width=.49\textwidth]{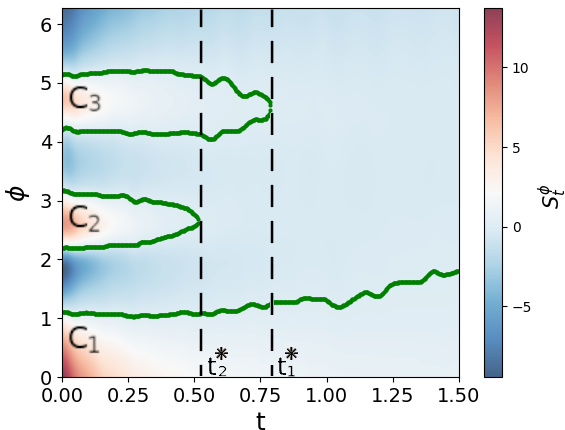}
\caption{Trimodal vMF distribution, learned score, equidistant centers on the equator. 
Left: $w_1=w_2=w_3=1/3$ case. Due to the symmetry, two saddle-node pairs are created at about the same speciation time. The original late-time branch at $\phi=\pi$ is preserved, and it continues to the nearest center located at $\phi=\pi$. Notice that the  speciation time emerging from the learned score for the two branches is similar but not exactly the same (grey area). 
Right: $w_1 = 1/2,  w_2=w_3=1/4$. The original late-time branch is steered towards the center with heaviest weight. The asymmetry of the weights leads to different speciation times for the pairs $C_1-C_2$ and $C_2-C_3$.   }\label{fig:approx_score_trimodal_weights2}
\end{figure}

\subsection{Realistic multimodal distribution: Fires on Earth dataset}
We perform the last numerical experiment with a multimodal mixture using a real-world dataset: NASA's Fire Information for Resource Management System (FIRM), a global near real-time dataset on active wild fire data from Earth-observing satellites \cite{fire_data}. Figure~\ref{fig:earth_fires} shows the original dataset and its reconstruction.  Figure~\ref{fig:earth_fires_speciation} depicts the values of the zeroes of two sections of the learned score, showing the speciation of the central Africa/central America/Indian-Southeast Asia clusters (left) and of the central Africa/East Europe cluster (right). In absence of analytical estimates of the speciation times of the different fires clusters (the underlying probability distribution is unknown), we make use of \eqref{eq:bimodal_speciation_time} to give a rough a priori approximation of the speciation times.

\begin{figure}[h]
\centering
        \includegraphics[width=.475\textwidth]{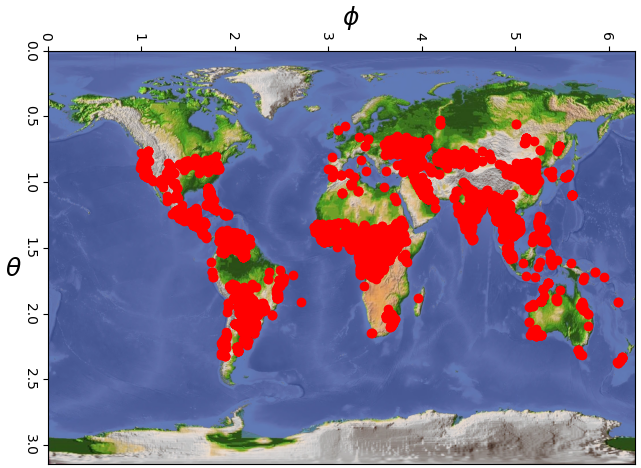}
        \includegraphics[width=.475\textwidth]{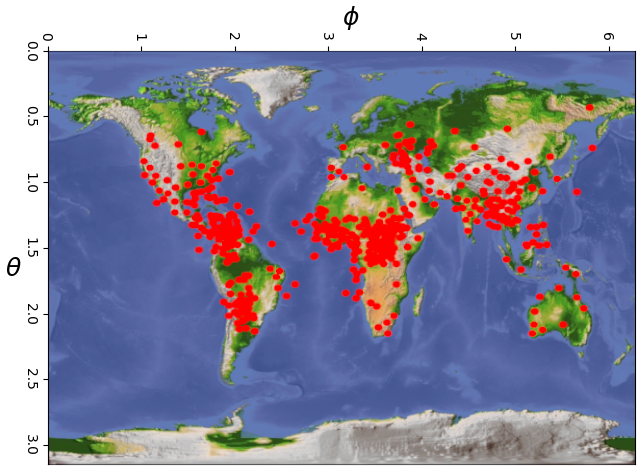}
\caption{Fires on Earth dataset. Left: original dataset; Right: $500$ points generated by the model. Earth Map from \cite{earth}}\label{fig:earth_fires}
\end{figure}

\begin{figure}[h]
\centering
        \includegraphics[width=.49\textwidth]{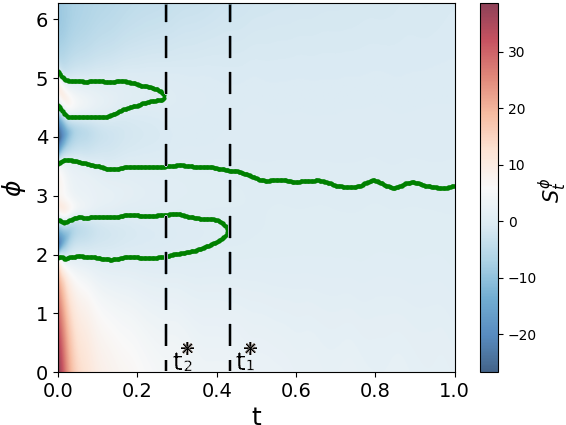}
        \includegraphics[width=.49\textwidth]{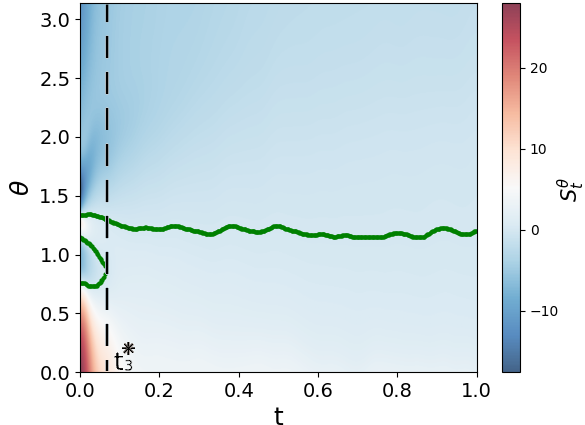}
\caption{Fires on Earth dataset. 
Zeros (green) and values  (colors) of the score components. 
Left ($\theta = \pi/2$): $\phi$ component. The late time branch contributes to the formation of the central Africa fire cluster ($\phi \approx 3.5$); The central America cluster ($\phi \approx 1.9$) is generated via a bifurcation at $t_1^* \approx 0.45$ (applying \eqref{eq:bimodal_speciation_time} to this pairs of cluster with $\sigma = 0$ yields $t_1^* \approx 0.56$); The saddle-node at $t^*_2 \approx 0.27$ corresponds to the formation of the Indian-Southeast Asia cluster ( \eqref{eq:bimodal_speciation_time} yields $t_2^* \approx 0.25$). The repellers separating the the clusters are located in the Atlantic and Indian oceans respectively. Right ($\phi \approx 3.9$): $\theta$ component. The central Africa cluster speciates from the East Europe cluster at $t^*_3 \approx 0.1$ ( \eqref{eq:bimodal_speciation_time} yields $t_2^* \approx 0.14$). The repeller separating the two clusters is located in the Mediterranean sea.} 
\label{fig:earth_fires_speciation}
\end{figure}

\section{Conclusion} \label{sec:conclusion} 
In this work, we have  developed an intrinsic framework for studying speciation in generative diffusion models supported on compact Riemannian manifolds.  Rather than identifying speciation events with purely symmetry-breaking scenarios, we have characterized them through bifurcations of the critical points of the time-dependent density, or equivalently, of the zeroes of the Riemannian score field. This perspective places the emergence of generative branches within the broader setting of dynamical systems and singularity theory, retaining all the information about the geometry and topology of the underlying data manifold. 
\medskip 

Through Poincaré--Hopf theorem and Morse theory, we showed that critical points of the score landscape must satisfy topological constraints, giving rise to both data modes, which represent the intended components or classes of the data distribution, and geometrical modes, which need not correspond to a data class but arise as part of a topologically admissible critical point configuration. For mixtures of heat kernels, we showed that generic speciation events have a one-dimensional critical kernel and are locally described by an $A_2$ fold normal form. Other kind of bifurcation, e.g. pitchforks or multidirectional bifurcations, instead arise from nongeneric symmetric configurations. We also derived geometry-dependent estimates of speciation times for bimodal distributions and for symmetric mixtures whose centers are placed on the vertices of Riemannian regular simplices. At last, we proved that nondegenerate folds persist under small score perturbations, with a first-order time shift determined only by the score error along the critical direction. This structural stability ensures that the theoretical bifurcation picture continues to hold true when the exact score is replaced by a sufficiently accurate learned approximation. We support our theoretical picture with illustrative examples on the sphere $\mathbb{S}^2$, based on known probability distributions, and numerical experiments with learned intrinsic scores. These examples exhibit all of the theoretical phenomena discussed above, including the formation of folds and pitchforks in speciation events, their unfolding under symmetry breaking and the presence of geometrical modes. Taken together, these results show that speciation events are not, in general, generated by a symmetry-breaking, but from a geometrically-organized evolution of the score landscape. 

\medskip

\noindent Future work will be devoted to investigate points that have not been addressed here,  
and, namely to: i) explore the role of the curvature of the manifold on the generative dynamics and on speciation events; ii) exploit the stability of one-dimensional bifurcations to improve the effectiveness of the guidance in generative models; iii) investigate analytic mixtures and realistic datasets for which the kernel of the bifurcation is not merely unidimensional but might involve more directions; iv) investigate diffusion models on manifolds with generic boundaries whose underlying process is a reflected Brownian motion;
v) investigate 
latent diffusion models where VAEs are combined with diffusion models to tame the computational cost.

\section*{Acknowledgments}
\noindent Paola Causin is member of the Italian group GNCS (Gruppo Nazionale Calcolo Scientifico) of INdAM.
This research has been partially performed in the framework of the GNCS Project 
\textit{``Oltre il gradiente deterministico: a\-na\-li\-si di dinamiche stocastiche per funzioni non convesse"}, CUP: E53C25002010001. 
Part of the computational resources were provided by the INDACO Platform, a project of High Performance Computing of the University of Milan ({\tt https://www.indaco.unimi.it/}).


\appendix

\section{The von Mises-Fisher distribution} 
\phantomsection
\label{app:appA}
\noindent In this Appendix we collect
all the fundamental notions needed to deal with data obtained from
the von Mises-Fisher distribution, the analogous of a Gaussian on $\mathbb{S}^d$ manifolds. In particular, we consider in detail the case $d=2$.

\subsection{Definition and properties
}
 For a unit vector point $x \in \mathbb{S}^{d} \hookrightarrow \mathbb{R}^{d+1}$, the von Mises-Fisher (vMF) distribution is given by
\begin{equation}
p_{\text{vMF}}(x;\mu,\kappa) = \dfrac{\kappa e^{\kappa \mu \cdot x}}{4\pi \sinh(\kappa)},
\label{eq:vMFdistri}
\end{equation}
where $\cdot$ is the standard scalar product of $\mathbb{R}^{d+1}$, $\mu$ is another unit vector called mean direction of the distribution and $\kappa$ is parameter which controls the concentration of the distribution around the point on the sphere individuated 
by~$\mu$. The denominator serves as a normalization factor. When $\kappa = 0$, the distribution is the uniform on the sphere, as $\kappa$ grows the mass concentrates tightly around~$\mu$ approaching a delta
distribution. In the numerical experiments reported in the main text we always set $\kappa=20.$


\subsection{Analytic solution of the Fokker-Planck equation on $\mathbb{S}^2$ for data from the vMF distribution}

We use the standard superficial spherical coordinates with $R=1$, $\theta$ as colatitude and $\phi$ as longitude:
$$
(\theta, \phi) \;\mapsto\; (\sin\theta\cos\phi,\; \sin\theta\sin\phi,\; \cos\theta),
$$
with $\theta \in (0, \pi)$ and $\phi \in [0, 2\pi)$. The metric in these coordinates is 
$g \;=\; d\theta^2 \;+\; \sin^2\theta\, d\phi^2$ and
the induced volume form is
$d\mathrm{vol}_g = \sqrt{\det g}\; d\theta\, d\phi = \sin\theta\, d\theta\, d\phi$, 
which gives $|\mathbb{S}^2| = 4 \pi^2$. 

\medskip

\noindent We work hereafter in the coordinate chart $(0,\pi)\times(0,2\pi]$, imposing periodic boundary conditions on $\phi$. The Fokker-Planck equation~\eqref{eq:app-fp-forward} on $\mathbb{S}^2$ reads in superficial spherical coordinates
\begin{equation}
\label{eq:fokker_s2}
\frac{\partial p_t}{\partial t} =
\frac{1}{2}\left[\partial_\theta^2 f  + \cot\theta\,\partial_\theta f \right] + \dfrac{1}{2 \ \sin^2 \theta} \partial_\phi^2 f,
\end{equation}
where the term $[\cot\theta\,\partial_\theta (\cdot)]$
is a chart-dependent \textit{geometric drift}  pushing in both hemispheres the process towards the equator of the sphere. We remember that this term should not be considered as a drift of its own, but as part of the diffusive term associated with the Laplace-Beltrami operator. Equation~\eqref{eq:fokker_s2} can be solved by separation of variables, admitting the expansion
\begin{equation}\label{eq:p_sphere_spherical_harmonics}
p_t(\theta,\phi)
=
\sum_{\ell,m}
\langle p_0, Y_{\ell m} \rangle \, e^{-\frac{1}{2}\ell(\ell+1)t}\, Y_{\ell m}(\theta,\phi)
\end{equation}
where the spherical harmonics $Y_{\ell m}$ are the eigenfunctions of the Laplace--Beltrami operator operator with corresponding eigenvalues $\lambda_\ell=\ell(\ell+1)$ and first spectral gap $\ell(\ell+1)|_{\ell=1}=2$. 

\medskip

Let now the initial condition $p_0$ be a vMF distribution centered at $(\alpha,\beta)$, with mean direction
\[
\mu = (\sin\alpha\cos\beta,\ \sin\alpha\sin\beta,\ \cos\alpha).
\]
Since the Laplace–Beltrami operator $\Delta_g$ is invariant under rotations, the solution $p_t$
depends only on the geodesic angle $\gamma=\arccos(u)$ between $x$ and $\mu$, where
\begin{equation}\label{eq:generic_variable_u}
u \;=\; x\cdot\mu \;=\; \sin\theta\,\sin\alpha\,\cos(\phi-\beta) + \cos\theta\,\cos\alpha .
\end{equation}
By rotational symmetry it is therefore enough to solve the problem for the simple case $\alpha=0$, i.e.
$\mu=(0,0,1)$, in which case $u=\cos\theta$ and $p_0$ is independent of the azimuthal angle
$\phi$. This axial symmetry makes the projections $\langle p_0, Y_{\ell m}\rangle$ vanish for
every $m\neq 0$, so that the only contributing harmonics are
\begin{equation*}
Y_{\ell 0}(u) \;=\; \sqrt{\tfrac{2\ell+1}{4\pi}}\,P_\ell(u), \qquad u=\cos\theta .
\end{equation*}
Consequently, the expansion~\eqref{eq:p_sphere_spherical_harmonics} reduces to
\begin{equation}\label{eq:p_sphere_expansion}
p_t(u) \;=\; \frac{1}{4\pi}\sum_{\ell=0}^{\infty} (2\ell+1)\,c_\ell\,
e^{-\ell(\ell+1)t/2}\,P_\ell(u),
\end{equation}
where $P_\ell$ are the Legendre polynomials and the coefficients $c_\ell$ are the projections of
$p_0=p_{\mathrm{vMF}}$ onto the Legendre basis,
\[
c_\ell \;=\; 2\pi\!\int_{-1}^{1} P_\ell(u)\,p_0(u)\,\mathrm{d}u
\;=\; \frac{\kappa}{2\sinh\kappa}\!\int_{-1}^{1} P_\ell(u)\,e^{\kappa u}\,\mathrm{d}u .
\]
This integral has a closed form in terms of the modified spherical Bessel functions of the first
kind,
\[
c_\ell \;=\; \frac{i_\ell(\kappa)}{i_0(\kappa)}, \qquad
i_\ell(\kappa)=\sqrt{\tfrac{\pi}{2\kappa}}\,I_{\ell+1/2}(\kappa),
\]
which follows from the Rayleigh plane-wave expansion
$e^{\kappa u}=\sum_{\ell\ge0}(2\ell+1)\,i_\ell(\kappa)\,P_\ell(u)$ together with the orthogonality
relation $\int_{-1}^{1}P_\ell P_{\ell'}\,\mathrm{d}u=\tfrac{2}{2\ell+1}\delta_{\ell\ell'}$.

In~\eqref{eq:p_sphere_expansion} every mode $\ell\ge1$ decays exponentially, so that only
$\ell=0$ survives and $p_t(x)\to 1/(4\pi)$ uniformly as $t\to\infty$. Finally, since $p_t$ depends
on $x$ only through $u=x\cdot\mu$, the reduction $\alpha=0$ entails no loss of generality:
replacing $u=\cos\theta$ with the general expression~\eqref{eq:generic_variable_u} yields the
solution for an arbitrary mean direction,
\begin{equation}\label{eq:expavMF_general}
p_t(\theta,\phi) \;=\; \frac{1}{4\pi}\sum_{\ell=0}^{\infty} (2\ell+1)\,c_\ell\,
e^{-\ell(\ell+1)t/2}\,
P_\ell\!\big(\sin\theta\,\sin\alpha\,\cos(\phi-\beta) + \cos\theta\,\cos\alpha\big),
\end{equation}
where, from the recurrence of the spherical Bessel functions $i_{\ell+1} = i_{\ell-1} - \tfrac{2\ell+1}{\kappa}\,i_\ell$,
we get for $\ell=1,2,\dots$
$$c_0=1, \,\, c_1 = \coth \kappa -\frac{1}{\kappa}, \,\,
{c_{\ell+1}(\kappa) = c_{\ell-1}(\kappa) - \frac{2\ell+1}{\kappa}\,c_\ell(\kappa)}.
$$ 
The score of $p_t$ is the Riemannian gradient of its logarithm,
$S_t(x)=\nabla_{\mathbb{S}^2}\log p_t(x)$, which reads in explicit form 
\begin{equation}
S_t(\theta,\phi) =\displaystyle \frac{
{\displaystyle\sum_{\ell \ge 1}}
\left(D_u P_\ell(u(\theta,\phi))\right)
(2\ell + 1)\, c_\ell \, e^{-\ell(\ell+1)t/2}
}{
{\displaystyle \sum_{\ell \ge 0}}
P_\ell(u(\theta,\phi))\,
(2\ell + 1)\, c_\ell \, e^{-\ell(\ell+1)t/2}}
V 
\label{eq:scorevMF}
\end{equation}
where the numerator starts at $\ell=1$ since $D_u P_0\equiv 0$ and
we have set
\begin{equation*}
V = \nabla u(\phi,\theta) =
\begin{pmatrix}
1 & 0\\
0 & (\sin \theta)^{-2}
\end{pmatrix}
\begin{pmatrix}
\cos \theta \sin \alpha \cos(\phi -\beta) - \sin \theta \cos \alpha \\
- \sin \theta \sin \alpha \sin (\phi - \beta) + \cos \theta \cos \alpha
\end{pmatrix}
\end{equation*}
The numerical computation of~\eqref{eq:p_sphere_expansion}
and~\eqref{eq:scorevMF} is quite delicate and
a high number of terms is needed to achieve sufficient precision at small times. In this work, when not differently specified, we have used $19$ terms plus the constant. 

\subsection{Large $\kappa$ asymptotics of vMF distribution and comparison with the heat kernel}
\label{sec:vmf_asymptotics}
In Section~\ref{sec:bimodal_speciation_time} we derived an estimate for the speciation time of a bimodal distribution, under the assumption that the initial distribution is a Riemannian Gaussian of variance $\sigma^2$ -- a zeroth-order approximation of the heat kernel. In order to apply the estimate \eqref{eq:bimodal_speciation_time} in the case at hand, we need to study how the variance $\sigma^2$ and the concentration parameter $\kappa$ of the vMF distribution are related for large $\kappa$. Let $f_{\rm vMF}(x)
\propto e^{\kappa \mu\cdot x}$ be the probability density of a vMF distribution and let $r=d(x,\mu)$. Since $\mu\cdot x=\cos r$, we find $f_{\rm vMF}(r) \propto e^{\kappa\cos r}$. For small $r$ we can expand $\operatorname{exp}(\kappa \cos r)$ as
\begin{equation*}
e^{\kappa\cos r}
=
e^\kappa
\exp\!\left(
-\frac{\kappa r^2}{2}
+\frac{\kappa r^4}{24}
+\cdots
\right)
\end{equation*}
Considering this expression to leading order yields
\begin{equation*}
f_{\rm vMF}(r)
\propto
\exp\!\left(
-\frac{\kappa r^2}{2}
\right)
\end{equation*}
from which we deduce that a highly concentrated vMF looks (modulo a normalization factor) like a Gaussian with variance
\begin{equation}
\sigma^2=\frac1\kappa.
\end{equation}

\null

\section{Proofs of propositions of Section~\ref{section:generic_mixtures}}
\phantomsection
\label{app:appB}

We show that for a generic mixture of heat kernels, every bifurcation of the critical-point set $\Xi$ is an $A_2$ fold for a sufficiently large number of components of the mixture. The proof relies on the fact that the jet evaluation map of the mixture is a submersion, which allows to apply Thom's parametric transversality theorem to conclude that there is a set of parameters of positive measure for smooth $g$ and of full measure for analytic $g$ for which the bifurcation is of $A_2$ type.

\subsection{Analytic preliminaries}

\subsubsection{k-jets of functions}
\label{sec:jet}

In this section we briefly recall the notion of jets of a function, which can be seen as the coordinate--free version of Taylor expansions \cite{zbMATH00205892,zbMATH00042050}. Suppose that $M$ is a $n$-dimensional manifold and consider the bundle $(E = M \times \mathbb{R}, \pi, M)$, with $\pi$ the projection on the first component, whose local sections are the smooth functions over $M$. Let $f_1,f_2 \in \mathcal{C}^\infty(M)$. We say that $f_1$ and $f_2$ are $k$-equivalent at a point $x$ if, in any chart $\partial^\alpha f_1(x) = \partial^\alpha f_2(x)$ for every multi-index $\alpha$ with $|\alpha|\le k$. The equivalence class $j_x^k f_1$ is called the $k-$jet of $f_1$ at $x$. Concretely, $j_x^k f_1$ contains the information about the derivatives of $f_1$ up to order $k$ at the point $x_0$. The quotient of $\mathcal{C}^\infty(M)$ with respect to this equivalence relation is the vector space $J_x^k(E)$, which is of dimension $\binom{n+k}{k}$.
\begin{definition}
Denote by $J^k_{x_0}$E the space of all k-jets of sections of $E$ at $x_0$ and set 
$$J^k E = \bigsqcup_{x\in M} J_x^k(E).$$
$J^k E$ is the k-jet bundle of $E$.
\end{definition}
\begin{remark}
In coordinates, points in $J^k E$ are Taylor polynomials of sections of $E$ at all possible points of $M$.     
\end{remark}

The 1-jet of a function $f$ at $x$ -- described by $(f(x),df_x)$ -- admits the canonical splitting $J^1(E) \cong \mathbb R \oplus T^\star M$. However, there is no canonical choice for the splitting of higher order jets without additional structures over $M$. If $(M,g)$ is a Riemannian manifold, the Levi-Civita connection allows us to describe jets using (symmetrized) covariant derivatives, inducing the splitting:
\begin{equation*}
J^k E \cong \sum_{i=1}^k \Sym^i(T^\star M)= \mathbb R \oplus T^\star M \oplus \Sym^2(T^*M) \oplus \ldots \oplus \Sym^k(T^\star M)
\end{equation*}
In a local chart, the components of this splitting contain, respectively, the value of a function at a point, its gradient, its Hessian, and so on, up to the derivatives of order $k$.

\subsubsection{\texorpdfstring{$A_k$}{Ak} singularities}
\label{sec:Ak}
Now we recall the definition of corank-one singularities of function germs. Without loss of generality, we suppose that all germs are at $0$, which is also a critical point. In the following right
equivalence means a smooth change of coordinates in the source.

\begin{definition}[$A_k$ singularity]
\label{def:Ak}
A germ $f\colon \R^n \to \R$ has an $A_k$ singularity ($k\ge1$) if
it is right-equivalent to the normal form
\begin{equation}
  \label{eq:Ak-normal-form}
  A_k:\qquad
  f(\xi,\eta)\;=\;\xi^{\,k+1}\;+\;\tfrac12\sum_{j=1}^{n-1}\varepsilon_j\,\eta_j^2,
  \qquad \varepsilon_j=\pm1 .
\end{equation}
The quadratic part is a nondegenerate Morse function in the $n-1$ variables $\eta$; the singularity is isolated in the variable $\xi^{k+1}$. 
\end{definition}
The low $k$ cases are also know in literature with their classical names \cite{zbMATH01552061,zbMATH03676845}
\[
  A_1:\ \xi^2\ (\text{Morse}),\quad
  A_2:\ \xi^3\ (\text{fold}),\quad
  A_3:\ \xi^4\ (\text{cusp}),\quad
  A_4:\ \xi^5\ (\text{swallowtail}).
\]
\begin{remark}
\label{rem:Ak-codim}
The Hessian of $f$ at zero has corank one, with kernel $\partial_\xi$ for every $k \geq 2$, therefore the germ $A_k$ is fixed up to right equivalence by its $(k{+}1)$-jet. 
\end{remark}
By the Thom's splitting lemma \cite{zbMATH03473031}, any degenerate critical point of corank one is right-equivalent to
$\phi(\xi)+\tfrac12\sum_j\varepsilon_j\eta_j^2$ with $\phi''(0)=0$.
The exact type of singularity is therefore determined by the vanishing of the higher-order derivatives of $\phi$ at $x=0$. In the following we shall consider families of functions $f_t(x)$ smoothly parametrized by the time $t$. In this scenario, the Parametrized Morse Lemma \cite{zbMATH05129478} applied to $f_t$ yields that the family $f_t(x)$ is right-equivalent to
$\psi(t) + \phi_t(\xi)+\tfrac12\sum_j\varepsilon_j(t) \eta_j^2$ with $\phi_t''(0)=0$, where all the functions appearing in this expression depends smoothly on $t$. To ensure that the bifurcation is nondegenerate, one must impose additional conditions. Specifically, in the $A_2$ fold case we are interested in, we require $\partial_t \partial_\xi \phi_t(\xi) \neq 0$; we shall see in Corollary~\ref{cor:cerf} that, for analytic metrics and a sufficient number of centers in a mixture of heat kernels, this condition is almost always satisfied.

\subsection{Ampleness of heat-kernel jets}
\label{sec:ampleness}

Throughout, $(\Msf,g)$ is a compact Riemannian manifold, $p_{t}(x,y)$ is the heat kernel of $\tfrac12\Delta_g$, $\{\phi_k\}_{k\ge0}$ is an $L^2$-orthonormal basis of Laplace--Beltrami operator
eigenfunctions with associated eigenvalues $\{\lambda_k\}_{k\ge0}$. For a fixed point $x_0\in \Msf$ and $r\ge0$, we write $\jet{r}{x_0}\colon C^\infty(\Msf)\to J^r_{x_0}(\Msf)$ for the $r$-jet evaluation at $x_0$. This map is a linear surjection onto the fiber.
In the following we consider the family of jets
\begin{equation}
  \label{eq:atoms}
  \mathcal K^r_t
  \;:=\;
  \bigl\{\, \jet{r}{x_0}\,p_t(\cdot,y)\;:\;y\in \Msf \,\bigr\}
  \;\subset\; J^r_{x_0}(\Msf),
  \qquad t>0 .
\end{equation}
Since a mixture of heat kernels $\rho_t = \sum_{i=1}^K w_i p_{t_i}$, with $t_i = t + \tau_i$ and $t_i>0$ for every $i$, is linear in the weights, the linearity of the jet evaluation map yields
\begin{equation}
  \label{eq:mixture-linearity}
  \jet{r}{x_0}\,\rho_t(\cdot,y) = \jet{r}{x_0}\!\Bigl[\textstyle\sum_i w_i\,p_{t_i}(\cdot,y_i)\Bigr]
  \;=\;
  \sum_i w_i\, \jet{r}{x_0}\,p_{t_i}(\cdot,y_i).
\end{equation}
We start our enquiry by proving that the heat-kernel mixtures generate the jet bundle fibrewise.
\begin{lemma}
\label{lem:ampleness}
For every fixed $t>0$ and $x_0 \in \Msf$, and for every $r\ge0$,
\[
  \spann\,\mathcal K^r_t \;=\; J^r_{x_0}(\Msf).
\]
\end{lemma}
\begin{proof} \,
The thesis is equivalent to the fact that no nonzero linear functional on $J^r_{x_0}(\Msf)$ annihilates the
family $\{\,\jet{r}{x_0}p_t(\cdot,y):y\in \Msf\,\}$. Let $\ell\in\bigl(J^r_{x_0}(\Msf)\bigr)^*$ annihilate $\mathcal K^r_t$, and set
$L:=\ell\circ \jet{r}{x_0}\colon C^\infty(\Msf)\to\R$, so that $L\bigl(p_t(\cdot,y)\bigr)=0$ for all $y\in \Msf$. In any chart $(U,x)$ containing $x_0$, $L$ acts on smooth functions as
\begin{equation}
  \label{eq:point-distribution}
  L(f)\;=\;\sum_{|\alpha|\le r} c^{\alpha}\,\partial^{\alpha}f(x_0),
  \qquad f\in C^\infty(\Msf),
\end{equation}
Since $|L(f)|\le C\,\|f\|_{C^r(U)}$ for $f$ supported near
$x_0$, $L$ is continuous in $C^r(U)$. We claim that for a fixed $t>0$ the eigen-expansion
\begin{equation}
  \label{eq:eigenexp_mixture}
  p_t(x,y)\;=\;\sum_{k\ge0} e^{-\lambda_k t}\,\phi_k(x)\,\phi_k(y)
\end{equation}
converges in $C^\infty(\Msf\times \Msf)$. Indeed, standard elliptic estimates applied to the eigenfunction equation for $\tfrac12\Delta_g$ together with Sobolev embedding theorem \cite{taylor2010partial} yields $\|\phi_k\|_{C^r(\Msf)}\lesssim_{r} \lambda_k^{\,N(r)}$ for large $k$, while Weyl's law for the asymptotic behavior of eigenvalues of the Laplace–Beltrami operator gives $\lambda_k\sim c\,k^{2/n}$. Therefore the coefficients $e^{-\lambda_k t}$ decay faster than any polynomial in $\lambda_k$ and dominate
every $C^r\times C^r$ seminorm, yielding the sought convergence. Applying the $C^r$-continuous
functional $L$ to \eqref{eq:eigenexp_mixture} in the $x$-variable term by term, we obtain
\begin{equation}
  \label{eq:apply-L}
  0\;=\;L\bigl(p_t(\cdot,y)\bigr)
   \;=\;\sum_{k\ge0} e^{-\lambda_k t}\,L(\phi_k)\,\phi_k(y)
   \qquad\text{in } C^\infty(\Msf).
\end{equation}
Taking the $L^2$ inner product of \eqref{eq:apply-L} with $\phi_j$ and exploiting the orthonormality of the eigenfunctions, we find $e^{-\lambda_j t}L(\phi_j)=0$ and thus
\begin{equation}
  \label{eq:vanishing-on-eigen}
  L(\phi_k)=0\qquad\text{for every }k\ge0.
\end{equation}
Now let $f\in C^\infty(\Msf)$ be an arbitrary function, with spectral coefficients $\hat f_k=\langle f,\phi_k\rangle$. The smoothness of $f$ implies a rapid decay of $f_k$ -- faster than any polynomial in $\lambda_k$ -- so using the polynomial
bound on $\|\phi_k\|_{C^r}$, the partial sums $f_M:=\sum_{k\le M}\hat f_k\phi_k$
converge to $f$ in $C^\infty(\Msf)$. By $C^r$-continuity of $L$ and
\eqref{eq:vanishing-on-eigen},
\[
  L(f)\;=\;\lim_{M\to\infty}\sum_{k\le M}\hat f_k\,L(\phi_k)\;=\;0 .
\]
Since $f$ is an arbitrary smooth function and $\jet{r}{x_0}$ is surjective, we conclude that $\ell=0$.
\end{proof}

Fix $t>0$ and set $m:=\dim J^2_{x_0}(\Msf)=\binom{n+2}{2}$. Choose a basis of $J^2_{x_0}(\Msf)$ and, for a choice of centers $y_1,\dots,y_m\in \Msf$ of the mixture, let
\begin{equation}
  \label{eq:wronskian}
  W(y_1,\dots,y_m)
  \;:=\;
  \det\bigl[\,v(y_1)\ \big|\ \cdots\ \big|\ v(y_m)\,\bigr],
  \qquad
  v(y_a):=\bigl(\partial_x^{\alpha}p_t(x_0,y_a)\bigr)_{|\alpha|\le 2}\in\R^{m},
\end{equation}
where each column $v(y_a)$ collects the $x$-derivatives of order $\le2$ of
$p_t(\cdot,y_a)$ evaluated at $x=x_0$, namely the
coordinates of $\jet{2}{x_0}\rho_t(\cdot,y_a)$ in the partial-derivative basis of $J^2_{x_0}(\Msf)$ in a fixed chart containing $x_0$. The condition $W(y_1,\ldots,y_m) \neq 0$ does not depend on the chosen chart of $J^2_{x_0}(\Msf)$, as changing the basis multiplies $W$ by a nonzero constant.

\begin{corollary}
\label{cor:hv-ampleness}
The function $W$ in \eqref{eq:wronskian} is not identically zero on $\Msf^m$.
Consequently, the set of centers $(y_1,\dots,y_m)$ whose $2$-jets form a basis of $J^2_{x_0}(\Msf)$ is nonempty and open in $\Msf^m$ for smooth $g$ and is the complement of a proper real-analytic subvariety (hence residual and of full measure) when $g$ is real-analytic.
\end{corollary}

\begin{proof} \,
By Lemma~\ref{lem:ampleness} with $r=2$, the family
$\{\,\jet{2}{x_0}p_t(\cdot,y)\,\}_{y\in \Msf}$ spans the $m$-dimensional space $J^2_{x_0}(\Msf)$. Since a spanning set contains a basis, there are some centers $y_1^{\circ},\dots,y_m^{\circ}$ whose $2$-jets are linearly independent. Therefore the corresponding columns $v(y_a^{\circ})$ are then independent and $W(y_1^{\circ},\dots,y_m^{\circ})\neq0$. As $W$
is continuous, $\{W\neq0\}$ is a nonempty open set. Assume now $g$ real-analytic. For $t>0$ the heath kernels $p_t(x,y)$ are real-analytic in $(x,y)$ and therefore each derivative $\partial_x^{\alpha}p_t(x,y_a)|_{x=x_0}$ is a real-analytic function of $y_a$ \cite{zbMATH05129478,zbMATH03455728,zbMATH00967584}. As a consqequence, $W$, which is a polynomial in these derivatives, is real-analytic on the connected manifold $\Msf^m$. By the identity theorem for real-analytic functions on connected manifolds, its zero set $\{W=0\}$ is a proper real-analytic subset and has Lebesgue measure zero \cite{KrantzParks}. Therefore $\{W\neq0\}$ is open, dense and of full measure.
\end{proof}

\subsection{Submersion via unnormalized weights}
\label{subsec:submersion}

Now we use the above spanning Lemma~\ref{lem:ampleness} to prove that the parameter-to-jet map is a submersion for non-normalized weights. Working on the open cone $\R^K_{>0}$ (whose tangent space at every
point is all of $\R^K$) avoids the affine constraint $\sum_i\dot w_i=0$ that a probability simplex would impose. We shall transfer our unconstrained findings to the simplex by scale-invariance at the end (see Remark~\ref{rem:scaling-descent}).

\begin{corollary}
\label{cor:submersion}
Fix $t>0$ and components $y_1,\dots,y_K\in \Msf$ with $K\ge\binom{n+2}{2}$ chosen
so that $\{\jet{2}{x_0}p_t(\cdot,y_i)\}_{i=1}^K$ spans $J^2_{x_0}(\Msf)$. For unnormalized weights
$w\in\R^K_{>0}$ set $u_w=\sum_{i=1}^K w_i\,p_t(\cdot,y_i)$. Then
\begin{equation}\label{eq:parameters-to-jet-map}
  \R^K_{>0}\ni w\ \longmapsto\ \jet{2}{x_0}u_w\ \in\ J^2_{x_0}(\Msf)
\end{equation}
is a submersion at every point; its differential is the constant linear map
$\partial_{w_i}\jet{2}{x_0}u_w=\jet{2}{x_0}p_t(\cdot,y_i)$.
\end{corollary}

\begin{proof} \,

By Lemma~\ref{lem:ampleness} we can always choose $K$ centers $y_1,\dots,y_p\in \Msf$ so that $\{\jet{2}{x_0}p_t(\cdot,y_i)\}_{i=1}^K$ spans $J^2_{x_0}(\Msf)$. The parameter-to-jet map \eqref{eq:parameters-to-jet-map} is linear in $w$, so its differential at any $w$ is the constant map
$A\colon\R^K\to J^2_{x_0}(\Msf)$, $Ae_i=\jet{2}{x_0}p_t(\cdot,y_i)=:a_i$. By the spanning hypothesis $\{a_i\}$ generates $J^2_{x_0}(\Msf)$,
i.e.\ $A$ is onto; hence the parameter-to-jet map \eqref{eq:parameters-to-jet-map} is a submersion.
\end{proof}

\subsection{Generic bifurcations are corank one}
\label{sec:corank-one}

From now on, we consider the mixture
\begin{equation}
  \label{eq:mixture-density}
  u_t(x)\;=\;\sum_{i=1}^{K} w_i\,p_{\tau_i+t}(x,y_i),
  \qquad t\ge0,
  \qquad
  \theta:=(w,\tau,y)\in\R^K_{>0}\times(0,\infty)^K\times \Msf^K,
\end{equation}
where the fixed parameter $\theta$ contains the weights $w_i$, the initial component
scales $\tau_i>0$ -- the Riemannian analogous of the variance of a component in the Euclidean scenario -- and the centers $y_i$. The diffusion time $t\ge0$
advances every component's scale simultaneously, so that at time $t$ the $i$-th component heat kernel is $p_{\tau_i+t}(\cdot,y_i)$. We study critical points of $x\mapsto u_t(x)$ as $t$ varies. Since a bifurcation is a pair $(x^*,t^*)$ with $\nabla u_{t^*}(x^*)=0$
and $\operatorname{Hess} u_{t^*}(x^*)$ is degenerate, we have $\corank\ge1$. 

\subsubsection{The symmetric determinantal stratification}
\label{subsec:sym-det}

Let $\Sym(n)\cong \Sym^2\R^{n*}$, and for $k\ge1$ set
\begin{equation}
  \label{eq:corank-strata}
  \Sigma_k \;:=\;\{A\in\Sym(n):\corank A\ge k\},
  \qquad
  \Sigma_k^{\circ}:=\{\corank A=k\}.
\end{equation}
Each $\Sigma_k^{\circ}$ is a smooth submanifold with
\begin{equation}
  \label{eq:codim-sym-det}
  \codim_{\Sym(n)}\Sigma_k^{\circ}=\codim\Sigma_k=\binom{k+1}{2}=\frac{k(k+1)}{2},
\end{equation}
and $\Sigma_k=\bigsqcup_{j\ge k}\Sigma_j^{\circ}$ is Whitney (b)-regular. The codimension \eqref{eq:codim-sym-det} follows from the standard count for symmetric determinantal
varieties~\cite{zbMATH03848223}. 
\subsubsection{Codimension count in the jet bundle}
\label{subsec:jet-count}

In $J^2(\Msf,\R)$ the fiber splits into  $\mathbf{R} \oplus T^\star_x M \oplus Sym^2(T_x^\star M)$ -- in components: value, gradient and Hessian -- with $Sym^2 (T^\star_x\Msf)\cong\Sym(n)$. Set
\begin{equation}
  \label{eq:crit-corank-loci}
  \mathcal C:=\{\nabla u=0\}\subset J^2,
  \qquad
  \mathcal D_k:=\mathcal C\cap\{\text{Hessian block}\in\Sigma_k\}.
\end{equation}
Gradient ($n$ coords) and Hessian ($\binom{n+1}{2}$ coords) blocks are
independent, so $\mathcal C=\{\nabla u=0\}$ and $\{\text{Hessian}\in\Sigma_j^\circ\}$
constrain disjoint coordinate blocks of $J^2$. In particular, they are mutually transverse, and
their intersection
\begin{equation}
  \label{eq:D-strata}
  \mathcal D_k^{(j)}:=\mathcal C\cap\{\text{Hessian block}\in\Sigma_j^\circ\},
  \qquad j\ge k,
\end{equation}
is a smooth submanifold of $J^2$ of codimension
$\codim\mathcal C+\codim\Sigma_j^\circ=n+\binom{j+1}{2}$. These are the smooth
strata of the corank locus $\mathcal D_k=\bigsqcup_{j\ge k}\mathcal D_k^{(j)}$, inheriting the Whitney stratification of $\{\Sigma_j^\circ\}$; in particular
\begin{equation}
  \label{eq:total-codim}
  \codim_{J^2}\mathcal D_k^{(k)}
  \;=\;\underbrace{n}_{\nabla u=0}+\underbrace{\tfrac{k(k+1)}{2}}_{\corank=k},
\end{equation}
the smallest codimension among the strata of $\mathcal D_k$.
Let
\begin{equation}
  \label{eq:psi-section}
  \Psi\colon \Msf\times\R\longrightarrow J^2(\Msf,\R),
  \qquad \Psi(x,t)=\jet{2}{x}u_t(x),
\end{equation}
If $\Psi$
is transverse to the stratum $\mathcal D_k^{(j)}$,
\begin{equation}
  \label{eq:preimage-dim}
  \dim\Psi^{-1}(\mathcal D_k^{(j)})
  =(n+1)-\Bigl(n+\tfrac{j(j+1)}{2}\Bigr)
  =1-\frac{j(j+1)}{2}.
\end{equation}
In particular, we find $\dim\Psi^{-1}(\mathcal D_1^{(1)})=0$, namely for $D_1^{(1)}$ admits only isolated points, while for $k \geq 2$ we have $\dim\Psi^{-1}(\mathcal D_k^{(k)}) \leq -2$, that is $\Psi^{-1}(\mathcal D_k^{(k)})=\varnothing$.

\subsubsection{Transversality argument}
\label{subsec:transversality}

\begin{lemma}
\label{lem:transversality}
Assume $g$ is smooth. There is a nonempty open set of parameters
$G\subseteq\R^K_{>0}\times(0,\infty)^K\times \Msf^K$ -- the uniformly
spanning parameters of \eqref{eq:G-def} below, nonempty once $K$ is large enough-- and a full-measure subset of $G$
such that, for every $\theta$ in it, $\Psi_\theta$ is transverse to $\mathcal C$
and to every stratum $\mathcal D_1^{(j)}=\mathcal C\cap\{\Hess\in\Sigma_j^\circ\}$
of $\mathcal D_1$. Morevoer, when $g$ is real-analytic the good set $G$ is of full measure in the whole parameter space $\R^K_{>0}\times(0,\infty)^K\times \Msf^K$, provided
$K\ge\binom{n+2}{2}+n+1$.
\end{lemma}

\begin{proof} \,
For clarity, we divide the proof into three steps.

\emph{Step 1: Construction of the spanning set $G$.}
The total evaluation map over the unnormalized parameter space is
\begin{equation}
  \widehat\Psi\colon (\Msf\times\R)\times
  \bigl(\R^K_{>0}\times(0,\infty)^K\times \Msf^K\bigr)
  \to J^2(\Msf,\R),
  \quad
  \widehat\Psi\bigl((x,t),\theta\bigr)=\jet{2}{x}u_t^{\theta}(x),
\end{equation}
with weight-direction derivative
$\partial_{w_i}\jet{2}{x}u_t^{\theta}(x)=\jet{2}{x}p_{\tau_i+t}(\cdot,y_i)(x)$. $\widehat\Psi$ is submersion if the component jets span $J^2_x$ at every source point $x$ simultaneously. Lemma~\ref{lem:ampleness} and Corollary~\ref{cor:submersion} guarantee spanning at each fixed $x$, for centers
in an $x$-dependent open set. However, a single choice of centers need not span at every $x$ at once. Fixed a compact time-slab $[0,T]$, suppose that we can restrict $ \widehat\Psi$ to the set where uniform spanning holds:
\begin{equation}
  \label{eq:G-def}
  G:=\Bigl\{(\tau,y):\ \{\jet{2}{x}p_{\tau_i+t}(\cdot,y_i)\}_{i=1}^K
       \ \text{spans}\ J^2_x\ \text{for every}\ (x,t)\in \Msf\times[0,T]\Bigr\}
     \ \subseteq\ (0,\infty)^K\times \Msf^K,
\end{equation}
and take the parameter set to be $\R^K_{>0}\times G$. On $\R^K_{>0}\times G$ the weight-block alone is surjective onto $J^2_x$ at every $(x,t)\in \Msf\times[0,T]$, so
$\widehat\Psi|_{(\Msf\times[0,T])\times(\R^K_{>0}\times G)}$ is a submersion. Now we prove that $G$ is nonempty and open for smooth $g$ exploiting the fact that a bundle everywhere generated by a family of sections is generated by finitely many of them.

\emph{(a) Pointwise spanning.} Fix $(x_0,t_0)\in \Msf\times[0,T]$. By
Lemma~\ref{lem:ampleness}, at the scale $s>T$, the family
$\{\jet{2}{x_0}p_{s}(\cdot,y):y\in \Msf\}$ spans the $m$-dimensional space $J^2_{x_0}$ and therefore it must contain a basis, namely there are $m=\binom{n+2}{2}$ centers $y_1^\circ,\dots,y_m^\circ$ and initial scales $\tau_i^\circ:=s-t_0>0$ (positive,
as $s>T\ge t_0$) whose jets $\jet{2}{x_0}p_{\tau_i^\circ+t_0}(\cdot,y_i^\circ)$
form a basis of $J^2_{x_0}$.

\emph{(b) The spanning set is open in $(x,t)$.}
With those $m$ centers and scales fixed, the jets span $J^2_x$ only if the Wronskian $W(x,t)$ as per \eqref{eq:wronskian} does not vanish. $W(x_0,t_0)\neq0$ by (a) and since $p_s(x,y)$ is smooth in $(x,y,s)$ for $s>0$, $W$ is continuous in $(x,t)$ jointly. Hence $W\neq0$ on an open neighborhood
$U_{(x_0,t_0)}\subseteq \Msf\times[0,T]$, i.e.\ the same $m$ centers span $J^2_x$ for all $(x,t)\in U_{(x_0,t_0)}$. 

\emph{(c) Compactness and pooling.} The neighborhoods $\{U_{(x_0,t_0)}\}$ built in (b) cover
the compact set $\Msf\times[0,T]$, so we can extract a finite sequence of neighborhoods $U_1,\dots,U_K$. Let
$Y=Y_1\cup\cdots\cup Y_K$, a collection of $Km$ centers, together with the associated scales. Given any $(x,t)\in \Msf\times[0,T]$, it lies in some $U_j$, where $Y_j\subseteq Y$
already spans $J^2_x$; adjoining the remaining centers preserves spanning, since a spanning set remains so under enlargement. $Y$ spans $J^2_x$ at
every $(x,t)\in \Msf\times[0,T]$, so $Y \subseteq G$ and $G\neq\varnothing$ once
$K\ge Km$.

It remains to prove that $G$ is open. This property holds true because $G$ is the set where the finitely many continuous Wronskians $W_1,\dots,W_K$ are nonvanishing over the compact
$\Msf\times[0,T]$, a finite union of open sets. 

\emph{Step 2: Transversality.}
$\widehat\Psi|_{(\Msf\times[0,T])\times(\R^K_{>0}\times G)}$ is a submersion, it is transverse to $\mathcal C$  and to
each stratum $\mathcal D_1^{(j)}=\mathcal C\cap\{\Hess\in\Sigma_j^\circ\}$. By
the parametric transversality theorem \cite{GuilleminPollack}, almost every $\theta \in G$ makes $\Psi_\theta$ transverse to $\mathcal C$
and to each stratum $\mathcal D_1^{(j)}$.

\emph{Step 3: Analytic metric.}
Now suppose $g$ real-analytic and $K\ge\binom{n+2}{2}+n+1$. We study the dimension of the bad set $\Theta_{\tau,y}\setminus G$ -- where uniform spanning fails -- in the joint parameter space $(0,\infty)^K\times \Msf^K$ of both scales $\tau$ and centers $y$.  $(\tau,y)\notin G$ if and only if there exist $(x,t)\in \Msf\times[0,T]$ and a nonzero $\ell\in(J^2_x)^*$ annihilating all $K$ jets. Actually, it is convenient to use the projective annihilator $[\ell]\in\mathbb P((J^2_x)^*)$ rather than $\ell\in(J^2_x)^*\setminus\{0\}$ for two reasons. First, the conditions
$\ell(\jet{2}{x}p_{\tau_i+t}(\cdot,y_i))=0$ are unchanged under
$\ell\mapsto\lambda\ell$ ($\lambda\neq0$), so considering $\ell$ would add a redundant scaling dimension. Second, $\mathbb P((J^2_x)^*)$ is compact whereas
$(J^2_x)^*\setminus\{0\}$ is not, and this compactness makes the projection $\pi$ below a proper map.
Let us consider the the incidence set
\begin{equation*}
I=\bigl\{(x,t,[\ell],\tau,y):\ 0\neq\ell\in(J^2_x)^*,\
     \ell(\jet{2}{x}p_{\tau_i+t}(\cdot,y_i))=0\ \forall i\bigr\},
\end{equation*}
The bad set is given by $\Theta_{\tau,y}\setminus G=\pi(I)$, where $\pi$ is the projection omitting the first three coordinates. The base $(x,t,[\ell])$ has dimension $n+1+(m-1)=n+m$, since $x\in \Msf$,
$t\in[0,T]$ and $[\ell]\in\mathbb P((J^2_x)^*)$. Over a fixed base point, the
$i$-th condition $\ell(\jet{2}{x}p_{\tau_i+t}(\cdot,y_i))=0$ constrains the pair
$(\tau_i,y_i)$, which ranges in the $(n{+}1)$-dimensional space $(0,\infty)\times \Msf$. The analyticity of the heat kernel implies that the function
$(\tau_i,y_i)\mapsto\ell(\jet{2}{x}p_{\tau_i+t}(\cdot,y_i))$ is real-analytic and, by
Lemma~\ref{lem:ampleness} at the effective scale $\tau_i+t$, not identically zero
in $y_i$ -- a nonzero $\ell$ cannot annihilate the whole spanning family. Therefore its zero set is a proper analytic variety, of dimension $(n{+}1)-1=n$. The
$K$ pairs are independently constrained, so
\begin{equation*}
\dim I\ \le\ (n+m)+K\cdot n\ =\ K(n+1)-(K-m-n),
\end{equation*}
which is strictly less than $\dim\bigl((0,\infty)^K\times \Msf^K\bigr)=K(n+1)$ once
$K\ge m+n+1$. 
Locally $I$ is a subanalytic set: on a local chart each condition $\ell(\jet{2}{x}p_{\tau_i+t}(\cdot,y_i))=0$ is a real-analytic equation, so around each point $I$ is defined by the finitely many analytic functions $F_1,\dots,F_K$. The projection $\pi$ omitting $(x,t,[\ell])$ is proper -- the forgotten coordinates $x\in \Msf$, $t\in[0,T]$, $[\ell]\in\mathbb P((J^2_x)^*)$
lie in compact spaces -- and since the image of a subanalytic set under a proper
analytic projection is subanalytic of no larger dimension \cite{zbMATH04103444,zbMATH07865143},
$\pi(I)$ is a subanalytic subset of the $K(n+1)$-dimensional parameter space of dimension less than $K(n+1)$, hence of measure zero. Thus $G$ is open, dense, and of full measure in $(0,\infty)^K\times \Msf^K$, and the good set $G$ is residual and of full measure in the whole parameter space. The thesis follow from the dimension count \eqref{eq:preimage-dim}.
\end{proof}

\begin{remark}[Descent to the probability simplex by scale-invariance]
\label{rem:scaling-descent}
Corollary~\ref{cor:submersion} lives on the cone $\R^K_{>0}$, whereas a probability mixture has $w\in\Delta^{\circ}_{K-1}$. The transversality conclusions of Lemma~\eqref{lem:transversality} descend from the cone to the simplex because everything is
invariant under the scaling $w\mapsto cw$, $c>0$:
\begin{itemize}
\item The critical points are scale-invariant:
      $\nabla(cu)=c\,\nabla u$ has the same zero set, and
      $\operatorname{Hess}(cu)=c\,\operatorname{Hess} u$ has the same kernel, corank, and germ type $A_k$ for every $c>0$.
\item Loci $\mathcal C=\{\nabla u=0\}$, $\Sigma_k^{\circ}$ and $\mathcal D_k$ are cones in the jet fiber -- each condition is scale-invariant -- and $u\mapsto cu$ acts on the fiber as a linear automorphism preserving them. Hence
      $\Psi_{(w,\tau,y)}\pitchfork\Sigma\iff\Psi_{(cw,\tau,y)}\pitchfork\Sigma$.
\end{itemize}
Consequently, the exceptional set
$B\subset\R^K_{>0}\times(0,\infty)^K\times \Msf^K$ is invariant under scaling of $w$. Furthermore, writing $w=r\omega$ with $r>0$,
$\omega\in\Delta^{\circ}_{K-1}$ -- so that
$\R^K_{>0}\cong(0,\infty)_r\times\Delta^{\circ}_{K-1}$ and
$dw=r^{K-1}\,dr\,d\omega$) -- if $B$ is null then
\[
  0=\operatorname{meas}(B)
   =\int_{\Delta^{\circ}}\!\Big(\int_0^\infty \mathbf 1[\,\omega\in B_\Delta\,]\,
     r^{K-1}\,dr\Big)d\omega
   =\Big(\underbrace{\textstyle\int_0^\infty r^{K-1}\,dr}_{=\,\infty}\Big)
     \operatorname{meas}_{\Delta}(B_\Delta),
\]
forcing $\operatorname{meas}_{\Delta}(B_\Delta)=0$. Thus a full-measure set of
unnormalized parameters yields a full-measure set of probability
mixtures with the same transversality, and the final statements may be read for
honest probability mixtures $w\in\Delta^{\circ}_{K-1}$ with hypothesis
$K\ge\binom{n+2}{2}$.
\end{remark}
Now we can prove that, generically, the bifurcations are of $A_k$ type and that they are discrete in time.
\begin{proposition}[Generic corank one]
\label{prop:corank-one}
For every $\theta$ in the good set of Lemma~\ref{lem:transversality}, every
bifurcation $(x^*,t^*)$ satisfies $\dim\ker\operatorname{Hess} u_{t^*}(x^*)=1$; no corank-$\ge2$
degeneracy occurs.
\end{proposition}

\begin{proof} \,
By Lemma~\ref{lem:transversality}, $\Psi_\theta$ is transverse to every stratum $\mathcal D_1^{(j)}=\mathcal C\cap\{\text{Hessian block}\in\Sigma_j^\circ\}$. By the
transverse-preimage dimension count \eqref{eq:preimage-dim},
$\Psi_\theta^{-1}(\mathcal D_1^{(j)})$ is a manifold of negative dimension for $j\ge2$ and therefore empty. Consequently, the bifurcation locus is $\Psi_\theta^{-1}(\mathcal D_1^{(1)})$, of dimension $0$, which contains no point of corank $\ge2$.
\end{proof}

\begin{proposition}[Discreteness of bifurcation times]
\label{prop:discreteness}
For every $\theta$ in the good set of Lemma~\ref{lem:transversality}, the set of
bifurcations is discrete in $\Msf\times(0,\infty)$; in particular the bifurcation
times are discrete, and finite on every compact subinterval.
\end{proposition}

\begin{proof} \,
Write $B=\Psi_\theta^{-1}(\mathcal D_1)$ for the bifurcation locus. By
Proposition~\ref{prop:corank-one} the preimages
$\Psi_\theta^{-1}(\mathcal D_1^{(j)})$, $j\ge2$, are empty, so
$B=\Psi_\theta^{-1}(\mathcal D_1^{(1)})$, which by transversality to the stratum
$\mathcal D_1^{(1)}=\mathcal C\cap\{\Hess\in\Sigma_1^\circ\}$ is a $0$-dimensional
embedded submanifold, therefore every point of $B$ is isolated. Moreover $B$ is
closed in $\Msf\times(0,\infty)$, being the preimage of the closed set
$\mathcal D_1$ (the corank-$\ge1$ critical locus) under the continuous
$\Psi_\theta$. Hence $B$ is closed and
discrete; on a compact time-slab $\Msf\times[t_0,t_1]$ it is compact, hence
finite, and its projection to the time axis is discrete.
\end{proof}

\subsection{The generic fold is $A_2$}
\label{sec:finite-det}

Proposition~\ref{prop:corank-one} only guarantees that the bifurcations are of type $A_k$. Now we prove that, generically, they are $A_2$, whose normal form at a fixed time is given by \eqref{eq:Ak-normal-form} with $k=2$. 

\subsubsection{The corank-one reduction is smooth}
\label{subsec:smooth-splitting}
Let us consider a bifurcation at $(x^*,t^*)$. By Proposition~\ref{prop:corank-one} it has
$\dim\ker \operatorname{Hess} u_{t^*} (x^*)=1$. Then, by Thom's splitting lemma there is a smooth change of coordinates $(\xi,\eta)$ near $x^*$ such that
\begin{equation}
  \label{eq:smooth-splitting}
  u_{t^*}\;=\;\mathrm{const}+\phi(\xi)+\tfrac12\sum_{j=1}^{n-1}\varepsilon_j\,\eta_j^2,
  \qquad \phi(0)=\phi'(0)=\phi''(0)=0,
\end{equation}
where $\phi\in C^\infty$. The scalars
\begin{equation}
  \label{eq:phi-derivatives}
  \phi'''(0),\ \phi''''(0),\ \dots
\end{equation}
are components of the $3$-jet, $4$-jet, $\dots$ of $u_{t^*}$ and depend smoothly on $j^k_{x^*}u_{t^*}$.

The reduction of a bifurcation to the $A_2$ fold normal form $\xi^3$ requires that that infinitely-flat critical points do not
occur, and in particular that the reduced one-variable germ has nonvanishing cubic ($\phi'''(0)\neq0$). Both facts follow from a single transversality count in $J^3$, valid for smooth $g$.

To determine that the bifurcation is $A_2$, we need to work in $J^3$. Inside the critical points of corank-one locus $\mathcal{C} \cap \Sigma_1^{\circ}$, we require the vanishing of the reduced cubic, considering the subspace:
\begin{equation}
  \label{eq:Ageq3-stratum}
  \mathcal A_{\ge3}
  \;:=\;
  \{\nabla u=0\}\ \cap\ \{\text{Hessian}\in\Sigma_1^{\circ}\}\ \cap\
  \{\phi'''(0)=0\}
  \ \subset\ J^3(\Msf,\R).
\end{equation}
The three blocks are independent directions of the jet fiber: $\nabla u=0$ is
$n$ conditions; $\Sigma_1^{\circ}$ (corank exactly one) is imposing $1$ condition on the
Hessian block; and the condition $\phi'''(0)=0$ is unconstrained by the lower-order blocks. Hence
\begin{equation}
  \label{eq:Ageq3-codim}
  \codim_{J^3}\mathcal A_{\ge3}\;=\;n+1+1\;=\;n+2 .
\end{equation}

Let $\Psi^{3}(x,t)=\jet{3}{x}u_t(x)$ be the parametrized $3$-jet section. If $\Psi^{3}$ is transverse to $\mathcal A_{\ge3}$,
\begin{equation}
  \label{eq:Ageq3-preimage}
  \dim (\Psi^{3})^{-1}(\mathcal A_{\ge3})
  \;=\;(n+1)-(n+2)\;=\;-1,
\end{equation}
namely the preimage is empty and generically $\phi'''(0)\neq0$ at every bifurcation.

\subsubsection{Flat germs are contained in \texorpdfstring{$\mathcal A_{\ge3}$}{A>=3}}
\label{subsec:flat-in-Ageq3}

\begin{lemma}
\label{lem:flat-exclusion}
Let $x^*$ be a corank-one critical point of $u_{t^*}$ whose reduced germ $\phi$ in \eqref{eq:smooth-splitting} is infinitely flat
($\phi^{(k)}(0)=0$ for all $k$). Then $j^3_{x^*}u_{t^*}\in\mathcal A_{\ge3}$.
Consequently any parameter $\theta$ for which $\Psi^{3}_\theta$ is transverse to
$\mathcal A_{\ge3}$ admits no flat critical points at all.
\end{lemma}

\begin{proof} \,
Flatness forces $\phi'''(0)=0$, which together with
$\nabla u=0$ and corank one is exactly the definition of  $\mathcal A_{\ge3}$. By
\eqref{eq:Ageq3-preimage} that stratum has empty preimage under a transverse
$\Psi^{3}_\theta$, so no such point exists.
\end{proof}

\subsubsection{The surviving germ is exactly \texorpdfstring{$A_2$}{A2}}
\label{subsec:is-A2}

\begin{lemma}[Smooth $2$-determinacy of the fold]
\label{lem:A2-determinacy}
Let $\phi\in C^\infty(\R,0)$ satisfy $\phi''(0)=0$ and
$\phi'''(0)\neq0$. Then $\phi$ is right-equivalent (by a $C^\infty$
coordinate change) to $\xi\mapsto\xi^3$, namely a corank-one critical point with $\phi'''(0)\neq0$ is an $A_2$ fold.
\end{lemma}

\begin{proof} \,
Applying Hadamard's lemma thrice to $\phi(\xi)$ -- using that $\phi(0)=\phi'(0)=\phi''(0)=0$ -- we can write $\phi(\xi)=\xi^3\psi(\xi)$ for some $\psi\in C^\infty$ such that $\psi(0)=\phi'''(0)/6\neq0$. In a neighborhood of $\xi=0$ the sign permanence theorem guarantees that $\psi$ has constant sign, so
$\chi(\xi):=\xi\,\psi(\xi)^{1/3}$ is a well-defined $C^\infty$ function with
$\chi'(0)=\psi(0)^{1/3}\neq0$; therefore by the inverse function theorem $\chi$ is a local
$C^\infty$ diffeomorphism, and $\phi=\chi^3$. The coordinate change $\xi\mapsto\chi(\xi)$ yields the thesis.
\end{proof}

Assembling the pieces: corank one (Proposition~\ref{prop:corank-one}) reduces a bifurcation to a smooth one-variable germ $\phi(\xi)$ with $\phi''(0)=0$ \eqref{eq:smooth-splitting}; the cubic count \eqref{eq:Ageq3-preimage} makes $\phi'''(0)\neq0$ generic and simultaneously excludes flat germs (Lemma~\ref{lem:flat-exclusion}); and smooth determinacy (Lemma~\ref{lem:A2-determinacy}) identifies the surviving germ as the fold $\xi^3$.

\begin{proposition}[Generic folds]
\label{thm:generic-folds}
Let $(\Msf^n,g)$ be closed with $g$ smooth and let $u_t$ be the heat-kernel mixture \eqref{eq:mixture-density} with
$K$ sufficiently large components ($K\ge\binom{n+3}{3}+n+1$
suffices when $g$ is analytic) and weights $w\in\R^K_{>0}$ (equivalently, after
normalization, $w\in\Delta^{\circ}_{K-1}$; see
Remark~\ref{rem:scaling-descent}). There is a generic set of parameters $\theta$
-- open and of positive measure for smooth $g$, of full measure
when $g$ is real-analytic -- such that every bifurcation $(x^*,t^*)$ of the critical-point structure of $u_t$ is an $A_2$ fold: $\operatorname{Hess} u_{t^*}(x^*)$ has one-dimensional kernel and the reduced cubic
satisfies $\phi'''(0)\ne0$.
\end{proposition}

\begin{proof} \,
By Proposition~\ref{prop:corank-one} every bifurcation is corank one, so the
smooth splitting \eqref{eq:smooth-splitting} applies. Transversality of the
$3$-jet section $\Psi^{3}_\theta$ to $\mathcal A_{\ge3}$ holds for a generic set
of $\theta$ by the Thom--Abraham theorem: the $J^3$ ampleness of
Lemma~\ref{lem:ampleness}, read as a submersion in the unnormalized weight
directions exactly as in Corollary~\ref{cor:submersion}, makes the total $3$-jet
evaluation a submersion on the uniformly-spanning set $G$ (defined as in
\eqref{eq:G-def} with $J^3$ replacing $J^2$). The compactness construction of
Lemma~\ref{lem:transversality}, transfers to $G$ verbatim, since
ampleness (Lemma~\ref{lem:ampleness}) holds at every jet order $r$ and every
scale, so the pointwise-spanning, openness, and pooling arguments apply unchanged
with $J^3$ in place of $J^2$. Sard theorem gives
the measure-zero exceptional set within $G$; scale-invariance (Remark~\ref{rem:scaling-descent}) transfers these findings to probability mixtures. As in
Lemma~\ref{lem:transversality}, $G$ is open and of positive measure in the parameter space for smooth $g$ and of full measure when $g$ is real-analytic; the lower bound on $p$ follows from the same dimensional count in the proof of Lemma~\ref{lem:transversality}, applied to $J^3$ instead of $J^2$. For $\theta \in G$,
\eqref{eq:Ageq3-preimage} forces $\phi'''(0)\neq0$ at every bifurcation, and
Lemma~\ref{lem:flat-exclusion} prevents flat germs.
Lemma~\ref{lem:A2-determinacy} then identifies each bifurcation as an $A_2$ fold. \qed
\renewcommand{\qed}{}
\end{proof}

In order to obtain the local normal form at a bifurcation point, we need to introduce a particular notion of equivalence between germs of functions. Let $\mathcal{E}$ be the ring of the germs at zero with maximal ideal $\langle \xi \rangle$. The group $R$ of germs of diffeomorphism acts on $\mathcal{E}$ via $f \mapsto f \circ \phi$, $\phi \in R$ and $f \in \mathcal{E}$. Now consider the group of translations $(\mathbb{R},+)$. We call $R^+$ the direct product of the groups $R$ and $(\mathbb{R},+)$. An element of $R^+$ acts on $\mathcal{E}$ via the action $f \mapsto f \circ \phi - c$, $(\phi,c) \in R^+$. We say that two germs $f$ and $g$ in $\mathcal{E}$ are $R^+$-equivalent if there is a pair $(\phi,c) \in R^+$ such that $g=f \circ \phi + c$.

\begin{corollary}[Normal form at a speciation event]
\label{cor:cerf}
Let $\theta$ be in the good set of Theorem~\ref{thm:generic-folds} and let
$(x^*,t^*)$ be a bifurcation. Then there are a $t$-dependent $C^\infty$ change of
spatial coordinates near $x^*$, a $C^\infty$ reparametrization of $t$ near $t^*$,
and signs $\varepsilon_1,\dots,\varepsilon_n\in\{\pm1\}$ such that
\begin{equation}
  \label{eq:cerf2}
  u_t(x_1,\dots,x_n)\;=\;c(t)\;+\;x_1^3\;+\;\varepsilon_1\,(t-t^*)\,x_1
   \;+\;\varepsilon_2x_2^2+\dots+\varepsilon_nx_n^2 ,
\end{equation}
where $c(t)$ is a smooth function of $t$ alone and the signs
$\varepsilon_2,\dots,\varepsilon_n$ are those of the nonzero Hessian eigenvalues
at $(x^*,t^*)$.
\end{corollary}
\begin{remark}
For $\varepsilon_1=+1$ two critical points
merge and annihilate as $t$ increases through $t^*$ (a pair death), while for $\varepsilon_1=-1$ a pair is born. When $\varepsilon_2=\dots=\varepsilon_n=-1$ the merging pair is a local maximum and an
index-$(n-1)$ saddle, and the mode count drops by one.
\end{remark}

\begin{proof} \,
Let $\kappa:=\phi'''(0)\neq0$.
\emph{Step 1 (splitting, fibered over $t$).}
By Proposition~\ref{prop:corank-one} the bifurcation is corank one, so
$\partial_\eta^2u_{t^*}$ is invertible at $x^*$, hence invertible for $(\xi,t)$
near $(0,t^*)$. The implicit function theorem with $t$ as a parameter
therefore solves $\partial_\eta u_t=0$ for $\eta=\eta_*(\xi,t)$ smoothly in both
arguments, and substituting gives, after the Morse lemma with parameters applied
to the nondegenerate $\eta$-block (whose signature is locally constant, so the
signs $\varepsilon_j$ do not vary near $t^*$),
\begin{equation}
  \label{eq:param-split}
  u_t \;=\; c(t)+F(\xi,t)+\tfrac12\textstyle\sum_{j=1}^{n-1}\varepsilon_j\eta_j^2,
  \qquad F(\cdot,t^*)=\phi ,
\end{equation}
with $F$ smooth and $\phi$ the germ of \eqref{eq:smooth-splitting}, so
$\phi(0)=\phi'(0)=\phi''(0)=0$ and $\phi'''(0)=\kappa\neq0$. Note that in
these coordinates $u_t$ has no $\xi\eta$ cross terms, so
$\Hess u_t=\operatorname{diag}\bigl(\partial_\xi^2F,\varepsilon_1,\dots,\varepsilon_{n-1}\bigr)$
and $\det\Hess u_t=\partial_\xi^2F\cdot\prod_j\varepsilon_j$.

\emph{Step 2 (removing the quadratic term).}
Since $\partial_\xi^2F(0,t^*)=\phi''(0)=0$ and
$\partial_\xi^3F(0,t^*)=\kappa\neq0$, the implicit function theorem applied to
$\partial_\xi^2F(\xi,t)=0$ gives a unique smooth curve $\xi=\sigma(t)$,
$\sigma(t^*)=0$, of inflection points. Put $\zeta:=\xi-\sigma(t)$ and
$G(\zeta,t):=F(\zeta+\sigma(t),t)$, so that $\partial_\zeta^2G(0,t)\equiv0$ for
all $t$ near $t^*$. Taylor expansion in $\zeta$ with integral remainder then
gives
\begin{equation}
  \label{eq:tschirnhaus}
  G(\zeta,t)=g_0(t)+g_1(t)\,\zeta+\zeta^3\,h(\zeta,t),
  \qquad h(0,t^*)=\tfrac{\kappa}{6}\neq0,
\end{equation}
with $g_0,g_1,h$ smooth and the $\zeta^2$-term absent identically in $t$.
Evaluating at $t=t^*$ (where $\sigma=0$, $G(\cdot,t^*)=\phi$) gives
$g_0(t^*)=\phi(0)=0$ and
\begin{equation}
  \label{eq:g1}
  g_1(t^*)=\phi'(0)=0,\qquad
  g_1'(t^*)=\partial_\xi^2F(0,t^*)\,\sigma'(t^*)+\partial_t\partial_\xi F(0,t^*)
           =\partial_t\partial_\xi F(0,t^*)=:c_1 ,
\end{equation}
the first term vanishing because $\partial_\xi^2F(0,t^*)=0$. Thus the whole
$t$-dependence of the bifurcation is carried by the single function $g_1$, which
vanishes at $t^*$ with derivative $c_1$.

\emph{Step 3 (transversality $\iff c_1\neq0$).}
On the corank-one stratum $\mathcal D_1$ is cut out by the $n+1$ functions
$(\nabla u_t,\det\Hess u_t)$ with independent differentials, so
$\Psi_\theta\pitchfork\mathcal D_1$ at $(x^*,t^*)$ means precisely that
\[
  \Gamma(\xi,\eta,t):=\bigl(\underbrace{\partial_\xi F}_{1},\
   \underbrace{(\varepsilon_j\eta_j)_j}_{n-1},\
   \underbrace{\partial_\xi^2F\cdot\textstyle\prod_j\varepsilon_j}_{1}\bigr)
\]
has invertible differential at $(0,0,t^*)$ (Step 1 supplied the coordinate form).
The $\eta$-rows contribute the invertible block
$\operatorname{diag}(\varepsilon_j)$ and are independent of $(\xi,t)$, so
invertibility reduces to that of the $(\xi,t)$-block, which up to the nonzero
factor $\prod_j\varepsilon_j$ in its second row is
\[
  J=\begin{pmatrix}
     \partial_\xi^2F & \partial_t\partial_\xi F\\[2pt]
     \partial_\xi^3F & \partial_t\partial_\xi^2F
    \end{pmatrix}_{(0,t^*)}
   =\begin{pmatrix} 0 & c_1\\ \kappa & *\end{pmatrix},
  \qquad \det J=-\,c_1\,\kappa .
\]
By Lemma~\ref{lem:transversality} the good $\theta$ satisfy
$\Psi_\theta\pitchfork\mathcal D_1$, so $\det J\neq0$; since $\kappa\neq0$ this is
equivalent to
\begin{equation}
  \label{eq:versality}
  c_1\neq0 .
\end{equation}

\emph{Step 4 ($c_1\neq0 \Rightarrow$ versality and normal form).}

Define for $t=t^*$ the germ $f(\eta) = G(\zeta,t*) = g_0(t^*) + \eta^3 h(\zeta,t^*)$. Subtracting the constant term $g_0(0)$ yields the germ $\widetilde f (\eta) = \eta^3 h(\zeta,t^*) $. Since $h(0,t^*) \neq 0$, near $\zeta=0$ the map $\phi(\zeta) = \zeta \sqrt[3]{h(\zeta,0)}$ is a local diffeomorphism satisfying $\widetilde f \circ \phi^{-1}(y)=y^3$. Consequently, $G(\zeta,t^*)$ is $R^{+}$ equivalent to the singularity $y^3$ and $G(\zeta,t)$ is a one-parameter unfolding, namely a miniversal unfolding. The tangent space $TR^+(y^3)$ of the $R^+$ orbit of $y^3$ is $\langle1,3y^2\rangle$, so we have $NR^+(y^3)=\mathcal{E}/\langle1,3y^2\rangle = \langle y \rangle$. By the infinitesimal versality theorem \cite{zbMATH03826815, zbMATH06124312}, the one-parameter unfold $G$ is $R^+$-versal if $\partial_t G (\zeta,0)$ spans the 1D normal space, namely if $\partial_t G (\zeta,0) =\partial_t g_1(0)\neq 0 $. This is true by the transversality argument of Step 3.
By Thom--Mather Versal unfolding theorem \cite{zbMATH03676845,ArnoldGuseinZadeVarchenko,zbMATH03826815}, any two $R^+$-versal unfoldings of a function with the same number of parameters are $R^+$-equivalent, therefore $G$ must be equivalent, via a local (time-dependent) diffeomorphism $\zeta = \Phi(y,t)$, to the normal form
\begin{equation}
G(\Phi(y,t),t) = g_0(t) + y^3 + a(t) y   
\end{equation}
with $a$ a smooth function with $a^\prime(t^*) \neq 0$. Setting $\varepsilon_1:=\operatorname{sign}a^\prime(t^*)$ and reparametrizing time by
$\tau:=t^*+a(t)/\varepsilon_1 $ yields $a(t)=\varepsilon_1(\tau-t^*)$. Finally, setting $x_{j+1}:=\sqrt2\,\eta_j$ for $j>1$ yields \eqref{eq:cerf2}.
\end{proof}

\bibliographystyle{abbrvnat}
\bibliography{biblio}

\end{document}